\documentclass[twoside]{article}

\usepackage[accepted]{aistats2019}
\usepackage[round]{natbib}
\usepackage{apalike}

\usepackage{lipsum}                     
\usepackage{xargs}                      
\usepackage{color}
\usepackage{hyperref}       
\usepackage{url}            
\usepackage{booktabs}       
\usepackage{amsfonts}       
\usepackage{nicefrac}       
\usepackage{microtype}      
\usepackage{xfrac}
\usepackage{enumitem}
\usepackage{amssymb}
\usepackage{amsmath}
\usepackage{amsthm}
\usepackage{enumerate}
\usepackage{enumitem}
\usepackage{microtype}
\usepackage{graphicx}
\usepackage{bbm}
\usepackage{float}
\usepackage{framed}
\usepackage{algorithm}
\usepackage{algorithmic}
\usepackage{adjustbox}
\usepackage{subfigure}
\usepackage{booktabs} 
\usepackage{titlesec}
\usepackage{tkz-euclide}
\titleformat*{\subsubsection}{\small\bfseries}
\usepackage[font=footnotesize]{caption}

\newcommand{\cov}[2]{\text{Cov}\left[#1, #2\right]}

\newcommand{\var}{{\bf var}}
\newcommand{\g}{{\bf g}}

\newcommand{\w}{{\bf w}}
\renewcommand{\u}{{\bf u}}
\renewcommand{\d}{{\bf d}}
\newcommand{\pw}{\tilde{\w}}

\newcommand{\x}{{\bf x}}
\newcommand{\y}{{\bf y}}
\renewcommand{\v}{{\bf v}}
\newcommand{\z}{{\bf z}}
\newcommand{\q}{{\bf q}}

\newcommand{\V}{{\cal V}}

\newcommand{\bigo}{{\cal O}}

\def\Bm{{\bf B}}

\def\Qm{{\bf Q}}
\def\Im{{\bf I}}

\newcommand{\A}{{\bf A}}
\newcommand{\B}{{\bf B}}

\newcommand{\W}{{\bf W}}
\newcommand{\U}{{\bf U}}
\renewcommand{\V}{{\bf V}}
\newcommand{\D}{{\bf D}}
\renewcommand{\cov}{{\bf S}}

\newcommand{\balpha}{\boldsymbol{\alpha}}
\newcommand{\zero}{{\bf 0}}
\newcommand{\C}{{\bf C}}

\newcommand{\E}{{\mathbf E}}

\newcommand{\Z}{{\mathbf Z}}
\newcommand{\N}{{\mathbf N}}
\newcommand{\R}{{\mathbb{R}}}

\newcommand{\cosapow}[4]{\cos^{#1} \angle_{#2} (#3,#4)}
\newcommand{\sinapow}[4]{\sin^{#1} \angle_{#2} (#3,#4)}
\newcommand{\cosasq}[3]{\cos^2 \angle_{#1} (#2,#3)}
\newcommand{\sinasq}[3]{\sin^2 \angle_{#1} (#2,#3)}

\renewcommand{\rq}{\rho}
\newcommand{\fls}{f_{\text{{\sc ols}}}}
\newcommand{\flh}{f_{\text{{\sc lh}}}}
\newcommand{\fnn}{f_{\text{{\sc nn}}}}

\newcommand{\tdir}{T_{\text{d}}}
\newcommand{\tgd}{T_{\text{gd}}}
\newcommand{\tsc}{T_{\text{s}}}

\newtheorem{lemma}{Lemma}
\newtheorem{proposition}{Proposition}

\theoremstyle{definition}

\usepackage{thmtools,thm-restate}

\usepackage{eqparbox}

\newcommand{\comment}[1]{}

\begin{document}

%

%

\twocolumn[

\aistatstitle{Exponential convergence rates for Batch Normalization: \\ \large{The power of length-direction decoupling in non-convex optimization}}

\aistatsauthor{ Jonas Kohler*\And Hadi Daneshmand* \And  Aurelien Lucchi \And Thomas Hofmann}

\aistatsaddress{ ETH Zurich \And ETH Zurich \And ETH Zurich \And ETH Zurich}

\aistatsauthor{ Ming Zhou \And Klaus Neymeyr}
\aistatsaddress{Universit\"at Rostock \And Universit\"at Rostock } ]

\begin{abstract}
  Normalization techniques such as Batch Normalization have been applied successfully for training deep neural networks. Yet, despite its apparent empirical benefits, the reasons behind the success of Batch Normalization are mostly hypothetical. We here aim to provide a more thorough theoretical understanding from a classical optimization perspective. Our main contribution towards this goal is the identification of various problem instances in the realm of machine learning where 
  Batch Normalization can provably accelerate optimization. We argue that this acceleration is due to the fact that Batch Normalization splits the optimization task into optimizing length and direction of the parameters separately. This allows gradient-based methods to leverage a favourable global structure in the loss landscape that we prove to exist in Learning Halfspace problems and neural network training with Gaussian inputs. We thereby turn Batch Normalization from an effective practical heuristic into a provably converging algorithm for these settings. Furthermore, we substantiate our analysis with empirical evidence that suggests the validity of our theoretical results in a broader context.
\end{abstract}

\section{INTRODUCTION}

One of the most important recent innovations for optimizing deep neural networks is Batch Normalization ({\sc Bn})~\citep{ioffe15batch}. This technique has been proven to successfully stabilize and accelerate training of deep neural networks and is thus by now standard in many state-of-the art architectures such as ResNets \citep{he2016deep} and the latest Inception Nets \citep{szegedy2017inception}. 
The success of Batch Normalization has promoted its key idea that normalizing the inner layers of a neural network stabilizes training which recently led to the development of many such normalization methods such as~\citep{arpit2016normalization,klambauer2017self,salimans2016weight} and \citep{ba2016layer} to name just a few.

Yet, despite the ever more important role of Batch Normalization for training deep neural networks, the Machine Learning community is mostly relying on empirical evidence and thus lacking a thorough theoretical understanding that can explain such success. Indeed -- to the best of our knowledge -- there exists no theoretical result which provably shows faster convergence rates for this technique on any problem instance. So far, there only exists competing hypotheses that we briefly summarize below.

\subsection{Related work}
\paragraph{Internal Covariate Shift}
The most widespread idea is that Batch Normalization accelerates training by reducing the so-called internal covariate shift, defined as the change in the distribution of layer inputs while the conditional distribution of outputs is unchanged.
This change can be significant especially for deep neural networks where the successive composition of layers drives the activation distribution away from the initial input distribution. \cite{ioffe15batch} argue that Batch Normalization reduces the internal covariate shift by employing a normalization technique that enforces the input distribution of each activation layer to be whitened - i.e. enforced to have zero means and unit variances - and decorrelated . Yet, as pointed out by \cite{lipton2018troubling}, the covariate shift phenomenon itself is not rigorously shown to be the reason behind the performance of Batch Normalization. Furthermore, a recent empirical study published by \citep{santurkar2018does} provides strong evidence supporting the hypothesis that the performance gain of Batch Normilization is not explained by the reduction of internal covariate shift.

\paragraph{Smoothing of objective function}
Recently, \cite{santurkar2018does} argue that under certain assumptions a normalization layer simplifies optimization by smoothing the loss landscape of the optimization problem of the preceding layer. Yet, we note that this effect may - at best - only improve the constant factor of the convergence rate of Gradient Descent and not the rate itself (e.g. from sub-linear to linear). Furthermore, the analysis treats only the largest eigenvalue and thus one direction in the landscape (at any given point) and keeps the (usually trainable) {\sc BN} parameters fixed to zero-mean and unit variance. For a thorough conclusion about the overall landscape, a look at the entire eigenspectrum (including negative and zero eigenvalues) would be needed. Yet, this is particularly hard to do as soon as one allows for learnable mean and variance parameters since the effect of their interplay on the distribution of eigenvalues is highly non-trivial.

\paragraph{Length-direction decoupling}
Finally, a different perspective was brought up by another normalization technique termed Weight Normalization ({\sc Wn})~\citep{salimans2016weight}. This technique performs a very simple normalization that is independent of any data statistics with the goal of decoupling the length of the weight vector from its direction. The optimization of the training objective is then performed by training the two parts separately. As discussed in Section \ref{sec:background}, {\sc Bn} and {\sc Wn} differ in how the weights are normalized but share the above mentioned decoupling effect. Interestingly, weight normalization has been shown empirically to benefit from similar acceleration properties as Batch Normalization~\citep{gitman2017comparison,salimans2016weight}. This raises the obvious question whether the empirical success of training with Batch Normalization can (at least partially) be attributed to its length-direction decoupling aspect. 

\subsection{Contribution and organization}

We contribute to a better theoretical understanding of Batch Normalization by analyzing it from an optimization perspective. In this regard, we particularly address the following question:\\

\begin{center}
   \textit{Can we find a setting in which Batch Normalization provably accelerates optimization with Gradient Descent and does the length-direction decoupling play a role in this phenomenon?}
\end{center}

We answer both questions affirmatively. In particular, we show that the specific variance transformation of {\sc Bn} decouples the length and directional components of the weight vectors in such a way that allows local search methods to exploit certain global properties of the optimization landscape (present in the directional component of the optimal weight vector). Using this fact and endowing the optimization method with an adaptive stepsize scheme, we obtain an \textit{exponential} (or as more commonly termed \textit{linear}) convergence rate for Batch Norm Gradient Descent on the (possibly) \textit{non-convex} problem of Learning Halfspaces with Gaussian inputs (Section \ref{sec:learning_halfspaces}), which is a prominent problem in machine learning \cite{erdogdu2016scaled}. We thereby turn {\sc Bn} from an effective practical heuristic into a provably converging algorithm. Additionally we show that the length-direction decoupling can be considered as a non-linear reparametrization of the weight space, which may be beneficial for even simple convex optimization tasks such as logistic regressions. Interestingly, non-linear weightspace transformations have received little to no attention within the optimization community (see \citep{mikhalevich1988minimization} for an exception). 

Finally, in Section \ref{sec:neural_networks} we analyze the effect of {\sc Bn} for training a multilayer neural network (MLP) and prove -- again under a similar Gaussianity assumption -- that {\sc Bn} acts in such a way that the cross dependencies between layers are reduced and thus the curvature structure of the network is simplified. Again, this is due to a certain global property in the directional part of the optimization landscape, which {\sc BN} can exploit via the length-direction decoupling. As a result, gradient-based optimization in reparametrized coordinates (and with an adaptive stepsize policy) can enjoy a linear convergence rate on each individual unit. We substantiate both findings with experimental results on real world datasets that confirm the validity of our analysis outside the setting of our theoretical assumptions that cannot be certified to always hold in practice. 
\section{BACKGROUND} \label{sec:background}

\subsection{Assumptions on data distribution}
Suppose that $\x \in \R^d$ is a random input vector and $y \in \{ \pm 1 \}$ is the corresponding output variable. Throughout this paper we recurrently use the following statistics
\begin{align} \label{eq:stats}
\u:= \E \left[- y \x \right] , \;\; \cov := \E \left[ \x \x^\top \right]
\end{align}
and make the following (weak) assumption.
\begin{restatable}{assumption}{weakassumption} [Weak assumption on data distribution]\label{as:weak_distribution_assumption}
We assume that $\E \left[ \x \right] = 0$. We further assume that the spectrum of the matrix $\cov$ is bounded as 
\begin{align}  \label{eq:covariance_spectral_bound}
0< \mu := \lambda_{\min} \left( \cov \right), L := \lambda_{\max} \left(\cov \right) < \infty  .
\end{align} 
As a result, $\cov$ is the symmetric positive definite covariance matrix of $\x$.
\end{restatable}
The part of our analysis presented in Section \ref{sec:learning_halfspaces} and \ref{sec:neural_networks} relies on a stronger assumption on the data distribution. In this regard we consider the combined random variable
\begin{equation}\label{eq:def_z}
\z:= -y \x, 
\end{equation}
whose mean vector and covariance matrix are $\u$ and $\cov$ as defined above in Eq.~\eqref{eq:stats}. 
\begin{restatable}{assumption}{strongassumption} \label{as:strong_distribution_assumption} [Normality assumption on data distribution]
 We assume that $\z$
 is a multivariate normal random variable distributed with mean $\E \left[ \z \right] = \E \left[ - y \x \right] = \u $ and second-moment $\E \left[ \z \z^\top \right]- \E[\z]\E[\z]^\top = \E \left[ \x \x^\top \right] -\u\u^\top = \cov-\u\u^\top $.
\end{restatable} 

In the absence of further knowledge, assuming Gaussian data is plausible from an information-theoretic point of view since the Gaussian distribution maximizes the entropy over the set of all absolutely continuous distributions with fixed first and second moment~\citep{dowson1973maximum}. Thus, many recent studies on neural networks make this assumption on $\x$ (see e.g. \citep{brutzkus2017globally,du2018power}). Here we assume Gaussianity on $y\x$ instead which is even less restrictive in some cases\footnote{For example, suppose that conditional distribution $P(\x|y=1)$ is gaussian with mean $\mu$ for positive labels and $-\mu$ for negative labels (mixture of gaussians). If the covariance matrix of these marginal distributions are the same, $z = y\x$ is Gaussian while $\x$ is not.}.

\subsection{Batch normalization as a reparameterization of the weight space} \label{sec:batchnorm_as_reparam}
In neural networks a {\sc Bn} layer normalizes the input of each unit of the following layer. This is done on the basis of data statistics in a training batch, but for the sake of analyzablility we will work directly with population statistics. In particular, the output $f$ of a specific unit, which projects an input $\x$ to its weight vector $\w$ and applies a sufficiently smooth activation function $\varphi: \R \to \R$ as follows

\begin{equation}\label{eq:unit_output}
   f(\w)= \E_{\x} \left[ \varphi\left( \x^\top \w \right) \right]
\end{equation}

is normalized on the pre-activation level. That is, the input-output mapping of this unit becomes

\begin{equation}\label{eq:batch_norm_orginal}
    f_{{\sc BN}}(\w,g,\gamma)=\E_{\x} \left[ \varphi\left({\sc BN}(\x^\top \w) \right) \right].
\end{equation}

As stated (in finite-sum terms) in Algorithm 1 of \citep{ioffe15batch} the normalization operation amounts to computing
\begin{equation}
    {\sc BN}(\x^\top \w) =  g\frac{\x^\top \w - \E_{\x} [\x^\top \w]}{\var_{\x}[\x^\top \w]^{1/2}} + \gamma,
\end{equation}

where $g \in \mathbb{R}$ and $\gamma \in \mathbb{R}$ are (trainable) mean and variance adjustment parameters. In the following, we assume that $\x$ is zero mean (Assumption \ref{as:weak_distribution_assumption}) and omit $\gamma$.\footnote{In the (non-compositional) models of Section \ref{sec:OLS} and \ref{sec:learning_halfspaces}, fulfilling the assumption that $\x$ is zero-mean is as simple as centering the dataset. Yet, we also omit centering a neurons input as well as learning $\gamma$ for the neural network analysis in Section~\ref{sec:neural_networks}. This is done for the sake of simplicity but note that \cite{salimans2016weight} found that these aspects of {\sc Bn} yield no empirical improvements for optimization.} Then the variance can be written as follows 
\begin{equation}\label{eq:var}
    \begin{aligned}
    \var_{\x}[\x^\top\w] &= \E_{\x} \left[ (\x^\top \w)^2 \right]=\E_{\x} \left[ (\w^\top \x)(\x^\top \w) \right] \\
    &= \w^\top \E_{\x} \left[ \x \x^\top \right] \w = \w^\top \cov \w
\end{aligned}
\end{equation}

and replacing this expression into the batch normalized output of Eq.\eqref{eq:batch_norm_orginal} yields 
\begin{align}\label{eq:unit_output_BN}
    f_{{\sc BN}}(\w, g) =  \E_{\x} \left[\varphi( g \frac{\x^\top \w}{(\w^\top \cov \w)^{1/2}}) \right]. 
\end{align}
In order to keep concise notations, we will often use the induced norm of the positive definite matrix $\cov$ defined as 
$\| \w \|_\cov := \left( \w^\top \cov \w \right)^{1/2}$. Comparing Eq. \eqref{eq:unit_output} and \eqref{eq:unit_output_BN} it becomes apparent that {\sc Bn} can be considered as a reparameterization of the weight space. We thus define
\begin{align}\label{eq:reparametrization}
  \Tilde{\w} := g \frac{\w}{\| \w \|_\cov}
\end{align}
and note that $\w$ accounts for the direction and $g$ for the length of $\tilde{\w}$. As a result, the batch normalized output can then be written as
\begin{align}
    f_{{\sc BN}}(\w, g) =  \E_{\x} \left[\varphi(\x^\top \tilde{\w} ) \right].
\end{align}
Note that Weight Normalization~({\sc Wn}) is another instance of the above reparametrization, where the covariance matrix $\cov$ is replaced by the identity matrix $\mathbf{I}$ \citep{salimans2016weight}. In both cases, the objective becomes invariant to linear scaling of $\w$. From a geometry perspective, the directional part of {\sc Wn} can be understood as performing optimization on the unit sphere while {\sc Bn} operates on the $\cov$-sphere (ellipsoid) \citep{cho17rieman}. Note that one can compute the variance term \eqref{eq:var} in a matrix-free manner, i.e. $\cov$ never needs to be computed explicitly for {\sc Bn}.

Of course, this type of reparametrization is not exclusive to applications in neural networks. In the following two sections we first show how reparametrizing the weight space of linear models can be advantageous from a classical optimization point of view. In Section \ref{sec:neural_networks} we extend this analysis to training Batch Normalized neural networks with adaptive-stepsize Gradient Descent and show that the length-direction split induces an interesting decoupling effect of the individual network layers which simplifies the curvature structure.  


\section{ORDINARY LEAST SQUARES}\label{sec:OLS}
As a preparation for subsequent analyses, we start with the simple convex quadratic objective encountered when minimizing an ordinary least squares problem
\begin{equation}
\begin{aligned}  \label{eq:least_squares_objective}
&\min_{\Tilde{\w}\in\mathbb{R}^d} \left(\; \fls(\Tilde{\w}) := \E_{\x,y} \left[ \left(y- \x^\top \Tilde{\w}\right)^2 \right] \right)\\ \stackrel{(\text{A\ref{as:weak_distribution_assumption}})}{\Leftrightarrow}  & \min_{\Tilde{\w}\in\mathbb{R}^d} \left( 2 \u^\top \Tilde{\w} + \Tilde{\w}^\top \cov \Tilde{\w} \right).
\end{aligned}
\end{equation}

One can think of this as a linear neural network with just one neuron and a quadratic loss. Thus, applying {\sc Bn} resembles reparametrizing $\Tilde{\w}$ according to Eq.~\eqref{eq:reparametrization} and the objective turns into the \textit{non-convex} problem 
\begin{align} \label{eq:normalized_least_squares}
\min_{\w\in\mathbb{R}^d\setminus \{0\},g \in \mathbb{R}} \left(\fls(\w,g) := 2 g \frac{\u^\top \w}{\| \w \|_\cov } + g^2  \right).
\end{align} 
Despite the non-convexity of this new objective, we will prove that Gradient Descent ({\sc Gd}) enjoys a \textit{linear} convergence rate. Interestingly, our analysis establishes a link between $\fls$ in reparametrized coordinates (Eq.~\eqref{eq:normalized_least_squares}) and the well-studied task of minimizing (generalized) Rayleigh quotients as it is commonly encountered in eigenvalue problems \citep{argentati2017convergence}.

\subsection{Convergence analysis}
To simplify the analysis, note that, for a given $\w$, the objective of Eq.~\eqref{eq:normalized_least_squares} is convex w.r.t. the scalar $g$ and thus the optimal value $g_\w^*$ can be found by setting $\frac{\partial \fls}{\partial g} = 0$, which gives $g_\w^* := -\left( \u^\top \w\right) /\| \w \|_\cov$.
Replacing this closed-form solution into Eq.~\eqref{eq:normalized_least_squares} yields the following optimization problem 
\begin{align} \label{eq:LS_rayleigh_formulation}
   \min_{\w \in \mathbb{R}^d \setminus \{0\}} \left(\rq(\w) := - \frac{\w^\top \u \u^\top \w }{ \w^\top \cov \w}\right),
\end{align}
which -- as discussed in Appendix \ref{sec:backgroundA} -- is a special case of minimizing the generalized Rayleigh quotient for which an extensive literature exists \citep{knyazev1998preconditioned,d1995optimization}. Here, we particularly consider solving \eqref{eq:LS_rayleigh_formulation} with {\sc Gd}, which applies the following iterative updates to the parameters
\begin{align} \label{eq:gd_least_squares}
\w_{t+1} := \w_t + 2 \eta_t \frac{\left( (\u^\top \w_t)\u + \rq(\w_t) \cov \w_t \right)}{\w^\top_t \cov \w_t}.
\end{align}
Based upon existing results, the next theorem establishes a linear convergence rate for the above iterates to the minimizer $\w^*$ in the normalized coordinates.

\begin{restatable}{theorem}{leastsquaresconvergence}{[Convergence rate on least squares]}  \label{thm:least_squares_convergence}
Suppose that the (weak) Assumption~\ref{as:weak_distribution_assumption} on the data distribution holds. Consider the {\sc Gd} iterates $\{\w_t\}_{t\in \mathbb{N}^+}$ given in Eq.~\eqref{eq:gd_least_squares} with the stepsize $\eta_t =  \w_t^\top \cov \w_t/(2L|\rq(\w_t)|)$ and starting from $\rq(\w_0) \neq 0$. Then, 
    \begin{align}\label{eq:suboptimality_convergence_thm}
       \Delta \rho_t \leq \left(1- \frac{\mu}{L}\right)^{2t} \Delta \rho_0,
    \end{align}
where $\Delta \rho_t :=\rq(\w_t) - \rq(\w^*)$.
 Furthermore, the $\cov^{-1}$-norm of the gradient $\nabla \rq(\w_t)$ relates to the suboptimality as 
\begin{align} \label{eq:suboptimality_gradientnorm_thm}
 \| \w_t \|_{\cov}^2 \| \nabla \rq(\w_t) \|^2_{\cov^{-1}}/|4\rq(\w_t)|  = \Delta \rho_t.
\end{align} 
\end{restatable}

This convergence rate is of the same order as the rate of standard {\sc Gd} on the original objective $\fls$ of Eq.~\eqref{eq:least_squares_objective} \citep{nesterov2013introductory}. Yet, it is interesting to see that the non-convexity of the normalized objective does not slow gradient-based optimization down.
In the following, we will repeatedly invoke this result to analyze more complex objectives for which {\sc Gd} only achieves a \textit{sublinear} convergence rate in the original coordinate space but is provably accelerated after using Batch Normalization.
\section{LEARNING HALFSPACES}
\label{sec:learning_halfspaces}

We now turn our attention to the problem of Learning Halfspaces, which encompasses training the simplest possible neural network: the Perceptron. This optimization problem can be written as
\begin{equation}
\label{eq:halfspace_problem}
\begin{aligned} 
\min_{\Tilde{\w}\in\mathbb{R}^d} &\left( \flh(\Tilde{\w})  := \E_{y,\x} \left[ \varphi(\z^\top \Tilde{\w})\right]\right)
\end{aligned} 
\end{equation}
where  $\z := -y \x$ and $\varphi: \R \to \R^+$ is a loss function. Common choices for $\varphi$ include the zero-one, piece-wise linear, logistic and sigmoidal loss. We here tailor our analysis to the following choice of loss function. 
\begin{restatable}{assumption}{regassumption}[Assumptions on loss function] \label{as:reg_assumptions}
We assume that the loss function $\varphi: \R \to \R$ is infinitely differentiable, i.e. $\varphi \in C^\infty(\mathbb{R},\mathbb{R})$, with a bounded derivative, i.e. $\exists \Phi>0$ such that $| \varphi^{(1)}(\beta)| \leq \Phi, \forall \beta \in \R $. 
\end{restatable}
Furthermore, we need $\flh$ to be sufficiently smooth.
\begin{restatable}{assumption}{smoothassumption}[Smoothness assumption] \label{as:smooth_assumptions}
We assume that the objective $f: \R^d \to \R$ is $\zeta$-smooth if it is differentiable on $\R$ and its gradient is $\zeta$-Lipschitz. Furthermore, we assume that a solution  $\alpha^* := \arg\min_{\alpha} \| \nabla f(\alpha \w )\|^2$ exists that is bounded in the sense that $\forall \w \in \R^d, -\infty<\alpha^*< \infty.$\footnote{This is a rather technical but not so restrictive assumption. For example, it always holds for the sigmoid loss unless the classification error of $\w$ is already zero.}
\end{restatable}

Since globally optimizing \eqref{eq:halfspace_problem} is in general NP-hard \citep{guruswami2009hardness}, we instead focus on understanding the effect of the normalized parameterization when searching for a stationary point. Towards this end we now assume that $\z$ is a multivariate normal random variable (see Assumption~\ref{as:strong_distribution_assumption} and discussion there). 



\subsection{Global characterization of the objective}\label{sec:global_halfspace_property}
The learning halfspaces objective $\flh$ -- on Gaussian inputs -- has a remarkable property: all critical points lie on the same line, independent of the choice of the loss function $\varphi$. We formalize this claim in the next lemma. 
\begin{restatable}{lemma}{lhcharaterization} \label{lemm:lh_characterization}
Under Assumptions~\ref{as:weak_distribution_assumption} and~\ref{as:strong_distribution_assumption}, all bounded critical points $\Tilde{\w}_*$ of $\flh$ have the general form 
\begin{align*} 
\Tilde{\w}_* = g_* \cov^{-1} \u,
\end{align*} 
where the scalar $g_* \in \R$ depends on $\Tilde{\w}_*$ and the choice of the loss function $\varphi$. 
\end{restatable}

Interestingly, the optimal direction of these critical points spans the same line as the solution of a corresponding least squares regression problem (see Eq.~\eqref{eq:LS_mininizer} in Appendix A). In the context of convex optimization of generalized linear models, this fact was first pointed out in~\citep{brillinger2012generalized}. Although the global optima of the two objectives are aligned, classical optimization methods - which perform updates based on local information - are generally blind to such global properties of the objective function. This is unfortunate since Gradient Descent converges linearly in the quadratic least-squares setting but only sublinearly on general Learning Halfspace problems \citep{zhang2015learning}. 

To accelerate the convergence of Gradient Descent, \cite{erdogdu2016scaled} thus proposed a two-step global optimization procedure for solving generalized linear models, which first involves finding the optimal direction by optimizing a least squares regression as a surrogate objective and secondly searching for a proper scaling factor of that minimizer. Here, we show that running {\sc Gd} in coordinates reparameterized as in Eq.~\eqref{eq:reparametrization} makes this two-step procedure redundant. More specifically, splitting the optimization problem into searching for the optimal direction and scaling separately, allows even local optimization methods to exploit the property of global minima alignment. Thus - without having to solve a least squares problem in the first place - the directional updates on the Learning Halfspace problem can mimic the least squares dynamics and thereby inherit the linear convergence rate. Combined with a fast (one dimensional) search for the optimal scaling in each step the overall convergence stays linear.

As an illustration, Figure \ref{fig:path} shows the level sets as well as the optimal direction of a least squares-, a logistic- and a sigmoidal regression problem on the same Gaussian dataset. Furthermore, it shows iterates of {\sc Gd} in original coordinates and a sequential version of {\sc Gd}  in normalized coordinates that first optimizes the direction and then the scaling of its parameters ({\sc Gdnp}$_{seq}$). Both methods start at the same point and run with an infinitesimally small stepsize. It can be seen that, while {\sc Gd} takes completely different paths towards the optimal points of each problem instance, the dynamics of {\sc Gdnp}$_{seq}$ are exactly the same until the optimal direction\footnote{Depicted by the dotted red line. Note that -- as a result of Lemma 1 -- this line is identical in all problems.} is found and differ only in the final scaling.

\begin{figure}[h!]
\centering          
          \begin{tabular}{c@{}c@{}}
            \adjincludegraphics[width=0.485\linewidth, trim={22pt 22pt 30pt 30pt},clip]{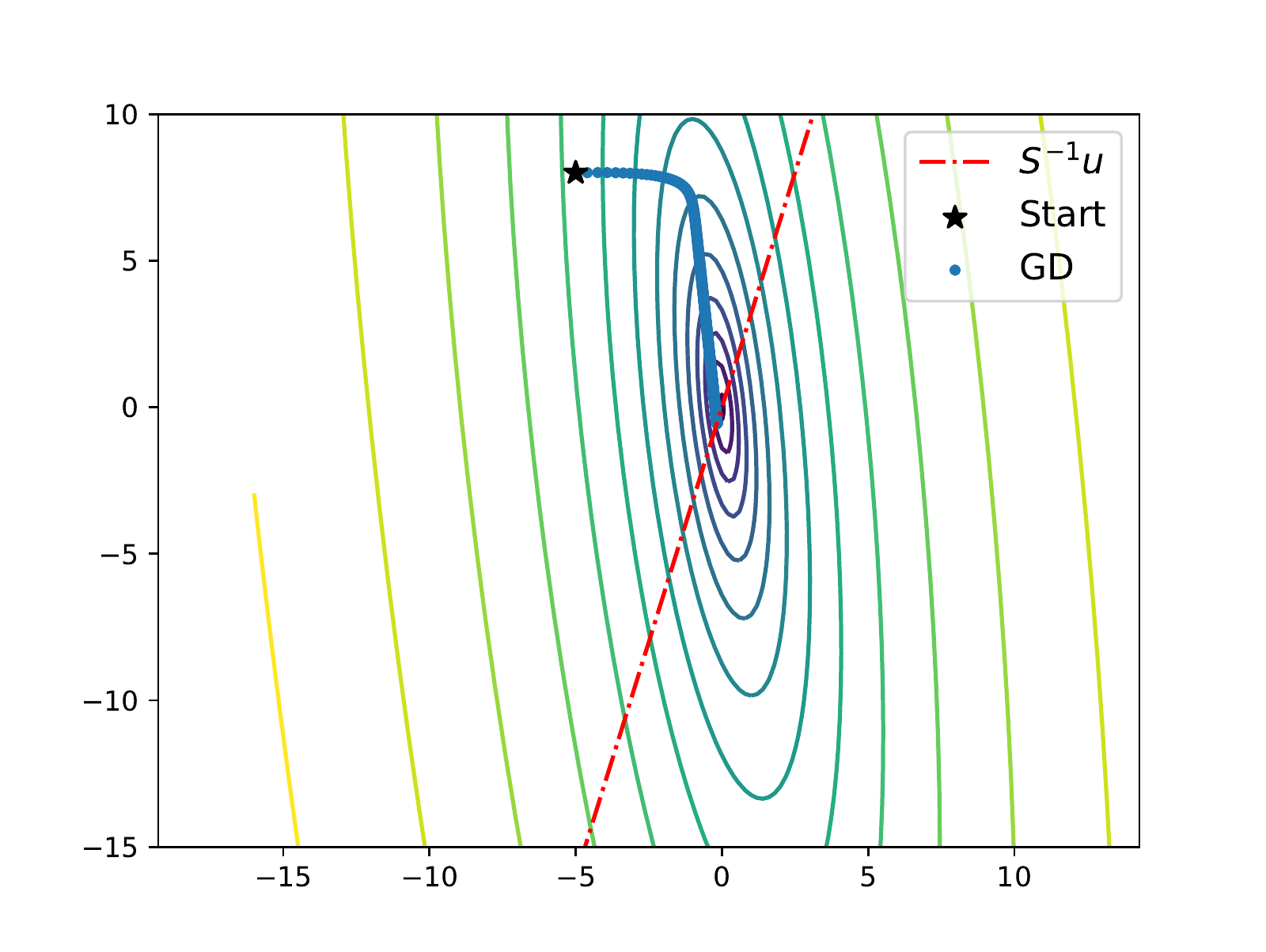} &
             
                \adjincludegraphics[width=0.485\linewidth, trim={22pt 22pt 30pt 30pt},clip]{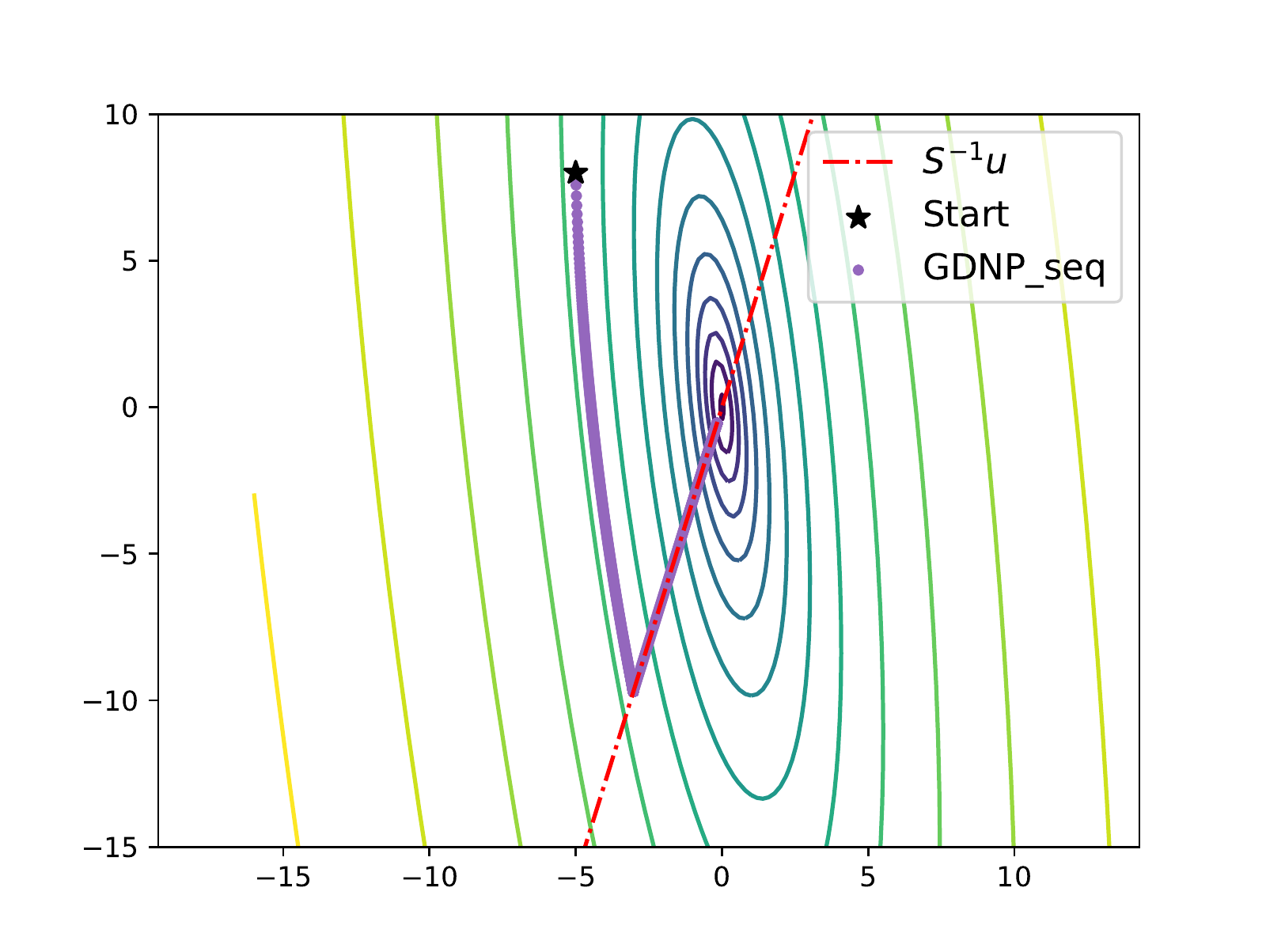} \\
                   \adjincludegraphics[width=0.485\linewidth, trim={22pt 22pt 30pt 30pt},clip]{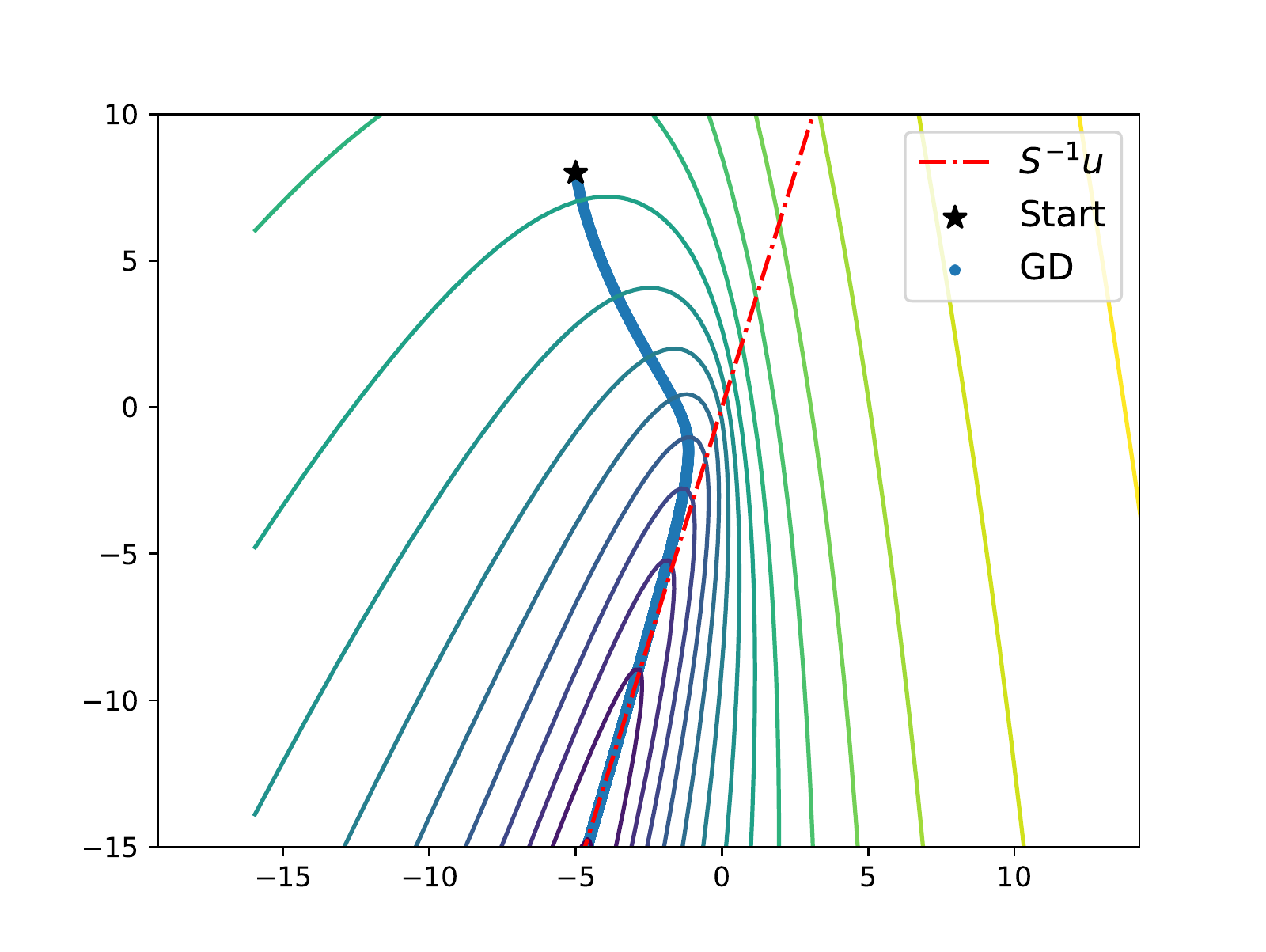} &
             
                \adjincludegraphics[width=0.485\linewidth, trim={22pt 22pt 30pt 30pt},clip]{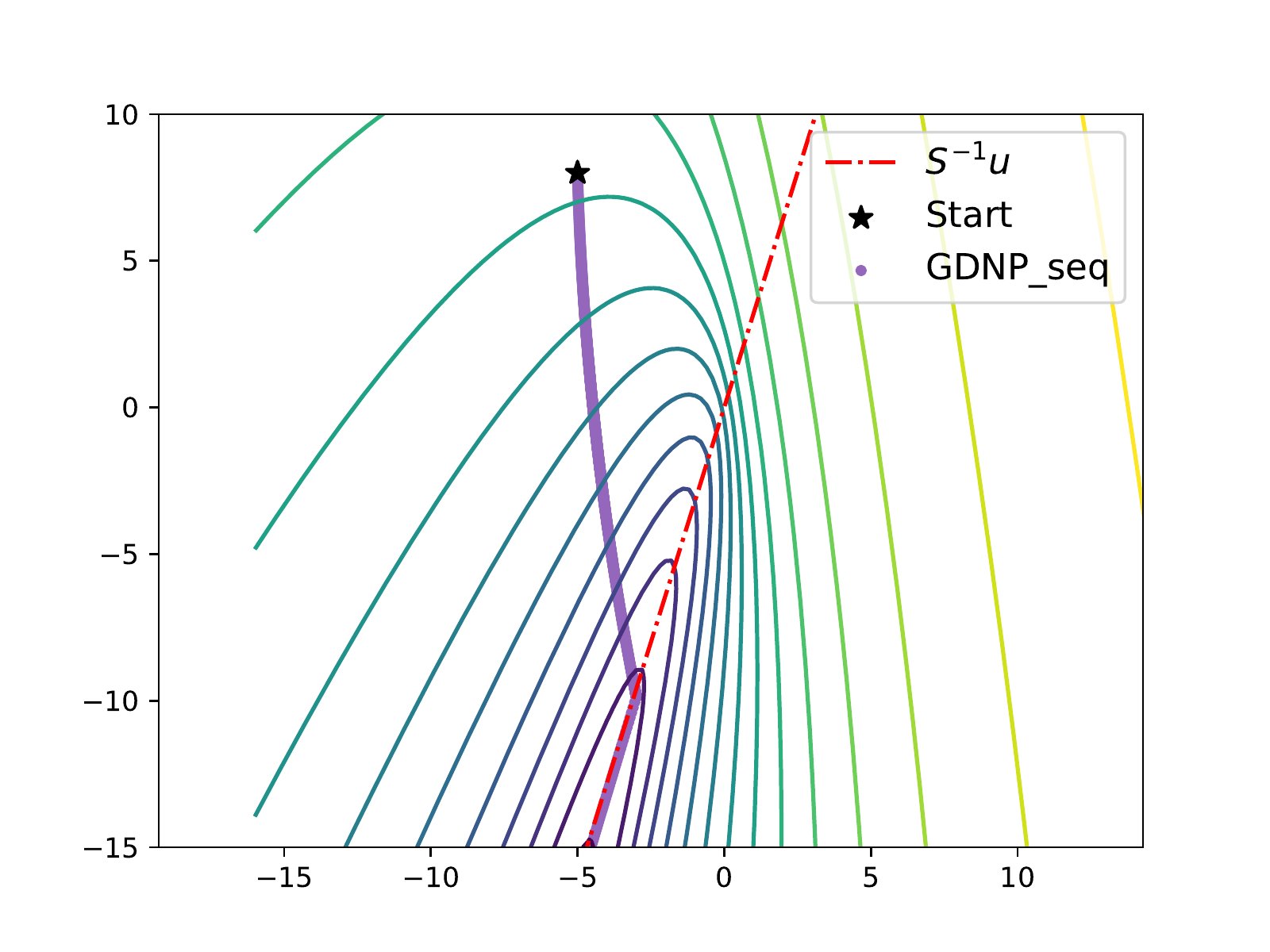} \\
                \adjincludegraphics[width=0.485\linewidth, trim={22pt 22pt 30pt 30pt},clip]{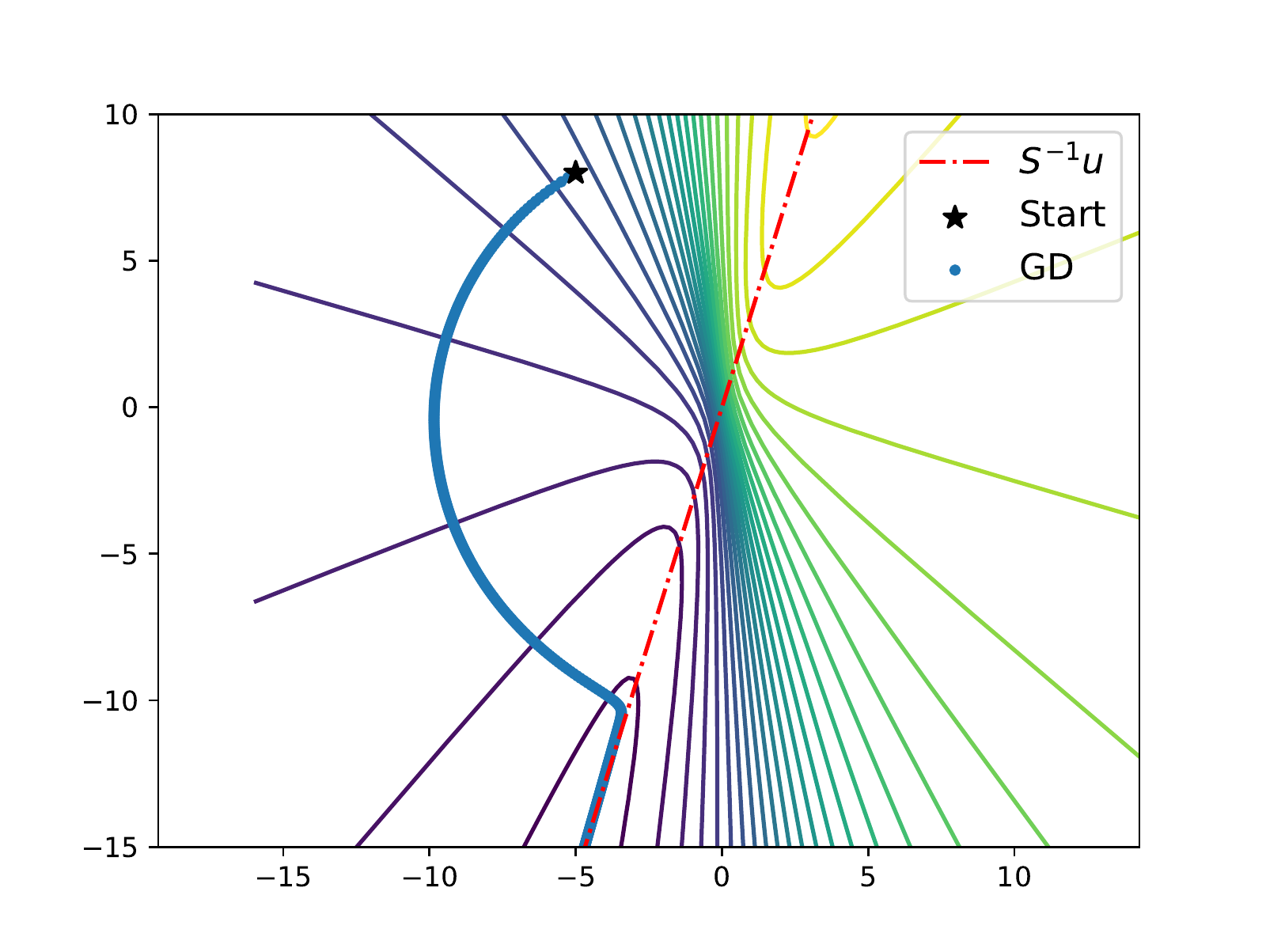}&
             \adjincludegraphics[width=0.485\linewidth, trim={22pt 22pt 30pt 30pt},clip]{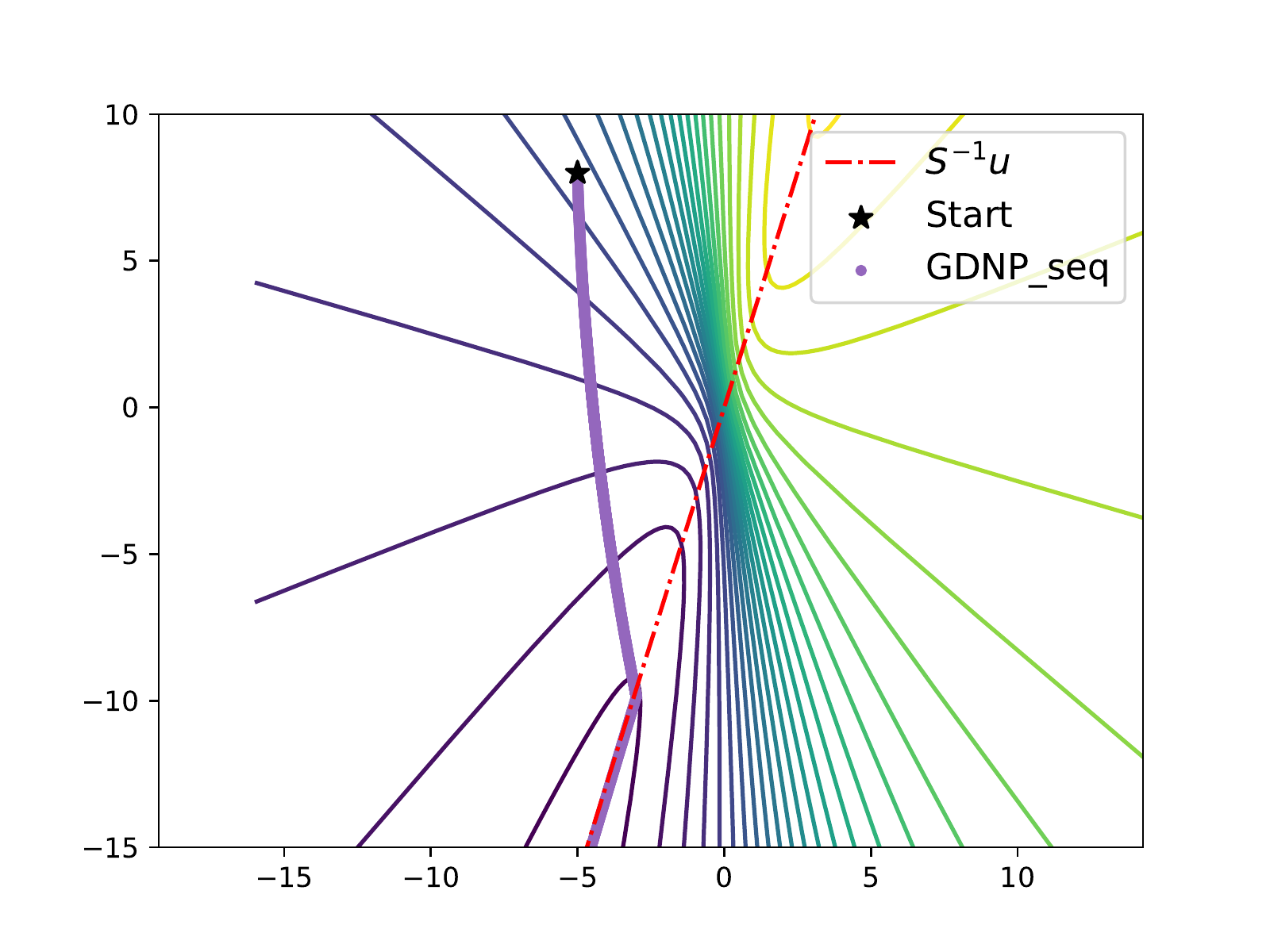}
	  \end{tabular}
          \caption{ \footnotesize{Level sets and path of {\sc Gd} (left) and {\sc Gdnp}$_{seq}$ (right)  
          for a (i) least squares-, (ii) logistic-, and (iii) sigmoidal regression on a Gaussian dataset. All iterates are shown in original coordinates. Note that the {\sc Gdnp}$_{seq}$ paths are identical until the optimal direction (red line) is found.}}
          
          \label{fig:path}
\end{figure}

\subsection{Local optimization in normalized parameterization}\label{sec:lh_reparam}

Let $\flh(\w,g)$ be the $\cov$-reparameterized objective with $\Tilde{\w} = g \w/\| \w \|_\cov $ as defined in Eq.~\eqref{eq:reparametrization}. 
We here consider optimizing this objective using Gradient Descent in Normalized Parameters ({\sc Gdnp}). In each iteration, {\sc Gdnp} performs a gradient step to optimize $\flh$ with respect to the direction $\w$ and does a one-dimensional bisection search to estimate the scaling factor $g$ (see Algorithm~\ref{alg:gdnp}). It is therefore only a slight modification of performing Batch Normalization plus Gradient Descent: 
Instead of taking a gradient step on both $\w$ and $g$, we search for the (locally)-optimal scaling $g$ at each iteration. This modification is cheap and it simplifies the theoretical analysis but it can also be substituted easily in practice by performing multiple {\sc Gd} steps on the scaling factor, which we do for the experiments in Section \ref{sec:exp_i}.

\begin{algorithm}[t]
\begin{algorithmic}[1]\footnotesize{
\STATE \textbf{Input:} $\tdir$, $\tsc$, stepsize policy  $s$, stopping criterion $h$ and objective  $f$
\STATE $g \leftarrow 1 $, and initialize $\w_0$ such that $\rq(\w_0) \neq 0$ 
\FOR{$t = 1,\dots, \tdir$ }
\IF{$h(\w_t,g_t)\neq 0$}
    \STATE $s_t \leftarrow s(\w_t,g_t)$
    \STATE  $\w_{t+1} \leftarrow \w_t - s_t \nabla_{\w} f(\w_t,g_t) $ \quad  \# directional step
\ENDIF

\STATE $g_{t} \leftarrow$ Bisection($\tsc,f,\w_t$) \quad \# see Alg.~\ref{alg:bisection}
\ENDFOR

\STATE $\Tilde{\w}_{\tdir} \leftarrow g_{\tdir} \w_{\tdir}/\| \w_{\tdir}\|_\cov $
\STATE \textbf{return} $\Tilde{\w}_{\tdir}$}
\end{algorithmic}
\caption{\footnotesize{Gradient Descent in Normalized Parameterization ({\sc Gdnp})}}  
\label{alg:gdnp}
\end{algorithm}

\subsection{Convergence result}
We now show that Algorithm~\ref{alg:gdnp} can achieve a linear convergence rate to a critical point on the possibly non-convex objective $\flh$ with Gaussian inputs. Note that all information for computing the adaptive stepsize $s_t$ is readily available and can be computed efficiently.
\begin{restatable}{theorem}{learninghalfspacesconvergence}[Convergence rate of {\sc Gdnp} on learning halfspaces] \label{thm:lh_convergence}
Suppose Assumptions~\ref{as:weak_distribution_assumption}-- \ref{as:smooth_assumptions} hold. Let $\Tilde{\w}_{\tdir}$ be the output of {\sc Gdnp} on $\flh$ with the following choice of stepsizes
\begin{align} \label{eq:step_size}
s_t := s(\w_t,g_t) = - \frac{\| \w_t \|_\cov^{3}}{ L g_t h(\w_t,g_t)  }
\end{align}
for $t=1,\dots, T_d$, where  
\begin{equation}
\begin{aligned}
h(\w_t,g_t) := & \E_\z \left[ \varphi'\left(\z^\top \Tilde{\w}_t\right)  \right](\u^\top \w_t) \\& - \E_\z \left[ \varphi''\left(\z^\top\Tilde{\w}_t\right) \right]\left(\u^\top\w_t \right)^2
\end{aligned}
\end{equation}

is a stopping criterion.
If initialized such that $\rq(\w_0) \neq 0$ (see Eq.~\eqref{eq:LS_rayleigh_formulation}), then $\Tilde{\w}_{\tdir}$ is an approximate critical point of $\flh$ in the sense that \begin{align} 
\| \nabla_{\Tilde{\w}} f(\Tilde{\w}_{\tdir}) \|^2 \leq& (1-\mu/L)^{2\tdir} \Phi^2  \left( \rq(\w_0) - \rq^*  \right) \nonumber\\&+ 2^{-\tsc}\zeta  | b^{(0)}_t - a_t^{(0)} |/\mu^{2}.  
\end{align}
\end{restatable}

To complete the picture, Table \ref{table:complexity_nonconvex_learninghalfspaces} compares this result to the order complexity for reaching an $\epsilon$-optimal critical point $\Tilde{\w}$ (i.e.$\| \nabla_{\Tilde{\w}} \flh(\Tilde{\w})\| \leq \epsilon$) of Gradient Descent and Accelerated Gradient Descent ({\sc Agd}) in original coordinates. Although we observe that the accelerated rate achieved by {\sc Gdnp} is significantly better than the best known rate of {\sc Agd} for non-convex objective functions, we need to point out that the rate for {\sc Gdnp} relies on the assumption of a Gaussian data distribution. Yet, Section \ref{sec:exp_i} includes promising experimental results in a more practical setting with non-Gaussian data. 

Finally, we note that the proof of this result relies specifically on the $\cov$-reparametrization done by Batch Normalization. In Appendix \ref{sec:word_on_WN} we detail out why our proof strategy is not suitable for the $\mathbf{I}$-reparametrization of Weight Normalization and thus leave it as an interesting open question if other settings (or proof strategies) can be found where linear rates for {\sc Wn} are provable.

\begin{table*}[h]
\centering
\footnotesize
\centering
\begin{tabular}{lllll}
\hline
\textbf{Method}        & \textbf{Assumptions} & \textbf{Complexity}   & \textbf{Rate}  &  \textbf{Reference}   \\
\hline 
{\sc Gd} & Smoothness  & $\bigo(\tgd\epsilon^{-2})$ & Sublinear &  \citep{nesterov2013introductory} \\ 
{\sc Agd} & Smoothness & $\bigo(\tgd\epsilon^{-7/4} \log(1/\epsilon))$ & Sublinear &  \citep{jin2017accelerated}\\
{\sc Agd} & Smoothness+convexity & $\bigo(\tgd\epsilon^{-1})$ & Sublinear &  \citep{nesterov2013introductory}\\
 \hline 
{\sc Gdnp}   & \ref{as:weak_distribution_assumption}, \ref{as:strong_distribution_assumption}, \ref{as:reg_assumptions}, and \ref{as:smooth_assumptions} & $\bigo(\tgd\log^2(1/\epsilon))$ & Linear & This paper
\end{tabular} 
\caption{ \footnotesize{Computational complexity to reach an $\epsilon$-critical point on $\flh$ -- with a (possibly) non-convex loss function $\varphi$.  $\bigo(\tgd)$ represents the time complexity of computing a gradient for each method. 
}} 
\label{table:complexity_nonconvex_learninghalfspaces}
\end{table*}
\subsection{Experiments I} \label{sec:exp_i}
\paragraph{Setting} In order to substantiate the above analysis we compare the convergence behavior of {\sc Gd} and {\sc Agd} to three versions of Gradient Descent in normalized coordinates. Namely, we benchmark (i) {\sc Gdnp} (Algorithm \ref{alg:gdnp}) with multiple gradient steps on $g$ instead of Bisection, (ii) a simpler version ({\sc Bn}) which updates $\w$ and $\g$ with just one fixed step-size gradient step\footnote{Thus {\sc Bn} is conceptually very close to the classical Batch Norm Gradient Descent presented in \citep{ioffe15batch}} and (iii) Weight Normalization ({\sc Wn}) as presented in \citep{salimans2016weight}. All methods use full batch sizes and -- except for {\sc Gdnp} on $\w$ -- each method is run with a problem specific, constant stepsize.

We consider empirical risk minimization as a surrogate for $\flh$ \eqref{eq:halfspace_problem} on the common real-world dataset \textit{a9a} as well as on synthetic data drawn from a multivariate Gaussian distribution. We center the datasets and use two different functions $\varphi(\cdot)$. First, we choose the \textit{softplus}  
which resembles the classical logistic regression (convex). 
 Secondly, we use the \textit{sigmoid} 
 which is a commonly used (non-convex) continuous approximation of the 0-1 loss \citep{zhang2015learning}. Further details can be found in Appendix D.

\begin{figure}[h!]
\centering          
          \begin{tabular}{c@{}c@{}}
            \adjincludegraphics[width=0.5\linewidth, trim={22pt 22pt 30pt 30pt},clip]{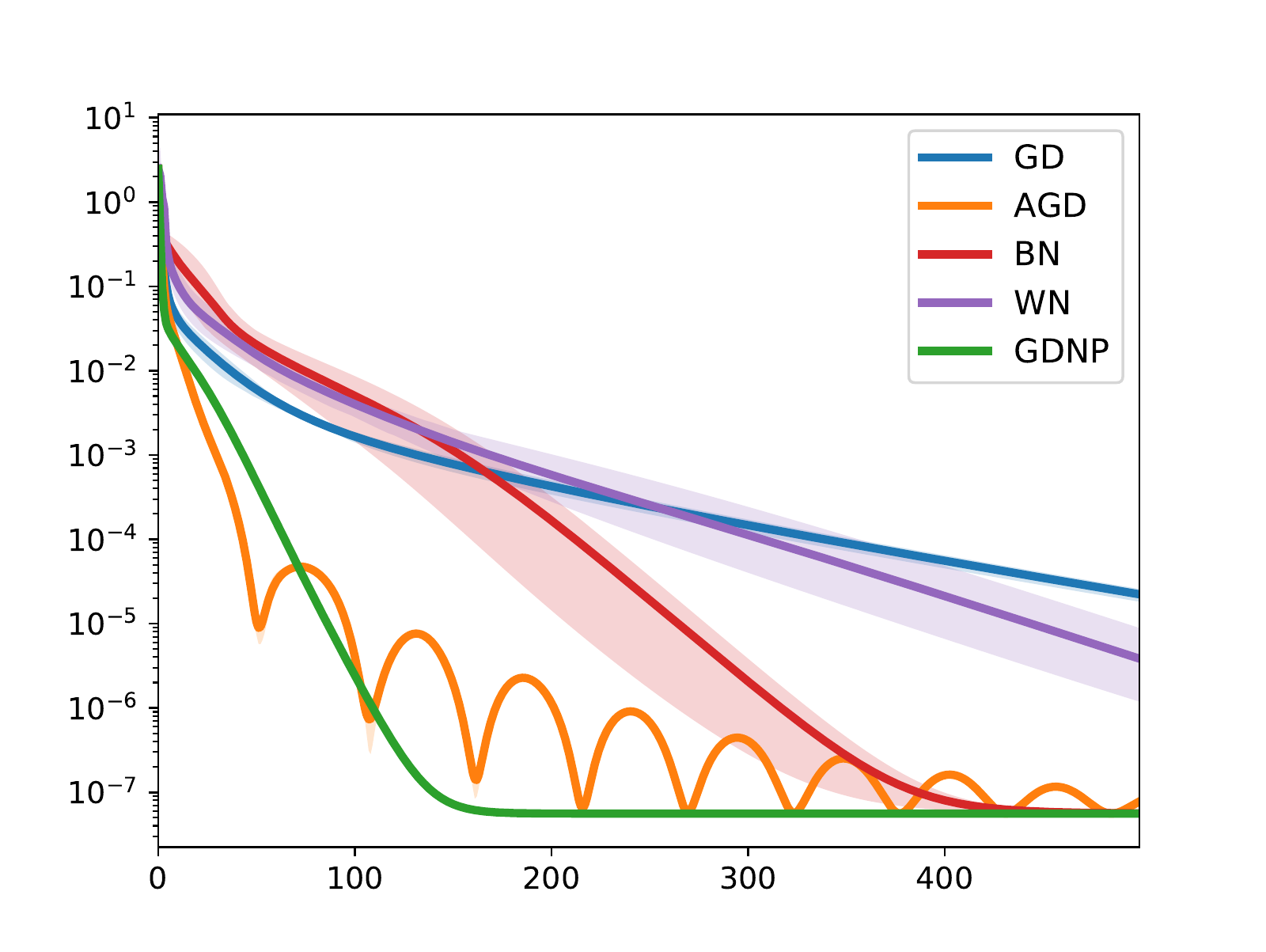} &
             \adjincludegraphics[width=0.5\linewidth, trim={22pt 22pt 30pt 30pt},clip]{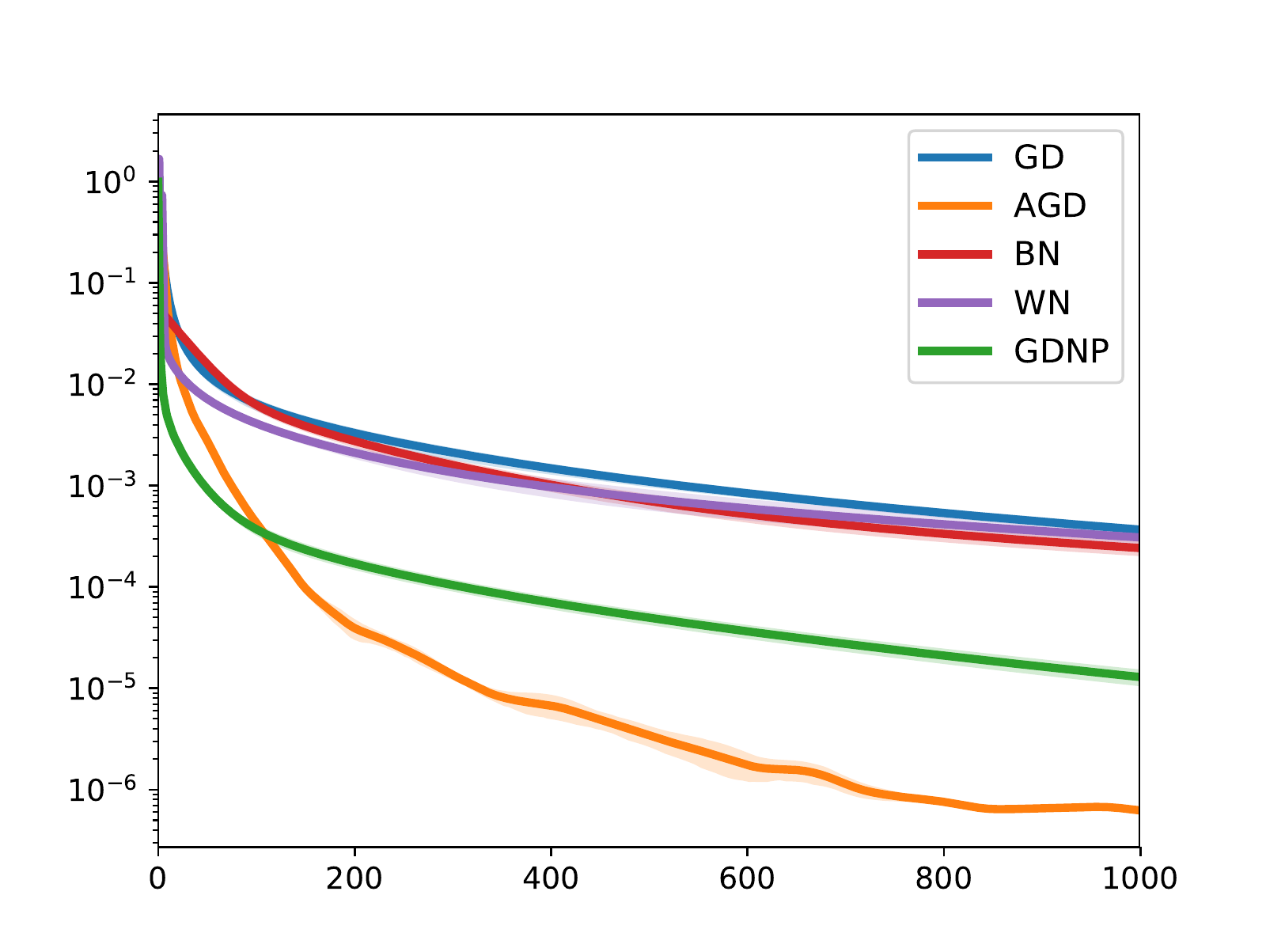}\\
            \footnotesize{{ gaussian softplus}} &
             \footnotesize{{ a9a softplus}} \\ \adjincludegraphics[width=0.5\linewidth, trim={22pt 22pt 30pt 30pt},clip]{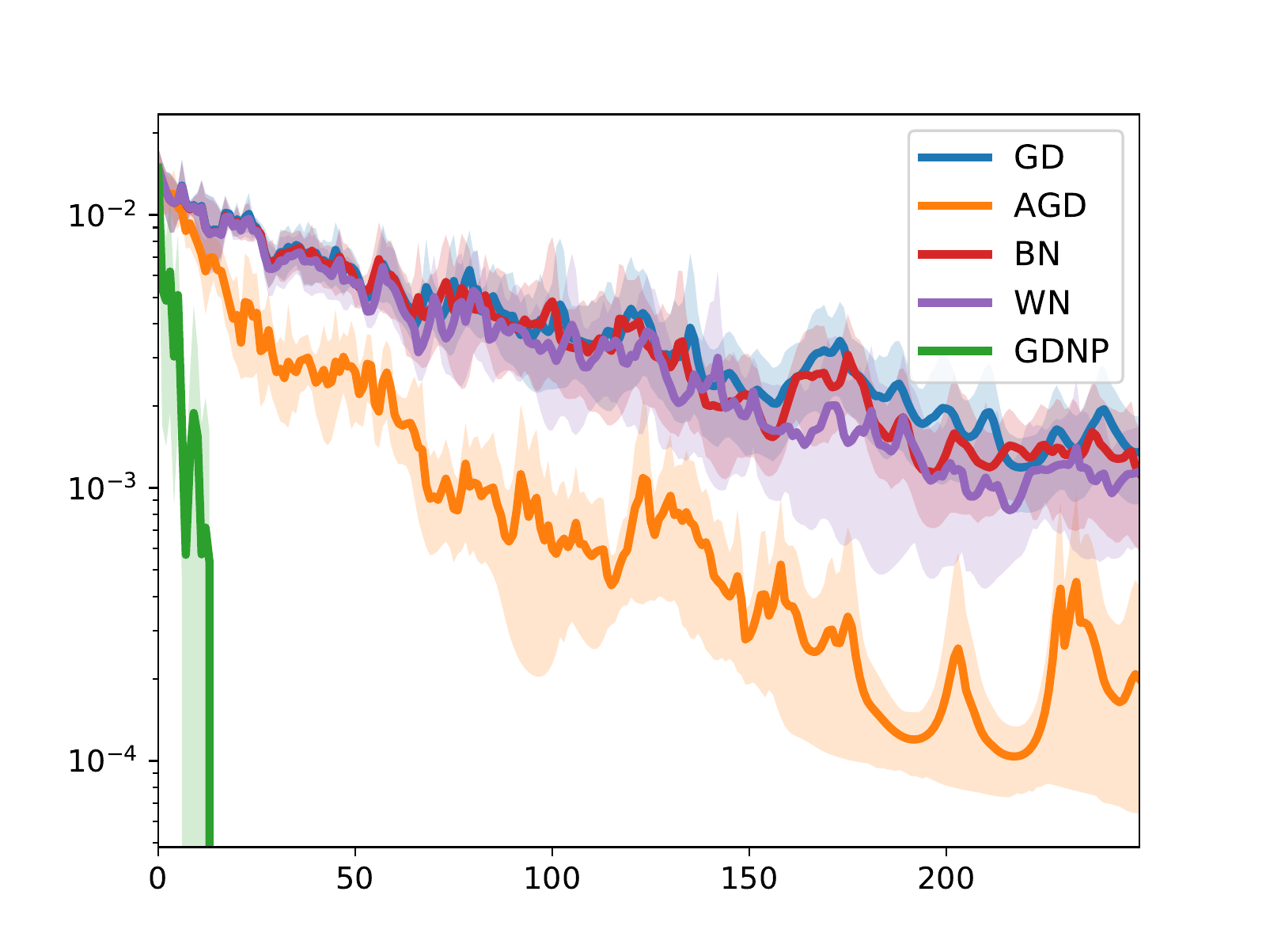}&
              \adjincludegraphics[width=0.5\linewidth, trim={22pt 22pt 30pt 30pt},clip]{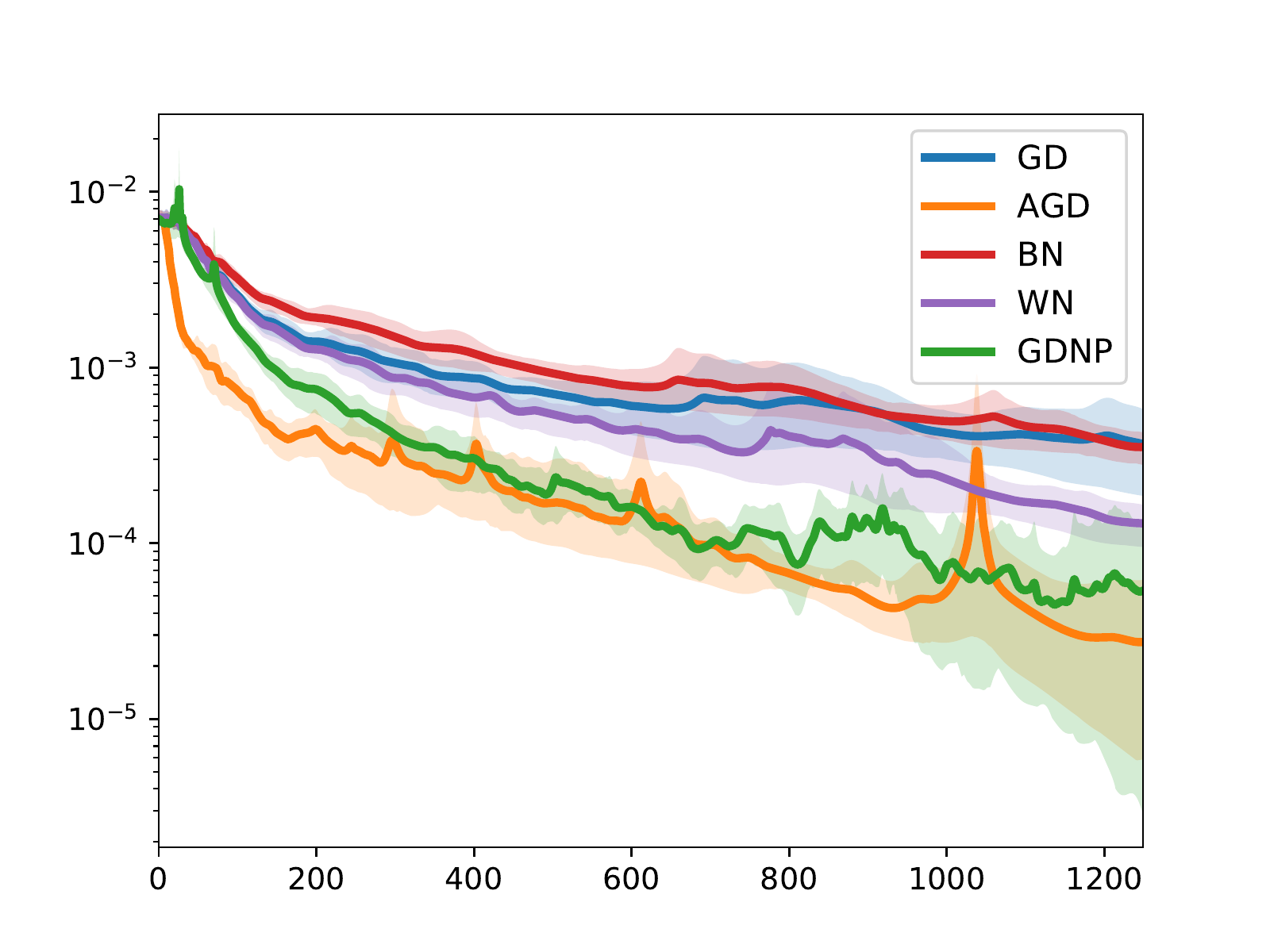} \\  
         	  \hspace{10pt}\footnotesize{{gaussian sigmoid}}&
            \hspace{5pt} \footnotesize{{a9a sigmoid}}
	  \end{tabular}
          \caption{ \footnotesize{Results of an average run (solid line) in terms of log suboptimality (softplus) and log gradient norm (sigmoid) over iterations as well as 90\% confidence intervals of 10 runs with random initialization. }}
          \label{fig:exp_halfspace_log}
\end{figure}

\paragraph{Results} The Gaussian design experiments clearly confirm Theorem~\ref{thm:lh_convergence} in the sense that the loss in the convex-, as well as the gradient norm in the non-convex case decrease at a linear rate. The results on \textit{a9a} show that {\sc Gdnp} can accelerate optimization even when the normality assumption does not hold and in a setting where no covariate shift is present, which motivates future research of normalization techniques in optimization. Interestingly, the performance of simple {\sc Bn} and {\sc Wn} is similar to that of {\sc Gd}, which suggests that the length-direction decoupling on its own does not capture the entire potential of these methods. {\sc Gdnp} on the other hand takes full advantage of the parameter splitting, both in terms of multiple steps on $g$ and -- more importantly -- adaptive stepsizes in $\w$.
\section{NEURAL NETWORKS}
\label{sec:neural_networks}
\vspace{-2mm}
We now turn our focus to optimizing the weights of a multilayer perceptron (MLP) with one hidden layer and $m$ hidden units that maps an input $\x \in \mathbb{R}^d$ to a scalar output in the following way
\begin{align*}
    F_\x(\Tilde{\W},\Theta) = \sum_{i=1}^m \theta_i \varphi(\x^\top \Tilde{\w}^{(i)}).
\end{align*} 
The $\Tilde{\w}^{(i)} \in \R^d$ and $\theta^{(i)}\in\mathbb{R}$ terms are the input and output weights of unit $i$ and $\varphi:\mathbb{R}\rightarrow \mathbb{R}$ is a so-called activation function. We assume $\varphi$ to be a $\tanh$, which is a common choice in many neural network architectures used in practice. Given a loss function $\ell: \mathbb{R}\rightarrow \mathbb{R}^+$, the optimal input- and output weights are determined by minimizing the following optimization problem 
\begin{align}\label{eq:NN_problem}
    \min_{\Tilde{\W},\Theta} \Big( \fnn(\Tilde{\W},\Theta) := \E_{y,\x} \Big[ \ell\Big(-y F_\x(\Tilde{\W},\Theta)\Big) \Big]    \Big),
\end{align}
where
\begin{align*}
\Tilde{\W} := \{ \Tilde{\w}^{(1)}, \dots, \Tilde{\w}^{(m)} \} , \;\; \Theta := \{ \theta^{(1)}, \dots, \theta^{(m)} \}.
\end{align*}
In the following, we analyze the optimization of the input weights $\Tilde{\W}$ with frozen output weights $\Theta$ and thus write $F(\tilde{\W})$ hereafter for simplicity.
As previously assumed for the case of learning halfspaces, our analysis relies on Assumption \ref{as:strong_distribution_assumption}. 
This approach is rather common in the analysis of neural networks, e.g. ~\cite{brutzkus2017globally} showed that for a special type of one-hidden layer network with isotropic Gaussian inputs, all local minimizers are global. Under the same assumption, similar results for different classes of neural networks were derived in~\citep{li2017convergence, soltanolkotabi2017theoretical, du2017gradient}.
\subsection{Global characterization of the objective}
We here show that, under the same assumption, the landscape of $\fnn$ exhibits a global property similar to one of the Learning Halfspaces objective (see Section \ref{sec:global_halfspace_property}). 
In fact, the critical points $\Tilde{\w}_i$ of all neurons in a hidden layer align along one and the same single line in $\mathbb{R}^d$ and this direction depends only on \textit{incoming} information into the hidden layer. 
\begin{restatable}{lemma}{globalpropertynn} \label{lem:critical_point_characterization_nn}

Suppose Assumptions~\ref{as:weak_distribution_assumption} and \ref{as:strong_distribution_assumption} hold and let $\hat{\w}^{(i)}$ be a critical point of $\fnn (\Tilde{\W})$ with respect to hidden unit $i$ and for a fixed $\Theta \neq \zero$. Then, there exits a scalar $\hat{c}^{(i)} \in \R$ such that 
\begin{align} \label{eq:opt_direction_NN}
\hat{\w}^{(i)} = \hat{c}^{(i)} \cov^{-1} \u, \quad \forall i=1,\ldots,m.
\end{align} 
\end{restatable}
See Appendix \ref{sec:deepnets} for a discussion of possible implications for deep neural networks.  


\begin{algorithm}[h]
\begin{algorithmic}[1]
\footnotesize{
\STATE Input: $\tdir^{(i)}$, $\tsc^{(i)}$, step-size policies $s^{(i)}$, stopping criterion $h^{(i)}$
\STATE $\W_{/1} \leftarrow \zero $,  ${\bf g}_{/1} = \textbf{1}$
\FOR{$i = 1,\dots, m$ }
\STATE $ \W_{/i} := \{ \w^{(1)}, \dots,  \w^{(i-1)},  \w^{(i+1)} ,\dots,  \w^{(m)} \} $, $ {\bf g}_{/i} := \{ g^{(1)}, \dots, g^{(i-1)}, g^{(i+1)}, \dots, g^{(m)} \} $
\STATE $f^{(i)}(\w^{(i)},g^{(i)}) := \fnn(\w^{(i)},g^{(i)},\W_{/i},{\bf g}_{/i})$
\STATE $(\w^{(i)}, g^{(i)}) \leftarrow$ {\sc Gdnp}($f^{(i)}$, $\tdir^{(i)}$, $\tsc^{(i)}$, $s^{(i)}$, $h^{(i)}$) \quad \quad \quad \quad \# optimization of unit $i$ 
\ENDFOR
\STATE \textbf{return} $\W := [\w^{(1)}, \dots, \w^{(m)}], {\bf g}:= [g^1, \dots, g^{m}] $.}
\end{algorithmic}
\caption{\footnotesize{Training an MLP with {\sc Gdnp} }} 
\label{alg:gdnp_nn}
\end{algorithm}
\subsection{Convergence result}
We again consider normalizing the weights of each unit by means of its input covariance matrix $\cov$, i.e. $ 
\Tilde{\w}^{(i)} = g^{(i)} \w^{(i)}/ \| \w^{(i)} \|_\cov, \; \forall i = 1, \dots, m$ and present a simple algorithmic manner to leverage the input-output decoupling of Lemma \ref{lem:critical_point_characterization_nn}. Contrary to {\sc Bn}, which optimizes all the weights simultaneously, our approach is based on alternating optimization over the different hidden units. We thus adapt {\sc Gdnp} to the problem of optimizing the units of a neural network in Algorithm~\ref{alg:gdnp_nn} but note that this modification is mainly to simplify the theoretical analysis.

Optimizing each unit independently formally results in minimizing the function $f^{(i)}(\w^{(i)},g^{(i)})$ as defined in Algorithm~\ref{alg:gdnp_nn}. In the next theorem, we prove that this version of {\sc Gdnp} achieves a linear rate of convergence to optimize each $f^{(i)}$.

\begin{restatable}{theorem}{convergencnn}[Convergence of {\sc Gdnp} on MLP] \label{lem:convergence_nn}
Suppose Assumptions~\ref{as:weak_distribution_assumption}-- \ref{as:smooth_assumptions} hold. We consider optimizing the weights 
$(\w^{(i)},g^{(i)})$ of unit $i$, assuming that all directions $\{ \w^{(j)}\}_{j<i}$ 
are critical points of $\fnn$ and $\w^{k} = \zero$ for $k>i$. Then, {\sc Gdnp} with step-size policy $s^{(i)}$ as in \eqref{eq:step_size_nn} and stopping criterion $h^{(i)}$ as in ~\eqref{eq:stopping_time_nn} 
yields a linear convergence rate on $f^{(i)}$ in the sense that
\begin{equation}
    \begin{aligned} 
  \| \nabla_{\Tilde{\w}^{(i)}} f(\Tilde{\w}^{(i)}_t) \|^2_{\cov^{-1}}  \leq&   (1-\mu/L)^{2t}  C \left( \rq(\w_0) - \rq^* \right) \\ &+ 2^{-\tsc^{(i)}}\zeta  | b^{(0)}_t - a_t^{(0)} |/\mu^{2},
\end{aligned} 
\end{equation}

where the constant $C>0$ is defined in Eq.~\eqref{eq:constant_c_nn}.
\end{restatable} 
The result of Theorem~\ref{lem:convergence_nn} relies on the fact that each $\w^{(j)}, j\neq i$ is either zero or has zero gradient. If we assume that an exact critical point is reached after optimizing each individual unit, then the result directly implies that the alternating minimization presented in Algorithm 2 reaches a critical point of the overall objective. Since the established convergence rate for each individual unit is linear, this assumption sounds realistic. We leave a more precise convergence analysis, that takes into account that optimizing each individual unit for a finite number of steps may yield numerical suboptimalities, for future work. 


\subsection{Experiments II}\label{sec:exp_ii}
\paragraph{Setting} In the proof of Theorem \ref{lem:convergence_nn} we show that {\sc Gdnp} can leverage the length-direction decoupling in a way that lowers cross-dependencies between hidden layers and yields faster convergence. A central part of the proof is Lemma \ref{lem:critical_point_characterization_nn} which says that -- given Gaussian inputs -- the optimal direction of a given layer is independent of all downstream layers. Since this assumption is rather strong and since Algorithm \ref{alg:gdnp_nn} is intended for analysis purposes only, we test the validity of the above hypothesis outside the Gaussian setting by training a Batch Normalized multilayer feedforward network ({\sc Bn}) on a real-world image classification task with plain Gradient Descent. 
For comparison, a second \textit{unnormalized} network is trained by {\sc Gd}. To validate Lemma \ref{lem:critical_point_characterization_nn} we measure the interdependency between the central and all other hidden layers in terms of the Frobenius norm of their second partial cross derivatives (in the directional component). Further details can be found in Appendix D.
\begin{figure}[h!]\label{fig:NN_result}
	\begin{center}
          \begin{tabular}{c@{}c@{}}
            \adjincludegraphics[width=0.5\linewidth, trim={22pt 22pt 30pt 30pt},clip]{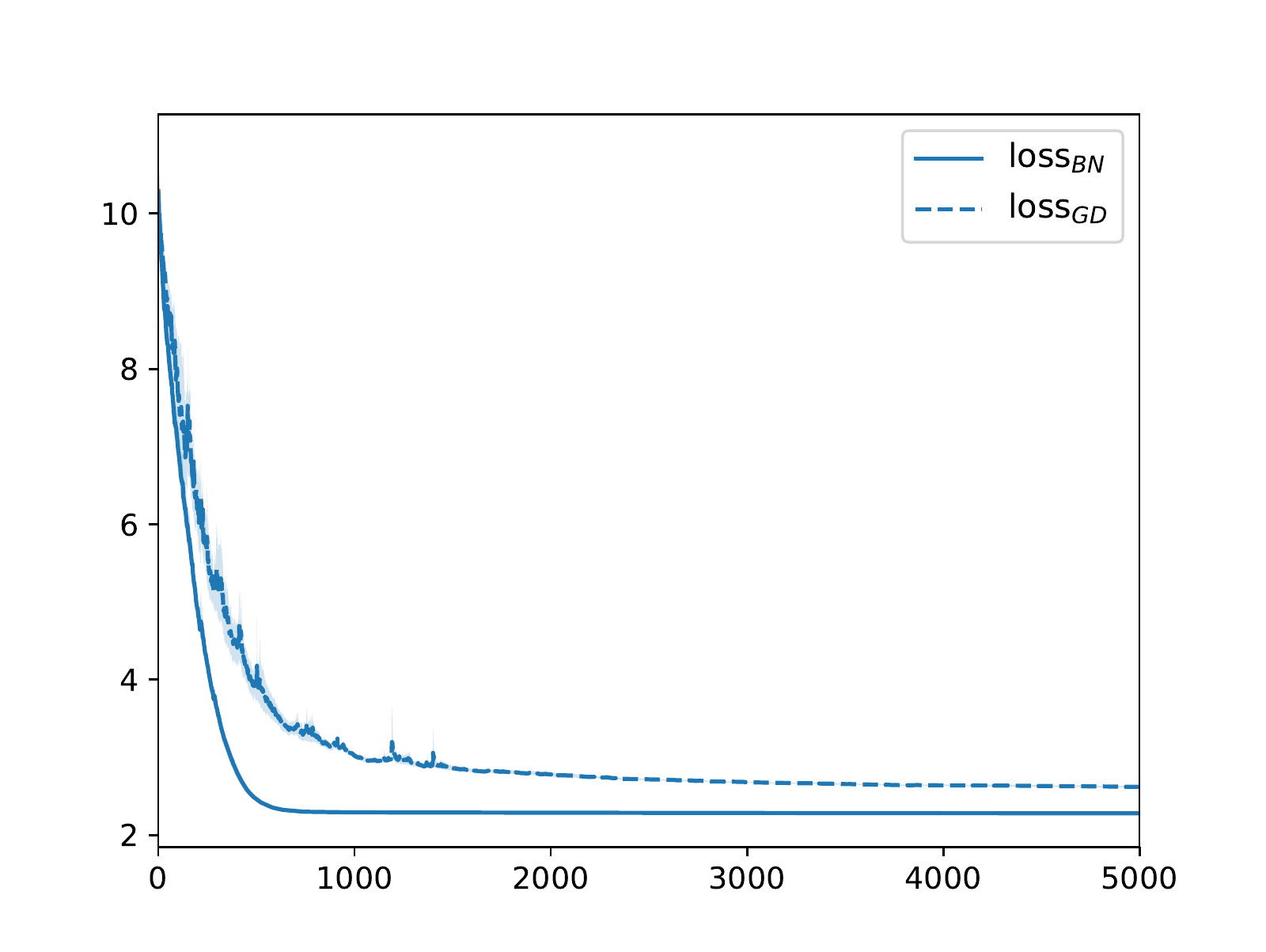} &
            \adjincludegraphics[width=0.5\linewidth, trim={22pt 22pt 30pt 30pt},clip]{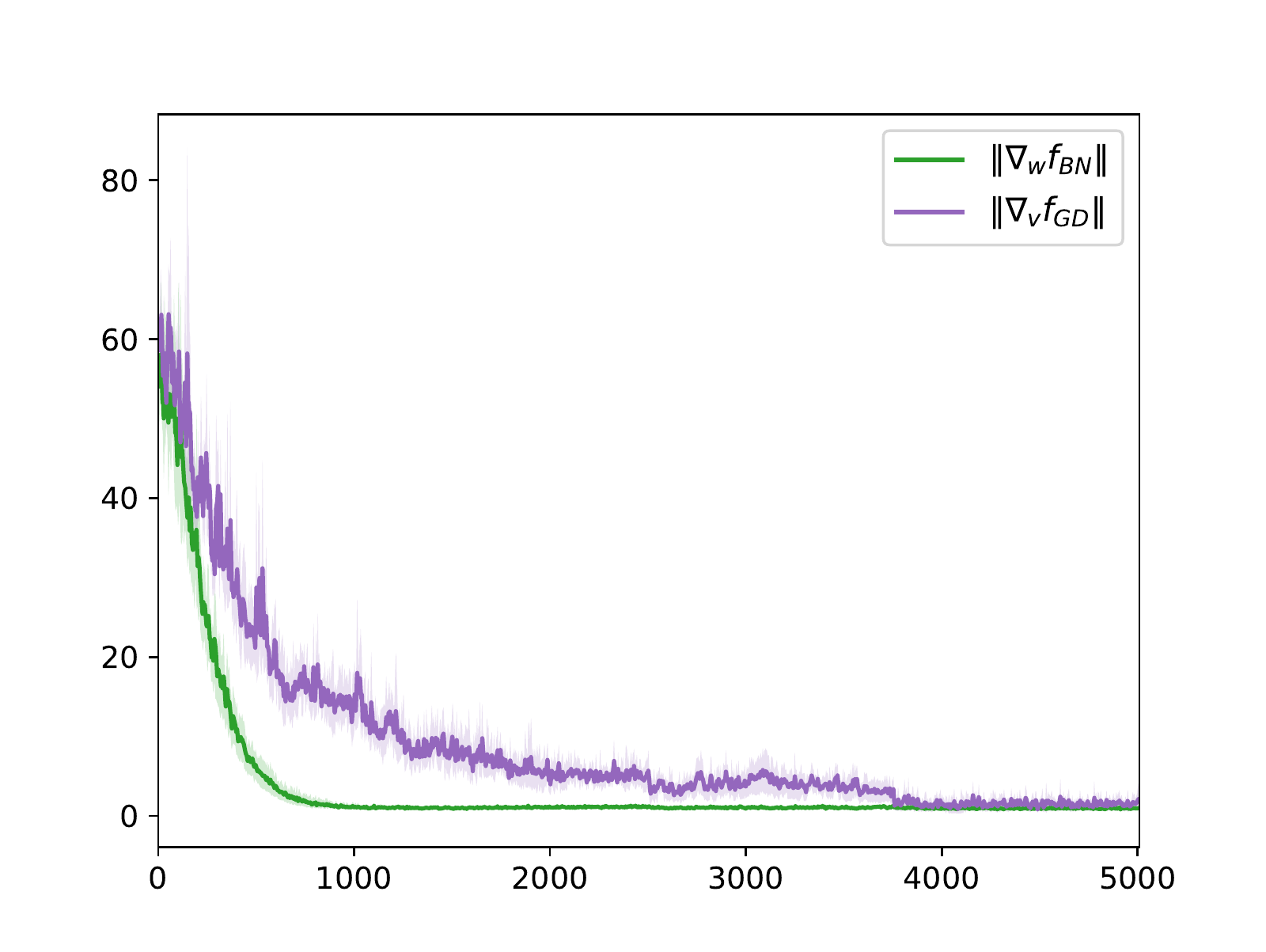} \\ \adjincludegraphics[width=0.5\linewidth, trim={22pt 22pt 30pt 30pt},clip]{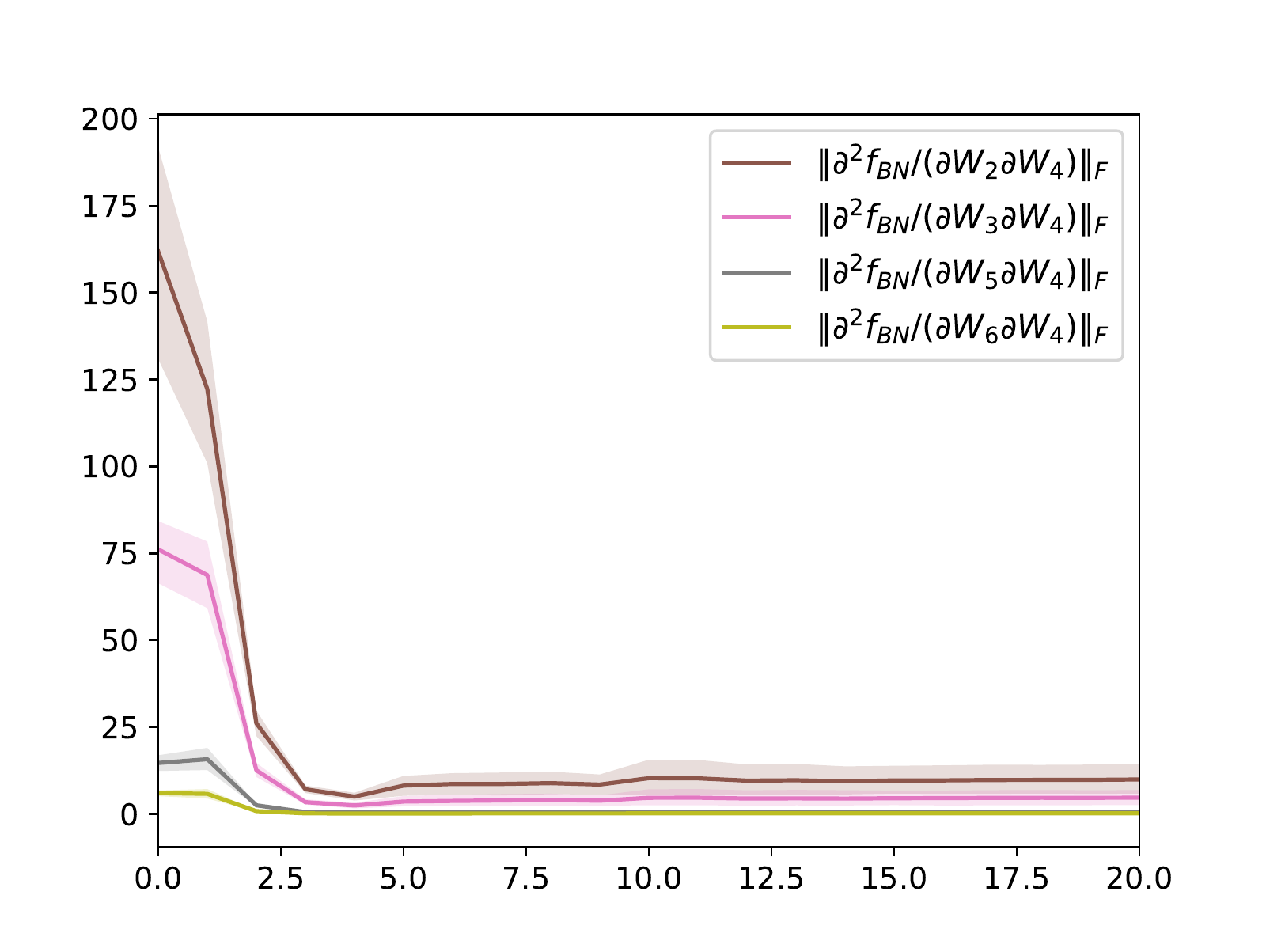}&
            
                \adjincludegraphics[width=0.5\linewidth, trim={22pt 22pt 30pt 30pt},clip]{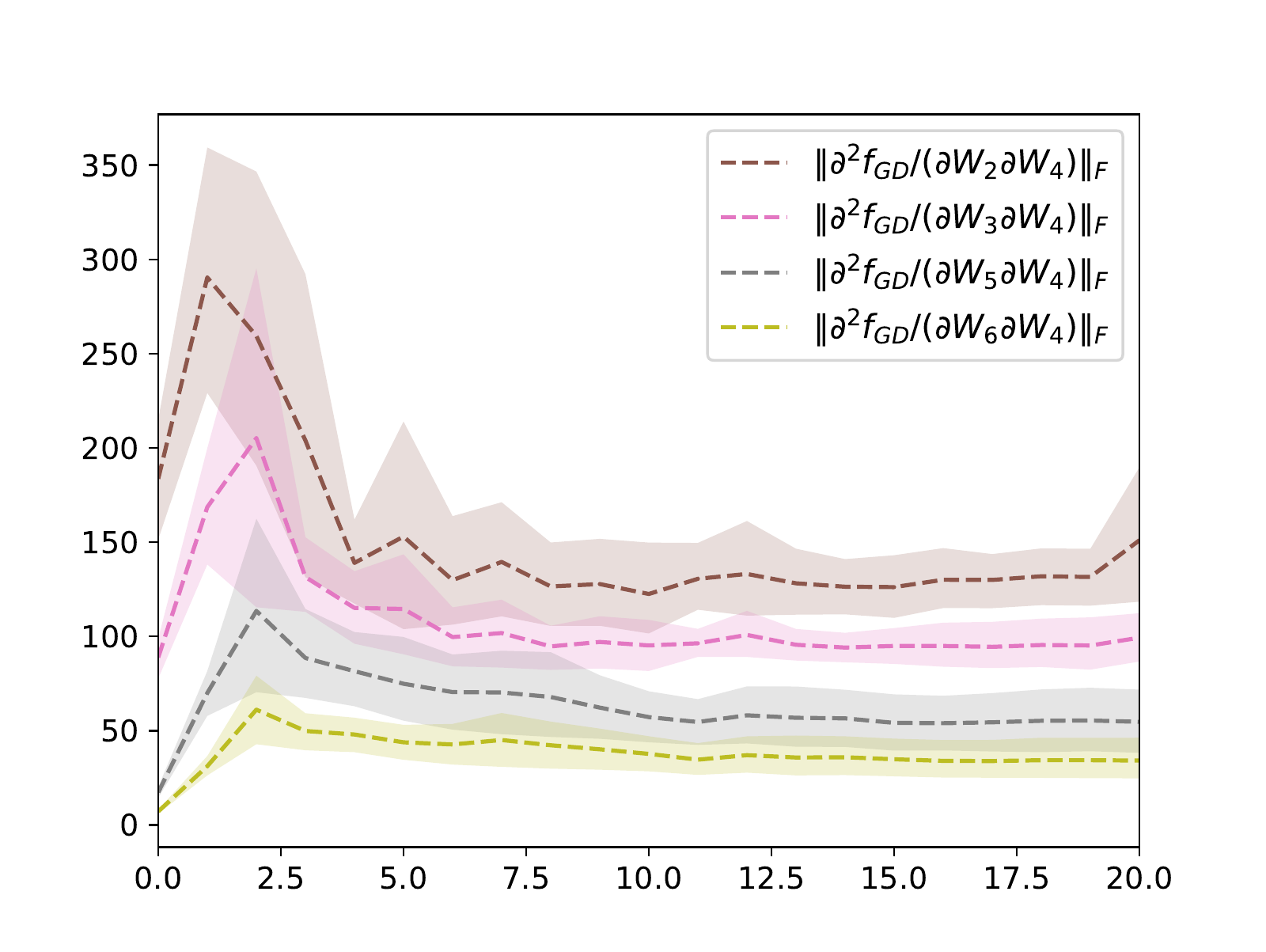}
            

	  \end{tabular}          

          \caption{\footnotesize{(i) Loss, (ii) gradient norm and dependencies between central- and all other layers for BN (iii) and GD (iv) on a 6 hidden layer network with 50 units (each) on the CIFAR10 dataset (5 runs with random initialization, 5000 iterations, curvature information computed each 200th).}}
          \label{fig:exp_NN_linear}
	\end{center}
\end{figure}

\paragraph{Results} Figure~\ref{fig:exp_NN_linear} 
confirms that the directional gradients of the central layer are affected far more by the upstream than by the downstream layers to a surprisingly large extent. Interestingly, this holds even before reaching a critical point. The downstream cross-dependencies are generally decaying for the Batch Normalized network ({\sc Bn}) (especially in the first 1000 iterations where most progress is made) while they remain elevated in the un-normalized network ({\sc Gd}), which suggest that using Batch Normalization layers indeed simplifies the networks curvature structure in $\w$ such that the length-direction decoupling allows Gradient Descent to exploit simpler trajectories in these normalized coordinates for faster convergence.

Of course, we cannot untangle this effect fully from other possible positive aspects of training with {\sc Bn} (see introduction). Yet, the fact that the (de-)coupling increases in the distance to the middle layer (note how earlier (later) layers are more (less) important for $\W_4$) emphasizes the relevance of this analysis particularly for deep neural network structures, where downstream dependencies might vanish completely with depth. This does not only make gradient based training easier but also suggests the possibility of using partial second order information, such as diagonal Hessian approximations (e.g. proposed in \citep{martens2012estimating}).

\section{CONCLUSION}

We took a theoretical approach to study the acceleration provided by Batch Normalization. In a somewhat simplified setting, we have shown that the reparametrization performed by Batch Normalization leads to a provable acceleration of gradient-based optimization by splitting it into subtasks that are easier to solve. In order to evaluate the impact of the assumptions required for our analysis, we also performed experiments on real-world datasets that agree with the results of the theoretical analysis to a surprisingly large extent. 

We consider this work as a first step for two particular directions of future research. First, it raises the question of how to optimally train Batch Normalized neural networks. Particularly, our results suggest that different and adaptive stepsize schemes for the two parameters - length and direction - can lead to significant accelerations. 
Second, the analysis of Section~\ref{sec:OLS} and~\ref{sec:learning_halfspaces} reveals that a better understanding of non-linear coordinate transformations is a promising direction for the continuous optimization community.

\bibliography{batch-norm}

\bibliographystyle{apalike}

\appendix
\addcontentsline{toc}{section}{Appendices}
\onecolumn
\section*{Appendix}

\appendix

\section{LEAST SQUARES ANALYSIS} 
\subsection{Restated result}
Recall that, after normalizing according to \eqref{eq:reparametrization} and using the closed form solution for the optimal scaling factor $g^* := -\left( \u^\top \w\right) /\| \w \|_\cov$, optimizing the ordinary least squares objective can be written as the following minimization problem 
\begin{align*} 
\min_{\w\in \mathbb{R}^d} \left( \rq(\w) := - \frac{\w^\top \u \u^\top \w}{ \w^\top \cov \w}\right).\tag{\ref{eq:LS_rayleigh_formulation} revisited}
\end{align*} 
We consider optimizing the above objective by {\sc Gd} which takes iterative steps of the form
\begin{align} \label{eq:gd_least_squares_app}
\w_{t+1} = \w_t - \eta_t \nabla \rq(\w_t),\end{align} 
where
\begin{equation}\label{eq:LS_gradient_w}
\nabla \rq(\w_t)=- 2 \left(\B\w_t+\rho(\w_t)\cov\w_t \right)/\w_t^\top\cov \w_t.
\end{equation}
Furthermore, we choose 
\begin{equation}\label{eq:stepsize_LS}
\eta_t=\frac{\w_t^\top \cov \w_t}{2 L |\rho(\w_t)|}.
\end{equation}

Our analysis relies on the (weak) data distribution assumption stated in A\ref{as:weak_distribution_assumption}.
It is noteworthy that $\rho(\w)\leq 0, \forall \w \in \mathbb{R}^d\setminus \{\mathbf{0}\}$ holds under this assumption. In the next theorem, we establish a convergence rate of {\sc Gd} in terms of function value as well as gradient norm. 
\begin{framed}
\leastsquaresconvergence*
\end{framed}

The proof of this result crucially relies on the insight that the minimization problem given in \eqref{eq:normalized_least_squares} resembles the problem of maximizing the generalized Rayleigh quotient which is commonly encountered in generalized eigenproblems. We will thus first review this area, where convergence rates are usually provided in terms of the angle of the current iterate with the maximizer, which is the principal eigenvector. Interestingly, this angle can be related to both, the current function value as well as the the norm of the current gradient. We will make use of these connections to prove the above Theorem in Section \ref{sec:proof_of_linear_rate_on_quad}. Although not necessarily needed for convex function, we introduce the gradient norm relation as we will later go on to prove a similar result for possibly non-convex functions in the learning halfspace setting (Theorem \ref{thm:lh_convergence}).

\subsection{Background on eigenvalue problems}\label{sec:backgroundA}
\subsubsection{Rayleigh quotient}
Optimizing the Rayleigh quotient is a classical non-convex optimization problem that is often encountered in eigenvector problems. Let $\B \in \mathbb{R}^{d\times d}$ be a symmetric matrix, then $\w^1 \in \mathbb{R}^d$ is the principal eigenvector of $\B$ if it maximizes 
\begin{align} \label{eq:rayleigh}
q(\w) = \frac{\w^\top \Bm \w}{\w^\top \w}
\end{align} 
and $q(\w)$ is called the Rayleigh quotient. Notably, this quotient satisfies the so-called Rayleigh inequality
\begin{equation*}
\lambda_{\min}(\B) \leq q(\w) \leq \lambda_1(\B),  \quad  \forall \w\in \mathbb{R}\setminus \{\mathbf{0}\},
\end{equation*}
where $\lambda_{\min}(\B)$ and $\lambda_1(\B)$ are the smallest and largest eigenvalue of $\B$ respectively.

Maximizing $q(\w)$ is a non-convex (strict-saddle) optimization problem, where the $i$-th critical point $\v_i$ constitutes the $i$-th eigenvector with corresponding eigenvalue $\lambda_i=q(\w_i)$ (see \citep{absil2009optimization}, Section 4.6.2 for details). It is known that optimizing $q(\w)$ with {\sc Gd} - using an iteration-dependent stepsize -  converges linearly to the principal eigenvector $\v_1$. The convergence analysis is based on 
the "minidimensional" method and yields the following result
\begin{equation}\label{eq:rate_ray}
\frac{\lambda_1-q(\w_t)}{\cosasq{}{\w_t}{\v_1} } \leq \left(1-\frac{\lambda_1-\lambda_2}{\lambda_1-\lambda_{\min}} \right)^{2t} \frac{\lambda_1-q(\w_0)}{\cosasq{}{\w_0}{\v_1} }
\end{equation}

under weak assumptions on $\w_0$. Details as well as the proof of this result can be found in \citep{knyazev1991exact}. 
\subsubsection{Generalized rayleigh quotient}

The reparametrized least squares objective \eqref{eq:LS_rayleigh_formulation}, however, is not exactly equivalent to
\eqref{eq:rayleigh} because of the covariance matrix that appears in the denominator. As a matter of fact, our objective is a special instance of the generalized Rayleigh quotient 
\begin{equation}\label{eq:gen_ray}
    \Tilde{\rho}(\w)=\frac{\w^\top \B \w}{\w^\top \A \w},
\end{equation}
where $\B$ is defined as above and $\A\in \mathbb{R}^{d\times d}$ is a symmetric positive definite matrix. 

Maximizing \eqref{eq:gen_ray} is a generalized eigenproblem in the sense that it solves the task of finding eigenvalues $\lambda$ of the matrix pencil $(\B,\A)$ for which $\det(\B-\lambda \A)=0$, i.e. finding a vector $\v$ that obeys $ \B\v = \lambda \A \v$. Again we have 
$$ \lambda_{\min}(\B,\A) \leq \Tilde{\rho}(\w) \leq  \lambda_1(\B,\A),\quad \forall \w\in \mathbb{R}\setminus\{0\}.$$ 

Among the rich literature on solving generalized symmetric eigenproblems, a {\sc Gd} convergence rate similar to \eqref{eq:rate_ray} has been established in Theorem 6 of \citep{knyazev2003geometric}, which yields

\begin{equation*}
   \frac{\lambda_1-\rho(\w_{t+1})}{\rho(\w_{t+1})-\lambda_2} \leq \left( 1- \frac{\lambda_1-\lambda_2}{\lambda_1-\lambda_{\min}}\right)^{2t} \frac{\lambda_1-\rho(\w_{t})}{\rho(\w_{t})-\lambda_2} ,
\end{equation*}

again under weak assumptions on $\w_0$. \subsubsection{Our contribution}
Since our setting in Eq.~\eqref{eq:LS_rayleigh_formulation} is a minimization task, we note that for our specific choice of $\A$ and $\B$ we have $-\Tilde{\rho}(\w)=\rho(\w)$ and we recall the general result that then $$\min \rho(\w)=-\max(\Tilde{\rho}(\w)).$$
 More importantly, we here have a special case where the nominator of  $\rq(\w)$ has a particular low rank structure. In fact, $\B:=\u \u^\top$ is a rank one matrix. Instead of directly invoking the convergence rate in~\citep{knyazev2003geometric}, this allows for a much simpler analysis of the convergence rate of {\sc Gd} on $\rho(\w)$ since the rank one property yields a simpler representation of the relevant vectors. Furthermore, we establish a connection between suboptimality on function value and the $\cov^{-1}$-norm of the gradient. As mentioned earlier, we need such a guarantee in our future analysis on learning halfspaces which is an instance of a (possibly) non-convex optimization problem.

\subsection{Preliminaries} 
\textbf{Notations } Let $\A$ be a symmetric positive definite matrix. We introduce the following compact notations that will be used throughout the analysis.
\begin{itemize}
    \item $\A$-inner product of $\w, \v \in \R^d$: $\langle\w,\v\rangle_{\A}=\w^\top \A\v$. 
    \item $\A$-norm of $\w \in \R^d$: $\|\w\|_\A = \left( \w^\top \A \w \right)^{1/2}$. 
    \item  $\A$-angle between two vectors $\w, \v \in \R^d   \textbackslash \{\zero\}$: 
    \begin{align} 
    \angle_\A \left( \w,\v \right) := \arccos\left( \frac{\langle \w,\v\rangle_\A}{\|\w\|_\A \|\v\|_\A} \right)
    \end{align}
    \item $\A$-orthogonal projection $\hat{\w}$ of $\w$ to span$\{\v\}$ for $\w, \v \in \R^d  \textbackslash \{ \zero \}$: 
    \begin{align}\label{eq:A-orthogonal}
        \hat{\w} \in \text{span}\{ \v \} \quad \text{with } \quad \langle \w - \hat{\w}, \v\rangle_{\A} = 0 
    \end{align}
    \item $\A$-spectral norm of a matrix $\C \in \R^{d\times d}$: 
    \begin{align} 
    \| \C\|_\A := \max_{\w \in \R^d / \{\zero \} } \frac{\|\C \w \|_\A }{\| \w \|_\A }
    \end{align} 
\end{itemize}
These notations allow us to make the analysis similar to the simple Rayleigh quotient case. For example, the denominator in \eqref{eq:gen_ray} can now be written as $\|\w\|_\A^2$. 

\textbf{Properties } We will use the following elementary properties of the induced terms defined above.
\begin{itemize}
     \item[(P.1)] $ \sinasq{\A}{\w}{\v} = 1 - \cosasq{\A}{\w}{\v}  $
    \item[(P.2)]  If $\hat{\w}$ is the $\A$-orthogonal projection of $\w$ to span$\{\v\}$, then it holds that 
    \begin{align}   \label{eq:properity_2} 
    \cosasq{\A}{\w}{\v}  = \frac{\|\hat{\w}\|_\A^2 }{\|\w\|_\A^2}, \quad \sin \angle_{\A}{\w}{\v}  = \frac{\|\w- \hat{\w}\|_\A }{\|\w\|_\A} 
    \end{align}
    \item[(P.3)] The $\A$-spectral norm of a matrix can be written in the alternative form 
    \begin{align}
        \| \C \|_{\A} & = \max_{\w \in \R^d \setminus \{\zero \} } \frac{\|\A^{1/2} \C \w \|_2 }{\|\A^{1/2} \w \|_2 } \\ & = \max_{\w \in \R^d \setminus \{\zero \} } \frac{\|\left(\A^{1/2} \C\A^{-1/2} \right)  \A^{1/2}  \w \|_2 }{\|\A^{1/2} \w \|_2 } \\ 
        & = \|\A^{1/2}\C\A^{-1/2} \|_2
    \end{align}
    \item[(P.4)] Let $ \C \in \R^{d \times d}$ be a square matrix and $\w \in \R^d $, then the following result holds due to  the definition of $\A$-spectral norm  
    \begin{align} 
        \| \B \w \|_\A  \leq  \| \B \|_\A \| \w \|_\A
    \end{align} 
\end{itemize}

\subsection{Characterization of the LS minimizer} \label{sec:minimizer_of_LS} 
By setting the gradient of \eqref{eq:least_squares_objective} to zero and recalling the convexity of $\fls$ we immediately see that the minimizer of this objective is
\begin{equation}\label{eq:LS_mininizer}
\Tilde{\w}^*:=-\cov^{-1}\u.
\end{equation}
Indeed, one can easily verify that $\Tilde{\w}^*$ is also an eigenvector of the matrix pair $(\B,\cov)$ since
\begin{equation}
\begin{aligned} \label{eq:generalized_eigenvector_solution}
 &\B \Tilde{\w}^* = \u \u^\top (-\cov^{-1} \u) = -\|\u\|^2_{\cov^{-1}} \u \\=&  ( \|\u\|^2_{\cov^{-1}})  \cov (-\cov^{-1}\u)=\lambda_1\cov \Tilde{\w}^*
\end{aligned}
\end{equation}

where  $\lambda_1:= \| \u \|_{\cov^{-1}}^2$ is the corresponding generalized eigenvalue. The associated eigenvector with $\lambda_1$ is  
\begin{equation}\label{eq:dominant_eigenvector}\v_1:= \Tilde{\w}^*/\|\u\|_{\cov^{-1}}.
 \end{equation}
Thereby we extend the normalized eigenvector to an $\cov$-orthogonal basis $(\v_1, \v_2, \dots,\v_d)$ of $\R^d$ such that 
\begin{equation}
\begin{aligned} \label{eq:Aorthonganlity}
\langle \v_i,\v_j \rangle_\cov=
\begin{cases}
0,\; i\not=j\\
1,\; i=j
\end{cases}
\end{aligned}
\end{equation}

holds for all $i,j$. Let $\V_2 := [ \v_2, \v_3, \dots, \v_d ]$ be the matrix whose $(i-1)$-th column is $\v_i, i\in \{2,\dots,d\}$. The matrix $\B$ is orthogonal to the matrix $\V_2$ since
\begin{equation}
\begin{aligned} \label{eq:orthognality_B}
&\B \V_2 = \u \u^\top \V_2 =  \u \u^\top \cov^{-1} \cov \V_2 \\  \stackrel{\eqref{eq:dominant_eigenvector}}{= }  & \| \cov^{-1} \u \|_\cov \u \left(\underbrace{-\v_1^\top \cov \V_2}_{\zero} \right) ,
\end{aligned}
\end{equation}

which is a zero matrix due to the $\A$-orthogonality of the basis (see Eq.~\eqref{eq:Aorthonganlity}). As a result, the columns of $\V_2$ are eigenvectors associated with a zero eigenvalue. Since $\v_1$ and $\V_2$ form an $\cov$-orthonormal basis of $\mathbb{R}^d$ no further eigenvalues exist. We can conclude that any vector $\w^* \in \text{span}\{\v_1\}$ is a minimizer of the reparametrized ordinary least squares problem as presented in \eqref{eq:LS_rayleigh_formulation} and the minimum value of $\rq$ relates to the eigenvalue as $$\lambda_1 =  - \min_\w \rho(\w).$$

\textbf{Spectral representation of suboptimality }
Our convergence analysis is based on the angle between the current iterate $\w_t$ and the leading eigenvector $\v_1$, for which we recall property (P.1). We can  express $\w \in \R^d$ in the $\cov$-orthogonal basis that we defined above: 
\begin{align} \label{eq:w_in_S_basis}
\w = \alpha_1 \v_1 + \V_2 \balpha_2, \quad \alpha_1 \in \R, \quad  \balpha_2 \in \R^{d-1}
\end{align}
and since $\alpha_1 \v_1$ is the $\cov$-orthogonal projection of $\w$ to span$\{\v_1\}$, the result of (P.2) implies 
\begin{align} \label{eq:cosine_squared_expansion}
\cosasq{\cov}{\w}{\v_1} = \frac{\|\alpha_1 \v_1 \|^2_\cov}{\| \w \|^2_\cov} = \frac{\alpha_1^2}{\| \w \|^2_\cov }. 
\end{align}

Clearly this metric is zero for the optimal solution $\v_1$ and else bounded by one from above. To justify it is a proper choice, the next proposition proves that suboptimality on  $\rq$, i.e. $\rq(\w) - \rq(\v_1)$, relates directly to this angle. 
\begin{proposition} \label{pro:function_value_expansion}
The suboptimality of $\w$ on $\rq(\w)$ relates to $\sinasq{\cov}{\w}{\v_1}$ as 
\begin{align} \label{eq:sine_suboptimality_connection}
    \rq(\w) - \rq(\v_1)= \lambda_1 \sinasq{\cov}{\w}{\v_1},
\end{align}
where $\rq(\v_1)=\lambda_1$. This is equivalent to 
\begin{align} \label{eq:cosine_objective_connection}
    \rq(\w) = -\lambda_1 \cosasq{\cov}{\w}{\v_1}.
\end{align}
\end{proposition}
\begin{proof} 
We use the proposed eigenexpansion of Eq.~\eqref{eq:w_in_S_basis} to rewrite
\begin{align} 
\B \w  = (\alpha_1 \B \v_1 + \B \V_2 \balpha_2)  
 \stackrel{\eqref{eq:orthognality_B}}{=} \alpha_1 \B \v_1   
 \stackrel{\eqref{eq:generalized_eigenvector_solution}}{=} \alpha_1 \lambda_1 \cov \v_1 \label{eq:Bw_expansion}
\end{align} 
and replace the above result into $\rq(\w)$. Then 
\begin{align*} 
\rq(\w)  &= -\frac{\w^\top \B \w }{\w^\top \cov \w }  \stackrel{\eqref{eq:Bw_expansion}}{=} - \alpha_1 \lambda_1 \frac{(\alpha_1 \v_1 + \V_2 \balpha_2)^\top  \cov \v_1}{\| \w \|_\cov^2} \\& \stackrel{\eqref{eq:Aorthonganlity}}{=}  - \lambda_1 \frac{\alpha_1^2}{\| \w \|_\cov^2 }  \stackrel{\eqref{eq:cosine_squared_expansion}}{=} - \lambda_1 \cosasq{\cov}{\w}{\v_1},
\end{align*} 
which proves the second part of the proposition. The first follows directly from property (P.1).
\end{proof}
\textbf{Gradient-suboptimality connection }
Fermat's first-order optimality condition implies that the gradient is zero at the minimizer of $\rho(\w)$. Considering the structure of $\rho(\w)$, we propose a precise connection between the norm of gradient and suboptimality.  
Our analysis relies on the representation of the gradient $\nabla \rq(\w)$ in the $\cov$-orthonormal basis $\{ \v_1,\dots,\v_d \}$ which is described in the next proposition.
\begin{proposition} \label{pro:gradient_expansion}
Using the $\cov$-orthogonal basis as given in Eq.~\eqref{eq:Aorthonganlity}, the gradient vector can be expanded as 
\begin{equation}
\begin{aligned} \label{eq:gradient_expansion}
\| \w\|^{2}_{\cov} \nabla \rq(\w)/2 =&  -\lambda_1 \alpha_1 \sinasq{\cov}{\w}{\v_1} \cov \v_1 \\&+ \lambda_1 \cosasq{\cov}{\w}{\v_1} \cov \V_2 \balpha_2 
\end{aligned} 
\end{equation}

\end{proposition}
\begin{proof}
The above derivation is based on two results: (i) $\v_1$ is an eigenvector of $(\B,\cov)$ and (ii) the representation of $\rho(\w)$ in Proposition~\ref{pro:function_value_expansion}. We recall the definition of $\nabla  \rho(\w)$ in \eqref{eq:LS_gradient_w} and write
\begin{equation}
\begin{aligned} 
\nabla \rq(\w) \| \w\|^2_{\cov}/2 & = -\rq(\w) \cov \w - \B \w \\ 
& 
\stackrel{\eqref{eq:Bw_expansion}}{=}-\rq(\w) \cov \left( \alpha_1 \v_1 + \V_2 \balpha_2 \right) + \lambda_1 \alpha_1 \cov \v_1 \\ 
& \stackrel{\eqref{eq:cosine_objective_connection}}{=} -(1-  \cosasq{\cov}{\w}{\v_1})\lambda_1 \alpha_1 \cov \v_1 \\+& \lambda_1 \cosasq{\cov}{\w}{\v_1} \cov \V_2 \balpha_2  \\ 
& = -\lambda_1 \alpha_1 \sinasq{\cov}{\w}{\v_1} \cov \v_1 \\+& \lambda_1 \cosasq{\cov}{\w}{\v_1} \cov \V_2 \balpha_2 
\end{aligned}
\end{equation}

\end{proof}
Exploiting the gradient representation of the last proposition, the next proposition establishes the connection between suboptimality and the $\cov^{-1}$-norm of gradient $\nabla_\w \rq(\w)$. 
\begin{proposition} \label{pro:gradient_norm_expansion}
Suppose that $\rho(\w) \neq 0$, then the $\cov^{-1}$-norm of the gradient $ \nabla \rq(\w)$ relates to the suboptimality as 
\begin{align} \label{eq:gradient_norm_expansion}
\| \w \|_\cov^{2} \| \nabla \rq(\w) \|^2_{\cov^{-1}}/(4|\rho(\w)|) =  \rq(\w)  - \rq(\v_1) 
\end{align} 
\end{proposition}
\begin{proof}
Multiplying the gradient representation in Proposition~\ref{pro:gradient_expansion} by $\cov^{-1}$ yields  
\begin{align*} 
\cov^{-1} \nabla \rq(\w) \| \w \|^2_{\cov}/2 = &-\lambda_1 \alpha_1 \sinasq{\cov}{\w}{\v_1} \v_1 \\&+ \lambda_1 \cosasq{\cov}{\w}{\v_1} \V_2 \balpha_2.
\end{align*} 
By combining the above result with the $\cov$-orthogonality of the basis $(\v_1,\V_2)$, we derive the (squared) $\cov^{-1}$-norm of the gradient as
\begin{equation}
\begin{aligned} \label{eq:T1_T2_expansion}
 & \nabla  \rq(\w)^\top \cov^{-1} \nabla \rq(\w)/4 \stackrel{\eqref{eq:Aorthonganlity}}{=} T_1  + T_2,\\ 
 & T_1 := \|\w \|_\cov^{-4}   \lambda_1^2 \alpha_1^2 \sinapow{4}{\cov}{\w}{\v_1},\\ &  T_2 := \|\w \|_\cov^{-4}  \lambda_1^2 \cosapow{4}{\cov}{\w}{\v_1}  \| \balpha_2 \|^2.
\end{aligned}
\end{equation}

It remains to simplify the terms $T_1$ and $T_2$. For $T_1$,
\begin{align*} 
\frac{\|\w \|_\cov^{2} }{\lambda_1^2 \sinapow{4}{\cov}{\w}{\v_1}} T_1&=   \|\w \|_\cov^{-2}  \alpha_1^2  \stackrel{\eqref{eq:cosine_squared_expansion}}{=}  \cosasq{\cov}{\w}{\v_1}.
\end{align*} 
Similarly, we simplify $T_2$: 
\begin{align*} 
\frac{ \|\w \|_\cov^2  }{\lambda_1^2 \cosapow{4}{\cov}{\w}{\v_1} } T_2& = \|\w \|_\cov^{-2}    \| \balpha_2 \|^2 \\ & = \|\w \|_\cov^{-2}    \left( \| \w \|_\cov^2 - \alpha_1^2 \right)  \\ 
 & = 1 - \frac{\alpha_1^2}{\|\w \|_{\cov}^2} \\ 
 & \stackrel{\eqref{eq:cosine_squared_expansion}}{=} 1 - \cosasq{\cov}{\w}{\v_1}  \\ 
 & = \sinasq{\cov}{\w}{\v_1}
\end{align*} 

Replacing the simplified expression of $T_1$ and $T_2$ into Eq.~\eqref{eq:T1_T2_expansion} yields 
\begin{align*}
   \| \w \|_\cov^{2} \| \nabla_\w \rq(\w) \|^2_{\cov^{-1}}/4  =&  \lambda_1^2 \cosasq{\cov}{\w}{\v_1} \sinasq{\cov}{\w}{\v_1} \left( \sinasq{\cov}{\w}{\v_1} +  \cosasq{\cov}{\w}{\v_1}\right) \\
     =&  \lambda_1^2 \cosasq{\cov}{\w}{\v_1} \sinasq{\cov}{\w}{\v_1}, \\ 
     \stackrel{\eqref{eq:cosine_objective_connection}}{=} &| \rho(\w) | \lambda_1 \| \sinasq{\cov}{\w}{\v_1} \\ 
     \stackrel{\eqref{eq:sine_suboptimality_connection}}{=}& |\rho(\w)| \left( \rho(\w) - \rho(\v_1)\right). 
    \end{align*}
    A rearrangement of terms in the above equation concludes the proof.  
\end{proof}
\subsection{Convergence proof}\label{sec:proof_of_linear_rate_on_quad}
We have seen: suboptimality in $\rq(\w)$ directly relates to $\sinasq{\cov}{\w}{\v_1}$ for all $\w\setminus \{\mathbf{0}\}$. In the next lemma we prove that this quantity is strictly decreased by repeated {\sc Gd} updates at a linear rate.  
\begin{lemma} \label{lem:convergence_metric_decay}
Suppose that Assumption~\ref{as:weak_distribution_assumption} holds and consider {\sc Gd} ({\sc Gd}) steps on \eqref{eq:LS_rayleigh_formulation} with stepsize $\eta_t = \frac{\| \w_t \|_\cov^2 }{2 L|\rho(\w_t)|}$. Then, for any $\w_0\in \mathbb{R}^d$ such that $\rho(\w_0)<0$, the updates of Eq.~\eqref{eq:gd_least_squares_app} yield the following linear convergence rate

\begin{align} \label{eq:convergence_metric_decay}
\sinasq{\cov}{\w_t}{\v_1} \leq \left(1- \frac{\mu}{L}\right)^{2t} \sinasq{\cov}{\w_{0}}{\v_1}
\end{align} 
\end{lemma}
\begin{proof}
To prove the above statement, we relate the sine of the angle of a given iterate $\w_{t+1}$ with $\v_1$ in terms of the previous angle $\angle(\w_t,\v_1)$. Towards this end, we assume for the moment that $\rho(\w_t)\not=0$ but note that this naturally always holds whenever $\rho(\w_0)\not=0$,\footnote{as we will show later by induction.} such that the angle relation can be recursively applied through all $t\geq 0$ to yield Eq.~\eqref{eq:convergence_metric_decay}.

(i) We start by deriving an expression for $\sin \angle_\cov(\v_1,\w_t)$. By \eqref{eq:A-orthogonal} and the definition $\rho(\w_t)$, we have that $-\rho(\w_t)\w_t$ is the $\cov$-orthogonal projection of $\cov^{-1}\B\w_t$ to $\text{span}\{\w_t\}$. Indeed,
\begin{align*}
&\langle\cov^{-1}\B\w_t+\rho(\w_t)\w_t,\w_t\rangle_\cov \\=&\w_t^\top\B \cov^{-1}\cov \w_t - \frac{\w_t^\top \B \w_t}{\w_t^\top \cov \w_t}\w_t^\top \cov \w_t=0. 
\end{align*}

Note that $\cov^{-1}\B\w_t=\left(\cov^{-1}\u \right)\left(\u^\top\w_t\right)$ is a nonzero multiple of $\v_1$ and thus  $\sin\angle_\cov(\cov^{-1} \B \w_t , \w_t) =\sin \angle_\cov(\v_1,\w_t)$.

Thus, by (P.2) we have
\begin{equation}\label{eq:angle}
\begin{aligned}
   \sin \angle_\cov(\v_1,\w_t) &= \sin \angle_\cov(\cov^{-1}\B\w_t,\w_t)\\&=\frac{\|\cov^{-1}\B\w_t+\rho(\w_t)\w_t \|_\cov}{\|\cov^{-1}\B\w_t \|_\cov}.
   \end{aligned}
\end{equation}

(ii) We now derive an expression for $\sin \angle_\cov(\v_1,\w_{t+1})$. Let $a_{t+1}\in\mathbb{R}$ such that $a_{t+1}\w_{t+1} \in\mathbb{R}^d$ is the $\cov$-orthogonal projection of $\cov^{-1}\B\w_t$ to $\text{span}\{\w_{t+1}\}$. Then
\begin{equation}\label{eq:sin_w_t_cross_terms_zero}
\begin{aligned}
    &\langle \cov^{-1}\B\w_t-a_{t+1}\w_{t+1},\w_{t+1}\rangle_\cov = 0\\
    \Rightarrow\;& \langle \cov^{-1}\B\w_t-a_{t+1}\w_{t+1},(a_{t+1}+\rho(\w_t))\w_{t+1}\rangle_\cov=0.
\end{aligned}
\end{equation}

By the Pythagorean theorem and \eqref{eq:sin_w_t_cross_terms_zero}, we get

\begin{equation}
\begin{aligned}\label{eq:pythago}
\|\cov^{-1}\B\w_t+\rho(\w_t)\w_{t+1} \|_\cov^2&=\| \cov^{-1}\B\w_t-a_{t+1}\w_{t+1}\|_\cov^2\\&+ \|(a_{t+1}+\rho(\w_t))\w_{t+1} \|_\cov^2\\& \geq \| \cov^{-1}\B\w_t-a_{t+1}\w_{t+1}\|_\cov^2.
\end{aligned}
\end{equation}

Hence, again by (P.2)
\begin{equation}\label{eq:angle_plus}
\begin{aligned}
       \sin \angle_\cov(\v_1,\w_{t+1})&=  \sin \angle_\cov(\cov^{-1}\B\w_t,\w_{t+1})\\& =\frac{\|\cov^{-1}\B\w_t-a_{t+1}\w_{t+1} \|_\cov}{\|\cov^{-1}\B\w_t \|_\cov} \\&\stackrel{\eqref{eq:pythago}}{\leq}  
 \frac{\|\cov^{-1}\B\w_t+\rho(\w_t)\w_{t+1} \|_\cov}{\|\cov^{-1}\B\w_t \|_\cov}.
\end{aligned}
\end{equation}

(iii) To see how the two quantities on the right hand side of \eqref{eq:angle} and \eqref{eq:angle_plus} relate, let us rewrite the {\sc Gd} updates from Eq.~\eqref{eq:gd_least_squares_app} as follows
\begin{equation}\label{eq:myequation}
\begin{aligned}
\rho(\w_t) \w_{t+1}&=\rho(\w_t) \w_{t} - \frac{\cov}{L}\left(\cov^{-1}\B\w_t+\rho(\w_t)\w_t\right)\\
\Leftrightarrow  \cov^{-1}\B\w_t + \rho(\w_t)\w_{t+1}&=\cov^{-1}\B\w_t+\rho(\w_t)\w_{t}-  \frac{\cov}{L}\left(\cov^{-1}\B\w_t+\rho(\w_t)\w_t\right)\\
\Leftrightarrow \cov^{-1}\B\w_t + \rho(\w_t)\w_{t+1} &= \left(\mathbf{I}-\frac{\cov}{L}\right)\left(\cov^{-1}\B\w_t+\rho(\w_t)\w_t \right).
\end{aligned}
\end{equation}

By taking the $\cov$-norm we can conclude

\begin{equation}\label{eq:gamma_bound}
\begin{aligned}
&\|\cov^{-1}\B\w_t+\rho(\w_t)\w_{t+1} \|_\cov \\  \stackrel{\eqref{eq:myequation}}{\leq}& \|\mathbf{I}-\frac{\cov}{L}\|_\cov \cdot \| \cov^{-1}\B\w_t+\rho(\w_t)\w_t \|_\cov\\ \leq & \left(1- \frac{\mu}{L}\right) \| \cov^{-1}\B\w_t+\rho(\w_t)\w_t \|_\cov,
\end{aligned}
\end{equation}

where the first inequality is due to property (P.4) of the $\cov$-spectral norm and the second is due to Assumption \eqref{as:weak_distribution_assumption} and (P.3) , which allows us to bound the latter in term of the usual spectral norm as follows
\begin{equation}
\begin{aligned} \label{eq:spectral_bound_result}
 &\| \Im - \cov/L \|_\cov  \stackrel{(\text{P.3})}{=} \| \cov^{1/2} \left( \Im - \cov/L \right) \cov^{-1/2} \|_2 \\ =& \| \Im - \cov/L \|_2 \stackrel{\eqref{eq:covariance_spectral_bound}}{\leq} 1 - \mu/L.
\end{aligned} 
\end{equation}

(iv) Combining the above results yields the desired bound

\begin{equation}\label{eq:sin_relation}
    \begin{aligned}
         \sin \angle_\cov(\v_1,\w_{t+1})& \stackrel{\eqref{eq:angle_plus}}{\leq} \frac{\|\cov^{-1}\B\w_t+\rho(\w_t)\w_{t+1} \|_\cov}{\|\cov^{-1}\B\w_t \|_\cov}\\
         & \stackrel{\eqref{eq:gamma_bound}}{\leq} \left(1- \frac{\mu}{L}\right) \frac{\| \cov^{-1}\B\w_t+\rho(\w_t)\w_t \|_\cov}{\|\cov^{-1}\B\w_t \|_\cov} \\&\stackrel{\eqref{eq:angle} }{=} \left(1- \frac{\mu}{L}\right) \sin \angle_\cov(\v_1,\w_t).
    \end{aligned}
\end{equation}

(v) Finally, we show that the initially made assumption $\rho(\w_{t})\not=0$ is naturally satisfied in all iterations. First, $\rho(\w_0)<0$ by assumption. Second, assuming $\rho(\w_{\hat{t}})<0$ for an arbitrary $\hat{t}\in \mathbb{N}^+$ gives that the above analysis (i-iv) holds for $\hat{t}+1$ and thus \eqref{eq:sin_relation} together with \eqref{eq:sine_suboptimality_connection} and the fact that $\lambda_1>0$ give
\begin{equation*}
\begin{aligned}
\rho(\w_{\hat{t}+1})&= - \lambda_1(1-\sin^2(\v_1,\w_{\hat{t}+1}))\\& < - \lambda_1(1-\sin^2(\v_1,\w_{\hat{t}})) =\rho(\w_{\hat{t}})<0,
\end{aligned}
\end{equation*}
where the last inequality is our induction hypothesis. Thus $\rho(\w_{\hat{t}+1})<0$ and we can conclude by induction that $\rho(\w_t)<0, \forall t \in \mathbb{N}^+$.

As a result, \eqref{eq:sin_relation} holds $\forall t \in \mathbb{N}^{+}$, which (applied recursively) proves the statement \eqref{eq:convergence_metric_decay}.
\end{proof}

We are now ready to prove Theorem \ref{thm:least_squares_convergence}.

\textit{Proof of Theorem \eqref{thm:least_squares_convergence}:} $\;$ By combining the results of Lemma \ref{lem:convergence_metric_decay} as well as Proposition \ref{pro:function_value_expansion} and \ref{pro:gradient_norm_expansion}, we can complete the proof of the Theorem \ref{thm:least_squares_convergence} as follows
\begin{equation*}
\begin{aligned}
&    \| \w_t \|_\cov^{2} \| \nabla_\w \rq(\w_t) \|^2_{\cov^{-1}}/|4 \rho(\w_t)| & \\ \stackrel{\eqref{eq:gradient_norm_expansion}}{=}&  \left(  \rq(\w_t) -   \rq(\w^*) \right) \\ 
     \stackrel{\eqref{eq:sine_suboptimality_connection}}{=}& \lambda_1 \sinasq{\cov}{\w_t}{\v_1} \\ 
     \stackrel{\eqref{eq:convergence_metric_decay}}{\leq}&  \left(1- \frac{\mu}{L}\right)^{2t} \lambda_1 \sinasq{\cov}{\w_{0}}{\v_1} \\ 
     \stackrel{\eqref{eq:sine_suboptimality_connection}}{=}& \left(1- \frac{\mu}{L}\right)^{2t} \left( \rq(\w_0) - \rq(\w^*) \right).
\end{aligned}
\end{equation*}

\section{LEARNING HALFSPACES ANALYSIS} \label{se:learnin_lh_app}

In this section, we provide a convergence analysis for Algorithm \ref{alg:gdnp} on the problem of learning halfspaces
\begin{align*} 
\min_{\Tilde{\w}\in\mathbb{R}^d} \left( \flh(\Tilde{\w}) := \E_{y,\x} \left[ \varphi(-y \x^\top \Tilde{\w})\right] = \E_{\z} \left[ \varphi(\z^\top \Tilde{\w}) \right] \right). \tag{\ref{eq:halfspace_problem} revisited}
\end{align*} 
As before, we reparametrize $\Tilde{\w}$ by means of the covariance matrix $\cov$ as 

\begin{align*}\tag{\ref{eq:reparametrization} revisited}
    \Tilde{\w} := g \w/\| \w \|_\cov.
\end{align*}

We assume that the domain of $\flh(\Tilde{\w})$ is $\R^d$ but exclude $\zero$ such that \eqref{eq:reparametrization} is always well defined. Thus, the domain of the new parameterization is $(\w,g) \in \left( \R^d \setminus \{\zero \} \right) \otimes \R \subset \R^{d+1}$.
\subsection{Preliminaries}
Recall the normality assumption on the data distribution.
\strongassumption* 
Under the above assumption we have that for a differentiable function $g(\z): \mathbb{R}^d \to \mathbb{R}$ the following equality holds 
\begin{align} \label{eq:stein_lemma}
 \E_{\z} \left[ g(\z) \z \right] = \E_{\z} \left[ g(\z) \right] \u +  (\cov - \u \u^\top)  \E_\z \left[ \nabla_{\z} g(\z) \right].
\end{align}
This result, which can be derived using a simple application of integration by parts, is called \textit{Stein's lemma} \citep{landsman2008stein}. In the next lemma, we show that this allows us to simplify the expression of the gradient of Eq.~\ref{eq:halfspace_problem}.
\begin{lemma}[restated result from \citep{erdogdu2016scaled}] \label{lem:gradient_expression_lh}
Under the normality assumption on the data distribution~(Assumption \ref{as:strong_distribution_assumption}), the gradient of $f_{LH}$ (Eq.~\ref{eq:halfspace_problem}) can be expressed as 
\begin{align}\label{eq:lh_gradient_after_stein}
    \nabla_{\Tilde{\w}} \flh(\Tilde{\w}) = c_1(\Tilde{\w}) \u + c_2(\Tilde{\w}) \cov \Tilde{\w},  
\end{align} 
where $c_i \in \mathbb{R}$ depends on the $i$-th derivative of the loss function denoted by $\varphi^{(i)}$ as
\begin{align*} 
c_1(\Tilde{\w}) & = \E_\z \left[ \varphi^{(1)}(\z^\top \Tilde{\w})\right] - \E_\z \left[ \varphi^{(2)}(\z^\top \Tilde{\w})\right] (\u^\top \Tilde{\w}), \\
c_2(\Tilde{\w})  
& =  \E_\z \left[ \varphi^{(2)}(\z^\top \Tilde{\w})\right]
\end{align*} 
\end{lemma} 
\begin{proof}
The gradient of $f_{LH}$ can be written as follows
\begin{equation}
    \nabla f_{LH}= \E \left[\varphi^{(1)}(\z^\top\Tilde{\w}) \z \right].
\end{equation}
A straight forward application of Stein's lemma (Eq. \eqref{eq:stein_lemma}) yields 
\begin{equation}
    \nabla f_{LH}= \E\left[\varphi^{(1)}(\z^\top\Tilde{\w})\right]\u  + (\cov-\u\u^\top) \E\left[\varphi^{(2)}(\z^\top\Tilde{\w}) \right] \Tilde{\w},
\end{equation}
which --after rearrangement -- proves the result. See detailed derivation in~ \citep{erdogdu2016scaled}.
\end{proof}

In addition to the assumption on the data distribution, the proposed analysis also requires a rather weak assumption on $\flh$ and loss function $\varphi$.
\regassumption*
\smoothassumption*
Recall that $\zeta$-smoothness of $\flh$, which is mentioned in the last assumption,  implies that the gradient of $\flh$ is $\zeta$-Lipschitz, i.e. 
\begin{align} \label{eq:smoothness_assumption}
 \| \nabla \flh(\Tilde{\w}_1) - \nabla  \flh(\Tilde{\w}_2) \| \leq \zeta \| \Tilde{\w}_1 - \Tilde{\w}_2 \|. 
\end{align}

\subsection{Global characterization}
Here, we prove a result about a global property of the solution of the problem of learning halfspaces. 
\lhcharaterization* 
\begin{proof}
Setting the gradient of the objective $\flh$ as given in Eq. \eqref{eq:lh_gradient_after_stein} to zero directly gives the result. 
\end{proof}

\subsection{Established Convergence Rate}

Based on this assumption, we derive a linear convergence rate for {\sc Gdnp} presented in Algorithm~\ref{alg:gdnp}. We first restate the convergence guarantee before providing a detailed proof.
\begin{framed}
\learninghalfspacesconvergence
\end{framed}
\subsubsection{Proof sketch}
As mentioned earlier, the objective $\flh$ on Gaussian inputs has a particular \emph{global} property. Namely, all its critical points are aligned along the same direction. The key idea is that $\cov$-reparameterization provides this global information to a local optimization method through an elegant length-direction decoupling. This allows {\sc Gdnp} to mimic the behaviour of Gradient Descent on the above mentioned Rayleigh quotient for the directional updates and thereby inherit the linear convergence rate. At the same time, the scaling factor can easily be brought to a critical point by a fast, one dimensional search algorithm. We formalize and combine these intuitions in a detailed proof below.
\subsubsection{Gradient in the normalized parameterization} 
Since {\sc Gdnp} relies on the normalized parameterization, we first need to derive the gradient of the objective in this parameterization
\begin{align} \label{eq:learning_halfspace_normalized}  
\min_{\w,g} \left(\flh(\w,g) := \E_\z \left[ \varphi\left(g \frac{ \z^\top \w}{\| \w \|_\cov} \right) \right]\right).
\end{align} 

Straight forward calculations yield the following connection between the gradient formulation in the original parameterization $\nabla_{\Tilde{\w}} \flh =\E_\z[ \varphi^{(1)}(\Tilde{\w}^\top \z)\z]$ and the gradient in the normalized parameterization 
\begin{equation}
    \begin{aligned} \label{eq:normalized_gradient}
\nabla_\w \flh(\w,g) = g \A_{\w} \nabla_{\Tilde{\w}} \flh(\Tilde{\w}),\\ \partial_g \flh(\w,g) = \w^\top \nabla_{\Tilde{\w}} \flh(\Tilde{\w})/\| \w \|_\cov 
\end{aligned} 
\end{equation}

where 
\begin{equation}\label{eq:A_in_nullspace_of_Sw}
    \A_\w := \Im/\| \w \|_\cov  -  \cov \w \w^\top/\|\w\|_\cov^{3}.
\end{equation}
Note that the vector $\cov \w$ is orthogonal to the column space  of $\A_\w$ since 
\begin{align} \label{eq:cov_w_A_orthogonality}
\A_\w \cov \w = \left( \cov \w - \frac{\| \w \|_\cov^2 }{\| \w \|^2_\cov }\cov \w \right)/\| \w \|_{\cov} = 0 .
\end{align}
We will repeatedly use the above property in our future analysis. 
In the next lemma, we establish a connection between the norm of gradients in different parameterizations. 
\begin{proposition} \label{prop:gradient_norm_in_different_parameterization}
Under the reparameterization~\eqref{eq:reparametrization}, the following holds:
\begin{equation}
\begin{aligned} \label{eq:gradient_norm_in_different_parameterization} 
    \| \nabla_{\Tilde{\w}} f(\Tilde{\w}) \|^2_{\cov^{-1}} =&  \| \w \|_\cov^2 \| \nabla_\w f(\w,g)\|^2_{\cov^{-1}}/g^2 \\&+ \left(\partial_g f(\w,g)\right)^2
\end{aligned}
\end{equation}

\end{proposition}
\begin{proof}
We introduce the vector $\q_1 = \cov \w/\| \w \|_\cov $  that has unit $\cov^{-1}$-norm, i.e.  $\|\q_1\|_{S^{-1}}=1$. 

According to Eq.~\eqref{eq:cov_w_A_orthogonality},
\begin{align*} 
\A_\w \q_1 = 0 
\end{align*}
holds. Now, we extend this vector to an $\cov^{-1}$-orthogonal basis $\{\q_1,\q_2, \dots, \q_d \}$ of $\R^d$ such that 
\begin{align*}
   \langle \q_i, \q_i \rangle_{\cov^{-1}} =1,\; \forall i \text{ and }  \langle \q_i, \q_j \rangle_{\cov^{-1}}  = 0,\; \forall i\neq j.
\end{align*}
 Let $\Qm_2$ be a matrix whose columns are $\{\q_2, \dots, \q_d \}$. The choice of $\q_1$ together with $\cov ^{-1}$-orthogonality of the basis imply that $\w$ is orthogonal to $\Qm_2$: 
\begin{align*}
    \w^\top \q_j = \|\w \|_{\cov}  \langle \q_1, \q_j \rangle_{\cov^{-1}}  = 0, \forall j \neq 1
\end{align*}
Consider the gradient expansion in the new basis, i.e.   
\begin{align*} 
\nabla_{\Tilde{\w}} f(\Tilde{\w}) = \alpha_1 \q_1 + \Qm_2 \balpha_2, \;\; \| \nabla_{\Tilde{\w}} f(\Tilde{\w}) \|^2_{\cov^{-1}} = \alpha_1^2 + \| \balpha_2 \|_2^2 
\end{align*} 
Plugging the above expansion into Eq.~\eqref{eq:normalized_gradient} yields 
\begin{align} \label{eq:normalized_gradient_eigenexpan_lh}
 \nabla_\w f(\w,g) = g \Qm_2 \balpha_2/\| \w \|_\cov, \;\; \partial_g f(\w,g) = \alpha_1
\end{align}
hence the $\cov^{-1}$-norm of the directional gradient in the new parameterization is 
\begin{align} \label{eq:directional_eigenexpan_lh}
    \| \nabla_\w f(\w,g) \|^2_{\cov^{-1}} = g^2 \| \balpha_2 \|_2^2/\| \w \|_\cov^2
\end{align}
Therefore, one can establish the following connection between the $\cov^{-1}$-norm of gradient in the two different parameterizations: 
\begin{equation*}
\begin{aligned}
    \| \nabla_{\Tilde{\w}} f(\Tilde{\w}) \|^2_{\cov^{-1}} & = \alpha_1^2 + \| \balpha_2 \|_2^2  \\ 
    & \stackrel{\eqref{eq:normalized_gradient_eigenexpan_lh}}{=} \left( \partial_g f(\w,g) \right)^2 + \| \balpha_2 \|^2\\ 
    & \stackrel{\eqref{eq:directional_eigenexpan_lh}}{=}\left( \partial_g f(\w,g) \right)^2 + \frac{ \| \w \|^2_{\cov} \| \nabla_\w f(\w,g) \|^2_{\cov^{-1}}}{g^2}
\end{aligned}
\end{equation*}

\end{proof}

The above lemma allows us to first analyze convergence in the $(\w,g)$-parameterization and then we relate the result to the original $\Tilde{\w}$-parameterization in the following way: Given that the iterates $\{ \w_t, g_t \}_{t \in \N^+}$ converge to a critical point of $\flh(\w,g)$, one can use Eq.~\eqref{eq:normalized_gradient_eigenexpan_lh} to prove that $\Tilde{\w}_t = g_t \w_t/\| \w_t \|_\cov$ also converges to a critical point of $\flh(\Tilde{\w})$.

For the particular case of learning halfspaces with Gaussian input, the result of Lemma~\ref{lem:gradient_expression_lh} allows us to write the gradient $\nabla \flh$ as 
\begin{align}\label{eq:grad_lh_wtild}
    \nabla_{\Tilde{\w}} \flh(\Tilde{\w}) = c_1(\Tilde{\w}) \u + c_2(\Tilde{\w}) \cov \Tilde{\w}, 
\end{align}
where the constants $c_1$ and $c_2$ are determined by the choice of the loss. Replacing this expression in Eq.~\eqref{eq:normalized_gradient} yields the following formulation for the gradient in normalized coordinates
\begin{align} \label{eq:gd_normal_lh}
\nabla_\w \flh(\w) = g c_1(\w,g) \A_\w \u + g c_2(\w,g) \A_\w \cov \w,
\end{align} 
where $c_i(\w,g) = c_i(\Tilde{\w}(\w,g))$. Yet, due to the specific matrix $\A_\w$ that arises when reparametrizing according to \eqref{eq:reparametrization}, the vector $\cov \w$ is again in the kernel of $\A_\w$ (see Eq.~\eqref{eq:cov_w_A_orthogonality}) and hence
\begin{align} \label{eq:gradient_hl_after_normalization}
\nabla_\w \flh(\w) = g c_1(\w,g) \A_\w \u.  
\end{align} 

 \subsubsection{Convergence of the scalar $g$}
\begin{lemma}[Convergence of scalar] \label{lem:convergence_in_g}
Under the assumptions of Theorem~\ref{thm:lh_convergence}, in each iteration $t\in \mathbb{N}^+$ of {\sc Gdnp} (Algorithm \ref{alg:gdnp}) the partial derivative of $f_{LH}$ as given in Eq. \eqref{eq:learning_halfspace_normalized} converges to zero at the following linear rate
\begin{align} \label{eq:convergence_lh_in_g}
    \left( \partial_g \flh(\w_t, a_t^{(\tsc)}) \right)^2 \leq  2^{-\tsc}\zeta  | b^{(0)}_t - a_t^{(0)} |/\mu^{2}.
\end{align}
\end{lemma}
 \begin{proof}
 
According to Algorithm \ref{alg:gdnp}, the length of the search space for $g$ is cut in half by each bisection step and thus reduces to  
\begin{align*} 
|a_t^{(\tsc)} - b_t^{(\tsc)}|\leq  2^{-\tsc} | b^{(0)}_t - a_t^{(0)} |
\end{align*} 
after $\tsc$ iterations. The continuity of $\partial_g \flh$ given by Assumption \ref{as:smooth_assumptions} and the fact that Algorithm \ref{alg:gdnp} guarantees $\partial_g \flh(\w_t, a_t^{(m)}) \cdot  \partial_g \flh(\w_t, b_t^{(m)})< 0,\; \forall m \in \mathbb{N}^+ $ allow us to conclude that there exists a root $g^*$ for $\partial_{g} \flh$ between $a_t^{(\tsc)}$ and $b_t^{(\tsc)}$ for which  
\begin{align*}
    |a_t^{(\tsc)} - g^*| \leq 2^{-\tsc} | b^{(0)}_t - a_t^{(0)} |
\end{align*}
holds. 

The next step is to relate the above distance to the partial derivative of $\flh(\w,g)$ w.r.t $g$. Consider the compact notation $\w'_t = \w_t/ \| \w_t \|_\cov $. Using this notation and the gradient expression in Eq.~\eqref{eq:normalized_gradient}, the difference of partial derivatives can be written as 
\begin{equation}\label{eq:g_conv_proof_1}
\begin{aligned} 
&\left( \partial_g \flh(\w_t, a_t^{(\tsc)}) \right)^2 \\ =& 
\left( \partial_{g} \flh(\w_t,a_t^{(\tsc)}) - \partial_{g} \flh(\w_t,g^*) \right)^2 \\
 =&\left(  \left( \nabla_{\Tilde{\w}} \flh(a_t^{(\tsc)}\w'_t) -\nabla_{\Tilde{\w}} \flh(g^*\w'_t)  \right)^\top \w'_t  \right)^2.
\end{aligned} 
\end{equation}

Using the smoothness assumption on $\flh$ we bound the above difference as follows
\begin{equation}\label{eq:g_conv_proof_2}
\begin{aligned}
&\left(  \left( \nabla_{\Tilde{\w}} \flh(a_t^{(\tsc)}\w'_t) -\nabla_{\Tilde{\w}} \flh(g^*\w'_t)  \right)^\top \w'_t  \right)^2 \\ \leq& \| \w'_t \|^2 \|  \nabla_{\Tilde{\w}} \flh(a_t^{(\tsc)}\w'_t) -\nabla_{\Tilde{\w}} \flh(g^*\w'_t)  \|^2  \\ 
 \stackrel{\eqref{eq:smoothness_assumption}}{\leq}& \zeta \| \w'_t \|^2   \| a_t^{(\tsc)}\w'_t - g^*\w'_t\|^2 \\ 
 \leq& \zeta \| \w'_t \|^4 (a_t^{(\tsc)} - g^*)^2 \\ 
 \leq & \zeta \| \w_t \|^4 \| \w_t \|^{-4}_\cov (a_t^{(\tsc)} - g^*)^2 \\ 
 \stackrel{\eqref{eq:covariance_spectral_bound}}{\leq} &\zeta \mu^{-2} (a_t^{(\tsc)} - g^*)^2,
\end{aligned}
\end{equation}
where the last inequality is due to Assumption \ref{as:weak_distribution_assumption}.

Combining \eqref{eq:g_conv_proof_1} and \eqref{eq:g_conv_proof_2} directly yields

\begin{align} \label{eq:convergence_lh_in_g}
    \left( \partial_g \flh(\w_t, a_t^{(\tsc)}) \right)^2 \leq  2^{-\tsc}\zeta  | b^{(0)}_t - a_t^{(0)} |/\mu^{2},
\end{align}
which proves the assertion.
\end{proof}
\subsubsection{Directional convergence}
\begin{lemma}[Directional convergence]\label{l:directional_conv_lh}
Let all assumptions of Theorem~\ref{thm:lh_convergence} hold. Then, in each iteration $t\in \mathbb{N}^+$ of {\sc Gdnp} (Algorithm \ref{alg:gdnp}) with the following choice of stepsizes
\begin{align} \label{eq:step_size}
s_t := s(\w_t,g_t) = - \frac{\| \w_t \|_\cov^{3}}{ L g_t h(\w_t,g_t)  } ,\quad t=1,\dots,T_d
\end{align}
where  
\begin{equation}
\begin{aligned}
h(\w_t,g_t) := &\E_\z \left[ \varphi'\left(\Tilde{\w}_t\right)  \right] \left(\u^\top\w_t\right)  - \E_\z \left[ \varphi''(\Tilde{\w}_t) \right] \left(\u^\top \w_t\right)^2 \neq 0.
\end{aligned}
\end{equation}

The norm of the gradient w.r.t. $\w$ of $\flh$ as in Eq.~\eqref{eq:learning_halfspace_normalized} converges at the following linear rate

\begin{align*}
    \| \w_t \|^2_\cov \| \nabla_\w \flh(\w_t,g_t) \|^2_{\cov^{-1}} \leq (1-\frac{\mu}{L})^{2t} \Phi^2 g_t^2 \left( \rq(\w_0) - \rq^*  \right). \label{eq:convergenc_direction_lh}
\end{align*}
\end{lemma}
\begin{proof}

The key insight for this proof is a rather subtle connection between the gradient of the reparametrized least squares objective (Eq. \eqref{eq:LS_rayleigh_formulation}) and the directional gradient of the learning halfspace problem (Eq. \eqref{eq:learning_halfspace_normalized}):
\begin{equation}
\begin{aligned} \label{eq:ls_lh_gradient_connection}
\nabla_\w \flh(\w,g) & \stackrel{\eqref{eq:gradient_hl_after_normalization}}{=} g c_1(\w,g) \A_\w \u \\ 
& = gc_1(\w,g) \left( \u - \left( \frac{\w^\top \u}{\| \w \|_\cov^2 }\right) \cov \w  \right)/ \| \w \|_{\cov}\\ 
& =g \frac{c_1(\w,g)}{\u^\top \w} \left( \u^\top \w \u - \frac{(\w^\top \u)^2}{\|\w \|^2_\cov } \cov \w \right)/\| \w\|_\cov \\
& = g\frac{c_1(\w,g)\| \w \|_\cov}{\u^\top \w} \left( \B\w_t +\rho(\w) \cov \w \right)/\| \w\|_\cov^2 \\
& \stackrel{\eqref{eq:LS_gradient_w}}{=}  -g\left( \frac{c_1(\w,g)\| \w \|_\cov}{2 \u^\top \w} \right) \nabla_\w \rho(\w) .
\end{aligned} 
\end{equation}
Therefore, the directional gradients $\nabla_\w \rho(\w)$ and $\nabla_\w \flh(\w,g)$ align in the same direction for all $\w \in \mathbb{R}^d\setminus \{\mathbf{0}\}$. Based on this observation, we propose a stepsize schedule for {\sc Gdnp} such that we can exploit the convergence result established for least squares in Theorem~\ref{thm:least_squares_convergence}. The iterates $\{\w_t\}_{t\in\N^+}$ of {\sc Gdnp} on $\flh$ can be written as 
\begin{equation}
\begin{aligned} \label{eq:LHsteps_w}
\w_{t+1} & = \w_t - s_t \nabla_\w \flh(\w_t,g) \\ 
 & \stackrel{\eqref{eq:ls_lh_gradient_connection}}{=} \w_t + s_t \left( \frac{g_tc_1(\w_t,g_t)\| \w_t \|_\cov}{2 \u^\top \w_t} \right) \nabla_\w \rq(\w_t).
\end{aligned}
\end{equation}

The stepsize choice of Eq.\eqref{eq:step_size} guarantees that
\begin{align*} 
s_t \left( \frac{g_tc_1(\w_t,g_t)\| \w_t \|_\cov}{2 \u^\top \w_t} \right) = -\frac{\| \w_t \|_\cov^4}{(2 L(\w_t^\top \u)^2)} \stackrel{\eqref{eq:stepsize_LS}}{=} -\eta_t.
\end{align*}

Thus,  Eq. \eqref{eq:LHsteps_w} can be rewritten as
\begin{align*} 
\w_{t+1} = \w_t - \eta_t \nabla_\w \rq(\w_t),
\end{align*} 
which exactly matches the {\sc Gd} iterate sequence of Eq.~\eqref{eq:gd_least_squares} on $\rho(\w)$. At this point, we can invoke the result of Theorem  ~\ref{thm:least_squares_convergence} to establish the following convergence rate:
\begin{equation}
\begin{aligned}
    & \| \w_t \|^2_\cov \| \nabla_\w \flh(\w_t,g_t) \|^2_{\cov^{-1}} & \\ \stackrel{\eqref{eq:ls_lh_gradient_connection}}{=}& \| \w_t \|^2_\cov c_1^2(\w_t,g_t) g_t^2  \| \nabla_\w \rq(\w_t) \|^2_{\cov^{-1}}/\left( 2 (\u^\top \w_t)/\| \w_t \|_\cov \right)^2  \\ 
     \stackrel{\eqref{as:reg_assumptions}}{\leq}& \Phi^2  \| \w_t \|^2_\cov g_t^2 \| \nabla_\w \rq(\w_t) \|^2_{\cov^{-1}} /\left( 2(\u^\top \w_t)/\| \w_t \|_\cov \right)^2 \\ 
     \leq &\Phi^2 \| \w_t \|_\cov^2  g_t^2\| \nabla_\w\rq(\w_t) \|^2_{\cov^{-1}}/|4\rq(\w_t)| \\ 
     \stackrel{\eqref{eq:suboptimality_gradientnorm_thm}}{\leq}& \Phi^2 g_t^2\left( \rq(\w_t) - \rq^* \right) \\
    \stackrel{\eqref{eq:suboptimality_convergence_thm}}{\leq} &(1-\mu/L)^{2t} \Phi^2  g_t^2 \left( \rq(\w_0) - \rq^*  \right) . \label{eq:convergenc_direction_lh}
\end{aligned}
\end{equation}
\end{proof}

\subsubsection{Combined convergence guarantee} 
Using Proposition~\ref{prop:gradient_norm_in_different_parameterization} and combining the results obtained for optimizing the directional and scalar components, we finally obtain the following convergence guarantee:
\begin{align*}
    \| \nabla_{\Tilde{\w}} \flh(\Tilde{\w}_{\tdir}) \|^2_{\cov^{-1}} & \stackrel{\eqref{eq:gradient_norm_in_different_parameterization}}{=}  \|\w_{\tdir}\|_\cov^2 \| \nabla_\w \flh(\w_{\tdir},g_{\tdir}) \|^2_{\cov^{-1}} /g_{\tdir}^2 \\&+ \left( \partial_g \flh(\w_{\tdir},g_{\tdir}) \right)^2  \\ 
    & \stackrel{\eqref{eq:convergenc_direction_lh}}{\leq} (1-\mu/L)^{2\tdir} \Phi^2  \left( \rq(\w_0) - \rq^* \right) \\&+ \left( \partial_g \flh(\w_{\tdir},g_{\tdir}) \right)^2 \\ 
    & \stackrel{\eqref{eq:convergence_lh_in_g}}{\leq }(1-\mu/L)^{2\tdir} \Phi^2  \left( \rq(\w_0) - \rq^*  \right)\\&+ 2^{-\tsc}\zeta  | b^{(0)}_{\tdir} - a_ {\tdir}^{(0)} |/\mu^{2}.
\end{align*}

\begin{algorithm}[t]
\begin{algorithmic}[1]
\small{
\STATE \textbf{Input:} $\tsc$, $a_t^{(0)}$, $b_t^{(0)}$, $f$
\STATE Choose $a_t^{(0)}$ and $b_t^{(0)}$ such that $\partial_g f(a_t^{(0)},\w_t) \cdot \partial_g f(b_t^{(0)},\w_t)>0$. 

\FOR{$m=0,\dots,\tsc$} 
\STATE $c = (a^{(m)} + b^{(m)})/2$ 
\IF{ $\partial_g f(c,\w_t) \cdot \partial_g f(a^{(m)},\w_t) > 0 $ }
\STATE $a^{(m+1)} \leftarrow c$
\ELSE
\STATE $b^{(m+1)} \leftarrow c$
\ENDIF
\ENDFOR
\STATE $g \leftarrow a^{(\tsc)}$
\STATE \textbf{return} $g$}
\end{algorithmic}
\caption{Bisection} 
\label{alg:bisection}
\end{algorithm}

\subsubsection{A word on Weight Normalization}\label{sec:word_on_WN}
The improved convergence rate for Batch Normalization (Theorem \ref{thm:lh_convergence}) relies heavily on the fact that normalizing and backpropagating through the variance term resembles splitting the optimization task into a length- and directional component. As mentioned in the introduction, this feature is also present in Weight Normalization and it is thus an obvious question, whether {\sc Wn} can achieve a similar convergence rate. From a theoretical perspective, we were not able to prove this which is essentially due to the subtle difference in how the normalization is done: While {\sc Bn} normalizes the parameters to live on the $\cov$-sphere, {\sc Wn} brings all parameters to the unit sphere.

The fast directional convergence rate of {\sc Bn} on Learning Halfspaces is essentially inherited from the fast convergence of Gradient Descent (with adaptive stepsize) on the Rayleight Quotient. This can be seen in the proof of Lemma \ref{l:directional_conv_lh} where we specifically use the fact that $\nabla_{\w} f_{LH}$ and $\nabla_{\w} \rho$  align in the same direction. To prove this fact we need two ingredients (i) Stein's Lemma which gives us the expression of $\nabla_{\Tilde{\w}} f_{LH}$ as in Eq. \eqref{eq:grad_lh_wtild} and (ii) the specific reparametrization of {\sc Bn} as in Eq. \eqref{eq:reparametrization} which lets us express the directional part of this gradient as $\nabla_{\Tilde{\w}} f_{LH} = g \A_{\w} \nabla_{\Tilde{\w}} \flh(\Tilde{\w})$. As we shall see, the second part is very specific to the reparametrization done by {\sc Bn}, which gives certain properties of $\A_{\w}=\Im/\| \w \|_\cov  -  \cov \w \w^\top/\|\w\|_\cov^{3}$ that then yield Eq. \eqref{eq:gradient_hl_after_normalization} which is simply a scaled version of the Rayleigh Quotient gradient (see Eq. \eqref{eq:LS_gradient_w}). This fact arises particularly because (i) $\A_{\w}$ is orthogonal to $\cov\w$ (see Eq. \eqref{eq:A_in_nullspace_of_Sw}) and (ii) $\A_{\w}$ and  $\nabla_{\w} \rho$ both involve division by the $\cov$-norm. Both properties are not given for the version of $\A_{\w,\text{WN}}=\Im/\| \w \|_2  -  \w \w^\top/\|\w\|_2^{3}$ that would arise when using Weight Normalization so the proof strategy breaks because we no longer match the gradients $\nabla_{\w} f_{LH}$ and $\nabla_{\w} \rho$ . 

That said, we observe similar empirical convergence behaviour in terms of suboptimality for {\sc Bn} and {\sc Wn} (without any adaptive stepsizes, see Section \ref{sec:exp_i}) but as can be seen on the right of Figure \ref{fig:loss_glms} the path that the two methods take can be very different. We thus leave it as an interesting open question if other settings and proof strategies can be found where fast rates for {\sc Wn} are provable. 
\section{NEURAL NETWORKS} \label{sec:nn_appendix}
Recall the training objective  of the one layer MLP presented in Section~\ref{sec:neural_networks}: 
\begin{align*}
    \min_{\Tilde{\W},\Theta} \bigg( \fnn(\Tilde{\W},\Theta) := \E_{y,\x} \left[ \ell\left(-y F(\x,\Tilde{\W},\Theta)\right) \right]    \bigg), \tag{Revisited~\ref{eq:NN_problem}}
\end{align*}
where
\begin{equation*}
     F(\z, \Tilde{\W},\Theta) ):=\sum_{i=1}^m \theta^{(i)} \varphi(\z^\top \Tilde{\w}^{(i)}). 
\end{equation*}
Figure~\ref{fig:nn} illustrates the considered architecture in this paper. 

\def\layersep{2.5cm}
\begin{figure}[h!]

    \begin{center}
    \begin{tikzpicture}[shorten >=1pt,->,draw=black!50, node distance=\layersep]
        \tikzstyle{every pin edge}=[<-,shorten <=1pt]
        \tikzstyle{neuron}=[circle,fill=black!25,minimum size=17pt,inner sep=0pt]
        \tikzstyle{input neuron}=[neuron, fill=green!50];
        \tikzstyle{output neuron}=[neuron, fill=red!50];
        \tikzstyle{hidden neuron}=[neuron, fill=blue!50];
        \tikzstyle{hidden neuron 2}=[neuron, fill=blue!50];
    
        \tikzstyle{annot} = [text width=4em, text centered]
    
        \foreach \name / \y in {1,...,2}
            \node[input neuron] (I-\name) at (0,-1.35*\y) {$\z_\y$};
        \node at (0,-1.35*2.5) {$\vdots$};
        \foreach \name / \y in {3,...,3}
            \node[input neuron] (I-\name) at (0,-1.35*\y-0.4) {$\z_d$};
    
        \path[yshift=0.5cm]
                    node[hidden neuron] (H-1) at (\layersep,-25-1 cm) {$\varphi_1$};
        \path[yshift=0.5cm]
                    node[hidden neuron 2] (H-2) at (\layersep,-25-2.3 cm) {$\varphi_2$};
        
        \path[yshift=0.5cm]
                    node[hidden neuron 2] (H-3) at (\layersep,-25-4.05 cm) {$\varphi_m$};
         \node at (\layersep,-25-2.5cm) {$\vdots$};
        
        \node[output neuron,  right=of $(H-2)$]  (O) {$F$};
    
        \foreach \source in {1,...,3}
            \foreach \dest in {1,...,3}{
                    \ifnum \source=1
                        \path (I-\source)  edge (H-\dest) ;
                    \else
                        \path (I-\source)  edge
                        (H-\dest)  ;
                    \fi
                }
        \path (I-1)  edge node [above] {\textcolor{darkgray}{$\Tilde{\w}^{(1)}$}} (H-1) ;
        \path (I-3)  edge node [above] {\textcolor{darkgray}{$\Tilde{\w}^{(m)}$}} (H-3) ;
        \foreach \source in {1,...,3}{
             \ifnum \source=3
                \path (H-\source) edge node [above] {\textcolor{darkgray}{$\theta^{(m)}$}} (O);
             \else
                \path (H-\source) edge node [above] {\textcolor{darkgray}{$\theta^{(\source)}$}} (O);
             \fi
            }
        \node[annot,above of=H-1, node distance=1cm] (hl) {\footnotesize{Hidden layer}};
        \node[annot,left of=hl] {\footnotesize{Input layer}};
        \node[annot,right of=hl] {\footnotesize{Output layer}};
    \end{tikzpicture}
    \end{center}
\caption{Neural network architecture considered in this paper. }
    \label{fig:nn}
\end{figure}
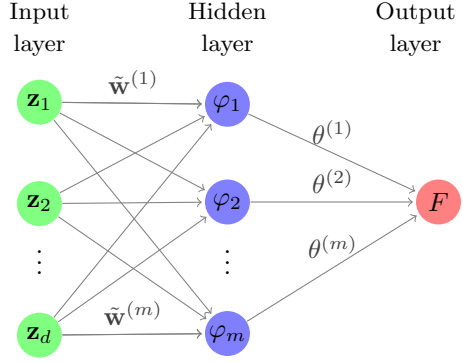
Since the activation function is assumed to be an odd function ($\tanh$), this choice allows us to equivalently rewrite the training objective as
\begin{align} \label{eq:NN_problem_z}
\min_{\Tilde{\W},\Theta} \bigg(
    \fnn(\Tilde{\W},\Theta) =  \E_{\z}  \left[ \ell( F\left(\z, \Tilde{\W},\Theta) \right) \right]
    \bigg).
\end{align}

 By means of Assumption~\ref{as:strong_distribution_assumption} and Stein's lemma (Eq. \eqref{eq:stein_lemma}) we can simplify the gradient w.r.t $\Tilde{\w}_i$ as follows
\begin{align} \label{eq:gradient_nn}
  \nabla_{\Tilde{\w}^{(i)}} \fnn(\Tilde{\W},\Theta)/\theta^{(i)} & =  \E_\z \left[ \ell^{(1)}(F(\z,\Tilde{\W},\Theta))\varphi^{(1)}(\z^\top \Tilde{\w}^{(i)}) \z \right] \\ 
  & = \alpha^{(i)} \u + \beta^{(i)} \cov \Tilde{\w}^{(i)} + \sum_{j=1}^m \gamma^{(i,j)} \cov \Tilde{\w}^{(j)},
\end{align}
where the scalars $\alpha^{(i)}$,  $\beta^{(i)}$ and $\gamma^{(i,j)}$ are defined as
\begin{align} 
\beta^{(i)} & := \E_\z \left[ \ell^{(1)}(F(\z,\Tilde{\W},\Theta)) \varphi^{(2)}(\z^\top \Tilde{\w}^{(i)}) \right] \\ 
\gamma^{(i,j)} & := \theta^{(j)} \E_\z \left[ \ell^{(2)}(F(\z,\Tilde{\W},\Theta)) \varphi^{(1)}(\z^\top \Tilde{\w}^{(i)}) \varphi^{(1)}(\z^\top \Tilde{\w}^{(j)}) \right] \\ 
\alpha^{(i)} & :=\E_\z \left[ \ell^{(1)}(F(\z,\Tilde{\W},\Theta))\varphi^{(1)}(\z^\top \Tilde{\w}^{(i)}) \right] - \sum_{j=1}^m \gamma^{(i,j)}(\u^\top \Tilde{\w}^{(j)}),
\end{align} 
where $l^{(i)}(\cdot)\in \mathbb{R}$ and $\varphi^{(i)}(\cdot)\in \mathbb{R}$ represent the $i$-th derivative of $l(\cdot)$ and $\varphi(\cdot)$ with respect to their input $(\cdot)$.

\subsection{Characterization of the objective}
Interestingly, the normality assumption induces a particular global property on $\fnn (\Tilde{\W})$. In fact, all critical weights $\Tilde{\w}_i$ align along one single line in $\mathbb{R}^d$, which only depends on \textit{incoming} information into the hidden layer. This result is formalized in the next lemma.

\globalpropertynn*

\begin{proof}
Recall the gradient of $\fnn$ as given in Eq.~\eqref{eq:gradient_nn}. Computing a first order critical point requires setting the derivatives of all units to zero which amounts to solving the following system of non-linear equations: 
\begin{equation} \label{eq:wide_layer_soe}
\begin{aligned}
(1) \quad & \alpha^{(1)} \u + \beta^{(1)} \cov \hat{\w}^{(1)} + \sum_{j=1}^m  \gamma^{(1,j)} \cov \hat{\w}^{(j)} &= 0 \\ 
(2) \quad & \alpha^{(2)} \u + \beta^{(2)} \cov \hat{\w}^{(2)} + \sum_{j=1}^m \gamma^{(2,j)} \cov \hat{\w}^{(j)}  &= 0 \\
&\hspace{3cm}\vdots &\\
(m) \quad & \alpha^{(m)} \u + \beta^{(m)} \cov \hat{\w}^{(m)} + \sum_{j=1}^m \gamma^{(m,j)} \cov \hat{\w}^{(j)}  & = 0 ,
\end{aligned}
\end{equation}

where each row (i) represents a system of $d$ equations. 

\textbf{Matrix formulation of system of equations} Let us rewrite \eqref{eq:wide_layer_soe} in matrix form. Towards this end, we define

\begin{equation*}
\begin{aligned}
    \U&=[\u,\u,\ldots,\u ] \in \mathbb{R}^{d\times m},\\ \hat{\w}&=[\hat{\w}^{(1)},\hat{\w}^{(2)},\ldots,\hat{\w}^{(m)} ] \in \mathbb{R}^{d\times m}
\end{aligned}
\end{equation*}
as well as 
\begin{equation*}
\begin{aligned}
\A &= \text{diag}\left(\alpha^{(1)},\alpha^{(2)},\ldots,\alpha^{(m)} \right) \in \mathbb{R}^{m\times m},\\  \B &= \text{diag}\left(\beta^{(1)},\beta^{(2)},\ldots,\beta^{(m)} \right) \in \mathbb{R}^{m\times m}
\end{aligned}
\end{equation*}
and
\begin{equation*}
\mathbf{\Gamma}= \begin{bmatrix} 
\gamma^{(1,1)}&  \gamma^{(1,2)} & \ldots & \gamma^{(1,m)} \\ 
\gamma^{(2,1)}&  \gamma^{(2,2)} & \ldots & \gamma^{(2,m)}\\
& &\ddots & \\
\gamma^{(m,1)}&  \gamma^{(m,2)} & \ldots & \gamma^{(m,m)}.
 \end{bmatrix}.
\end{equation*}
Note that $\mathbf{\Gamma}=\mathbf{\Gamma^\top}$ since $\gamma^{(i,j)}=\gamma^{(j,i)}, \forall i,j$.

\textbf{Solving the system of equations} Using the notation introduced above, we can write \eqref{eq:wide_layer_soe} as follows

\begin{equation}
\begin{aligned}
&\U\A+\cov \hat{\W} \B + \cov \hat{\W} \mathbf{\Gamma} = 0\\
\Leftrightarrow & \cov \hat{\W}(\B + \mathbf{\Gamma})=-\U\A \\
\Leftrightarrow & \hat{\W}=-\cov^{-1} \U \underbrace{\A(\B + \mathbf{\Gamma})^{ \dagger}}_{:=\D},
\end{aligned}
\end{equation}

where $ (\B + \mathbf{\Gamma})^{ \dagger}$ is the pseudo-inverse of $(\B + \mathbf{\Gamma})$.

As a result 
\begin{equation*}
\begin{aligned}
&[\hat{\w}^{(1)},\hat{\w}^{(2)},\ldots,\hat{\w}^{(m)}]\\=&- [\cov^{-1}\u,\cov^{-1}\u,\ldots,\cov^{-1}\u] [\d_1,\d_2,\ldots,\d_m]
\end{aligned}
\end{equation*}
and hence the critical points of the objective are of the following type 
\begin{equation}\label{eq:NN_critical_points}
\hat{\w}^{(i)} = - \cov^{-1}\U\d_1= -\left( \sum_{k=1}^m \d_i^{(k)}\right) \cov^{-1}\u.    
\end{equation}
\end{proof} 
\subsection{Possible implications for deep neural networks}\label{sec:deepnets}
From Eq.~\eqref{eq:opt_direction_NN} in the Lemma \ref{lem:critical_point_characterization_nn} we can conclude that the optimal direction of any $\hat{\w}^{(i)}$ is independent of the corresponding output weight $\theta^{(i)}$, which only affects $\hat{\w}^{(i)}$ through the scaling parameter $\hat{c}^{(i)}$. This is a very appealing property: Take a multilayer network and assume (for the moment) that all layer inputs are Gaussian. Then, Lemma \ref{lem:critical_point_characterization_nn} still holds for any given hidden layer and gives rise to a decoupling of the optimal direction of this layer with all downstream weights, which in turn simplifies the curvature structure of the network since many Hessian blocks become zero.

However, classical local optimizers such as {\sc Gd} optimize both, direction and scaling, at the same time and are therefore blind to the above global property. It is thus very natural that performing optimization in the reparametrized weight space can in fact benefit from splitting the subtasks of optimizing scaling and direction in two parts, since updates in the latter are no longer sensitive to changes in the downstream part of the network. In the next section, we theoretically prove that such a decoupling accelerates optimization of weights of each individual unit in the presence of Gaussian inputs. Of course, the normality assumption is very strong but remarkably the experimental results of Section \ref{sec:exp_ii} suggest the validity of this result beyond the Gaussian design setting and thus motivate future research in this direction.
 
\subsection{Convergence analysis}
Here, we prove the convergence result restated below.
\convergencnn*
\begin{proof}
According to the result of Lemma~\ref{lem:critical_point_characterization_nn}, all critical points of $\fnn$ are aligned along the same direction as the solution of normalized least-squares. This property is similar to the objective of learning halfspaces (with Gaussian inputs) and the proof technique below therefore follows similar steps to the convergence proof of Theorem~\ref{thm:lh_convergence}.

\textbf{Gradient in the original parameterization } Recall the gradient of $\fnn$ is defined as
\begin{align*} 
  \nabla_{\Tilde{\w}^{(i)}} f^{(i)}/\theta^{(i)} = \alpha^{(i)} \u + \beta^{(i)} \cov \Tilde{\w}^{(i)} + \sum_{j=1}^m \gamma^{(i,j)} \cov \Tilde{\w}^{(j)}. \tag{\ref{eq:gradient_nn} revisited}
\end{align*}




\textbf{Gradient in the normalized parameterization: } Let us now consider the gradient of $\fnn$ w.r.t the normalized weights, which relates to the gradient in the original parameterization in the following way
\begin{align*} 
\nabla_\w f(\w,g) &= g \A_{\w} \nabla_{\Tilde{\w}} f(\Tilde{\w}), \\ \partial_g f(\w,g) &=\frac{ \w^\top \nabla_{\Tilde{\w}} f(\Tilde{\w})}{\| \w \|_\cov}  \tag{\ref{eq:normalized_gradient} revisited}
\end{align*}
Replacing the expression given in Eq.~\eqref{eq:gradient_nn} into the above formula yields 
\begin{equation}
\begin{aligned} 
\nabla_{\w^{(i)}}  \fnn/(g^{(i)} \theta^{(i)}) & = \alpha^{(i)} \A_{\w^{(i)}} \u + \beta^{(i)} \A_{\w^{(i)}} \cov \w^{(i)} \\&+ \sum_{j=1}^m \gamma^{(i,j)} \A_{\w^{(i)}} \cov \w^{(j)},
\end{aligned}
\end{equation}
where
\begin{align}
    \A_{\w^{(i)}} & := \Im/\|\w^{(i)}\|_\cov - \cov \w^{(i)} \otimes \w^{(i)} / \| \w^{(i)} \|^3_\cov.  \label{eq:A_w_i}
\end{align}
Note that the constants $\alpha^{(i)}$, $\beta^{(i)}$ and $\gamma^{(i,j)}$ all depend on the parameters $\w^{(i)}$ and $\theta^{(j)}$ of the respective units $i$ and $j$. The orthogonality of $\cov \w^{(i)}$ to $\A_{\w^{(i)}}$ (see Eq.~\eqref{eq:cov_w_A_orthogonality}) allows us to simplify things further: 
\begin{align} \label{eq:normalized_gradient_nn}
\nabla_{\w^{(i)}}  \fnn/(g^{(i)}\theta^{(i)}) = \alpha^{(i)} \A_{\w^{(i)}} \u + \sum_{j\neq i} \gamma^{(i,j)} \A_{\w^{(i)}} \cov \w^{(j)}
\end{align} 
We now use the initialization of weights $\{ \w^{(k)} = c_k \cov^{-1}\u \}_{k<i}$ and $\{ \w^{(j)} = \zero\}_{j>i}$ into the above expression to get
\begin{align}
    \nabla_{\w^{(i)}}  \fnn &= \theta^{(i)} g^{(i)} \xi_t \A_{\w^{(i)}} \u, \quad \xi = \alpha^{(i)} +  \sum_{j < i} \gamma^{(i,j)}  c_j \nonumber \\
    &= \theta^{(i)} g^{(i)}  \xi \left(\| \w^{(i)} \|_\cov/(2 \u^\top \w^{(i)}) \right) \nabla \rq(\w^{(i)})
\end{align}
where $\nabla \rq(\w)$ is the gradient of the normalized ordinary least squares problem (Eq.~\eqref{eq:LS_rayleigh_formulation}), i.e.
\begin{align*}  
-\nabla \rq(\w)/2 = \left( \u \u^\top \w + \frac{(\u^\top \w)^2}{\| \w \|_\cov^2 } \cov \w \right)/\| \w \|_\cov^2.
\end{align*} 
We conclude that the global characterization property described in Eq.~\eqref{eq:opt_direction_NN} transfers to the gradient since the above gradient aligns with the gradient of $\rq(\w)$.

\textbf{Choice of stepsize and stopping criterion } We follow the same approach used in the proof for learning halfspaces and choose a stepsize to ensure that the gradient steps on $\fnn^{(i)}$ match the gradient iterates on $\rq$, i.e.  
\begin{align*} 
\w^{(i)}_{t+1} &= \w^{(i)}_t - s^{(i)}_t \theta^{(i)} g^{(i)}_t  \xi_t \left(\| \w^{(i)}_t \|_\cov/(2 \u^\top \w^{(i)}) \right) \nabla \rq(\w^{(i)}_t) \\&= \w^{(i)}_t - \frac{\| \w_t \|_\cov^2}{2L|\rq(\w_t^{(i)})|} \nabla \rq(\w^{(i)}_t),
\end{align*} 
which leads to the following choice of stepsize 
\begin{equation}
\begin{aligned} \label{eq:step_size_nn}
s^{(i)}_t &= \| \w^{(i)}_t \|^3_\cov /(L\theta^{(i)} g_t^{(i)} \xi_t \u^\top \w^{(i)}_t)\\
\xi_t &= \alpha^{(i)}_t + \sum_{j<i} \gamma^{(i,j)}_t c_j.
\end{aligned}
\end{equation}

If $\xi_t = 0$, then the gradient is zero. Therefore, we choose the stopping criterion as follows
\begin{align} \label{eq:stopping_time_nn}
h^{(i)}_t = \xi_t = \alpha^{(i)}_t + \sum_{j<i} \gamma^{(i,j)}_t c_j.
\end{align} 

\textbf{Gradient norm decomposition } Proposition~\ref{prop:gradient_norm_in_different_parameterization} relates the $\cov^{-1}$-norm of the gradient in the original space to the normalized space as follows 
\begin{align*}  
    \| \nabla_{\Tilde{\w}^{(i)}} f(\Tilde{\w}_t^{(i)}) \|^2_{\cov^{-1}} &=  \| \w_t^{(i)} \|_\cov^2 \| \nabla_{\w^{(i)}} f(\w_t^{(i)},g^{(i)})\|^2_{\cov^{-1}}/(g^{(i)}_t)^2\\& + \left(\partial_{g^{(i)}} f(\w_t^{(i)},g_t^{(i)})\right)^2. \tag{\ref{eq:gradient_norm_in_different_parameterization} revisited}
\end{align*}
In the following, we will establish convergence individually in terms of $g$ and $\w$ and then use the above result to get a global result. 

\textbf{Convergence in scalar $g^{(i)}$ }
Since the smoothness property defined in Assumption~\ref{as:smooth_assumptions} also holds for $\fnn^{(i)}$, we can directly invoke the result of Lemma~\ref{lem:convergence_in_g} to establish a convergence rate for $g$:  
\begin{align*} 
    \left( \partial_{g^{(i)}} f^{(i)}(\w_t^{(i)}, g_{\tsc}^{(i)}) \right)^2 \leq  2^{-\tsc^{(i)}}\zeta  | b^{(0)}_t - a_t^{(0)} |/\mu^{2}  \tag{\ref{eq:convergence_lh_in_g} revisited}
\end{align*}

\textbf{Directional convergence }
By the choice of stepsize in Eq.~\eqref{eq:step_size_nn}, the gradient trajectory on $\fnn$ reduces to the gradient trajectory on $\rq(\w)$. Hence, we can establish a linear convergence in $\w^{(i)}$ by a simple modification of Eq.~\eqref{eq:convergenc_direction_lh}: 
\begin{equation}
\begin{aligned} 
&\| \w_t^{(i)} \|^2_\cov \| \nabla_\w f^{(i)}(\w_t^{(i)}, g_t^{(i)})\|^2_{\cov^{-1}} \leq & (1-\mu/L)^{2t}  \xi_t^2 g_t^2 \left( \rq(\w_0) - \rq(\w^*) \right).
\end{aligned} 
\end{equation}

The assumption~\ref{as:reg_assumptions} on loss with the choice of activation function as $\tanh$ allows us to bound the scalar $\xi_t^2$: 
\begin{align}
    \xi_t^2 \leq  2\Phi^2 + 2 i \sum_{j<i} (\theta^{(j)} c_j)^2.
\end{align}

\textbf{Combined convergence bound } Combining the above results concludes the proof in the following way  
\begin{align*} 
    \| \nabla_{\Tilde{\w}^{(i)}} f(\Tilde{\w}^{(i)}_t) \|^2_{\cov^{-1}}  \leq &  (1-\mu/L)^{2t}  C \left( \rq(\w_0) - \rq^* \right) \\&+ 2^{-\tsc^{(i)}}\zeta  | b^{(0)}_t - a_t^{(0)} |/\mu^{2},
\end{align*}
where 
\begin{align} \label{eq:constant_c_nn}
C =  2\Phi^2 + 2 i \sum_{j<i} (\theta^{(j)} c_j)^2>0.
\end{align}
\end{proof}

\section{EXPERIMENTAL DETAILS}
\subsection{Learning Halfspaces}
\textbf{Setting }
We consider empirical risk minimization (ERM) as a surrogate for \eqref{eq:halfspace_problem} in the binary classification setting and make two different choices for $\varphi(\cdot)$:

\begin{equation*}
\begin{aligned}
&\textit{softplus}(\w^\top \z):=\E_\z\left[ \log(1+\exp(\Tilde{\w}^\top\z))\right], \\
&\textit{sigmoid}(\w^\top \z):=\E_\z\left[1/(1+\exp(-\Tilde{\w}^\top \z))\right]
\end{aligned}
\end{equation*}

The first resembles classical convex logistic regression when $\y_i\in \{-1,1\}$. The second is a commonly used non-convex, continuous approximation of the zero-one loss in learning halfspaces \citep{zhang2015learning}

As datasets we use the common realworld dataset \textit{a9a} ($n=32'561, d=123$) as well a synthetic data set drawn from a multivariate gaussian distribution such that $\z \sim \mathcal{N}(\u,\cov)$ ($n=1'000, d=50$).

\textbf{Methods }
We compare the convergence behavior of {\sc Gd} and Accelerated Gradient Descent ({\sc Agd}) \citep{nesterov2013introductory} to Batch Normalization plus two versions of {\sc Gd} as well as Weight Normalization. Namely, we assess 
\begin{itemize}
    \item  {\sc Gdnp} as stated in Algorithm \ref{alg:gdnp} but with the Bisection search replaced by multiple Gradient Descent steps on $g$ (10 per outer iteration)
    \item Batch Norm plus standard {\sc Gd} which simultaneously updates $\w$ and $\g$ with one gradient step on each parameter.
    \item Weight Normalization plus standard {\sc Gd} as above. \citep{salimans2016weight}
\end{itemize}

\begin{figure}[h!]
\centering          
          \begin{tabular}{c@{}c@{}}
            \adjincludegraphics[width=0.45\linewidth, trim={22pt 22pt 30pt 30pt},clip]{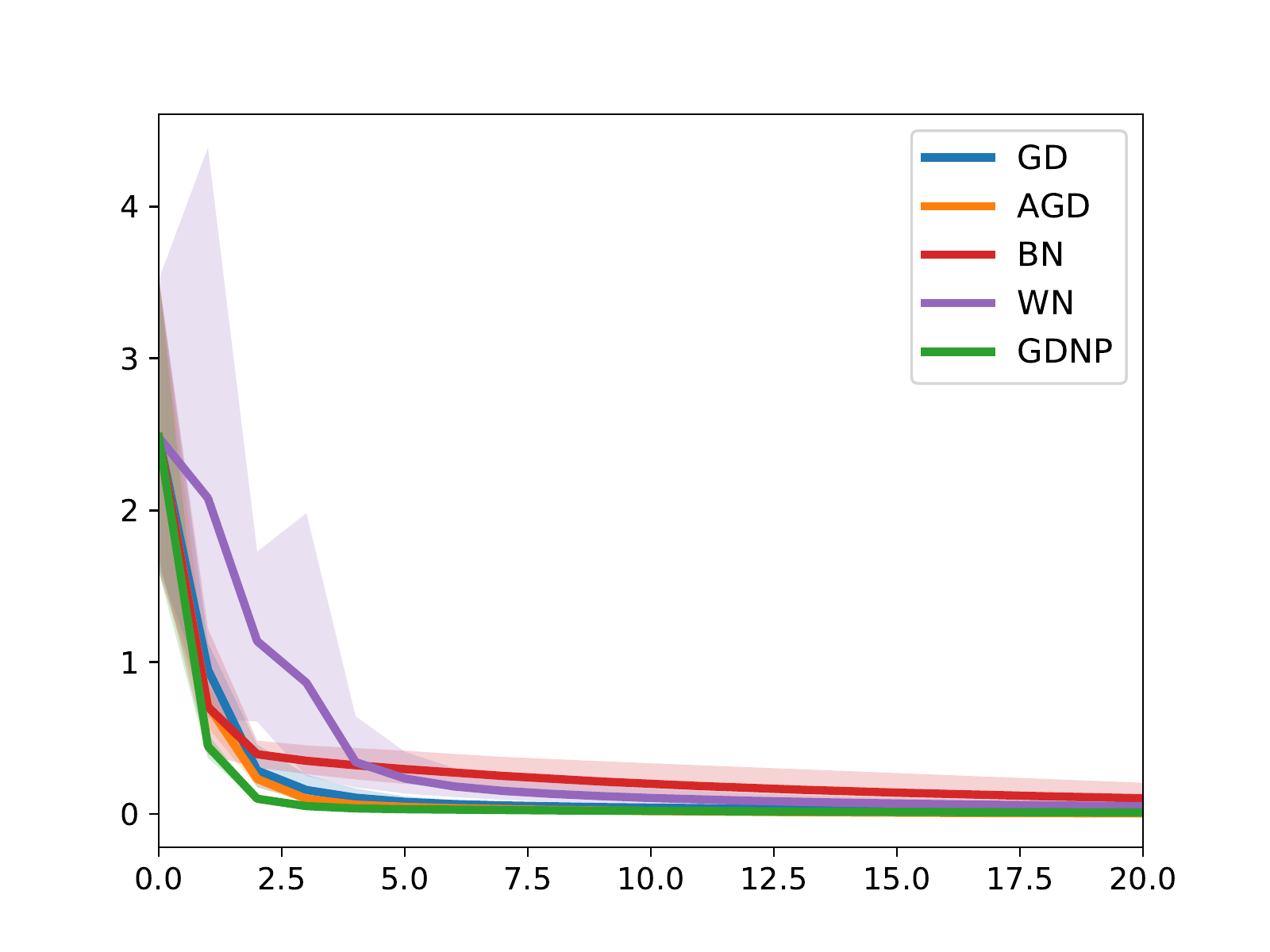} &
             \adjincludegraphics[width=0.45\linewidth, trim={22pt 22pt 30pt 30pt},clip]{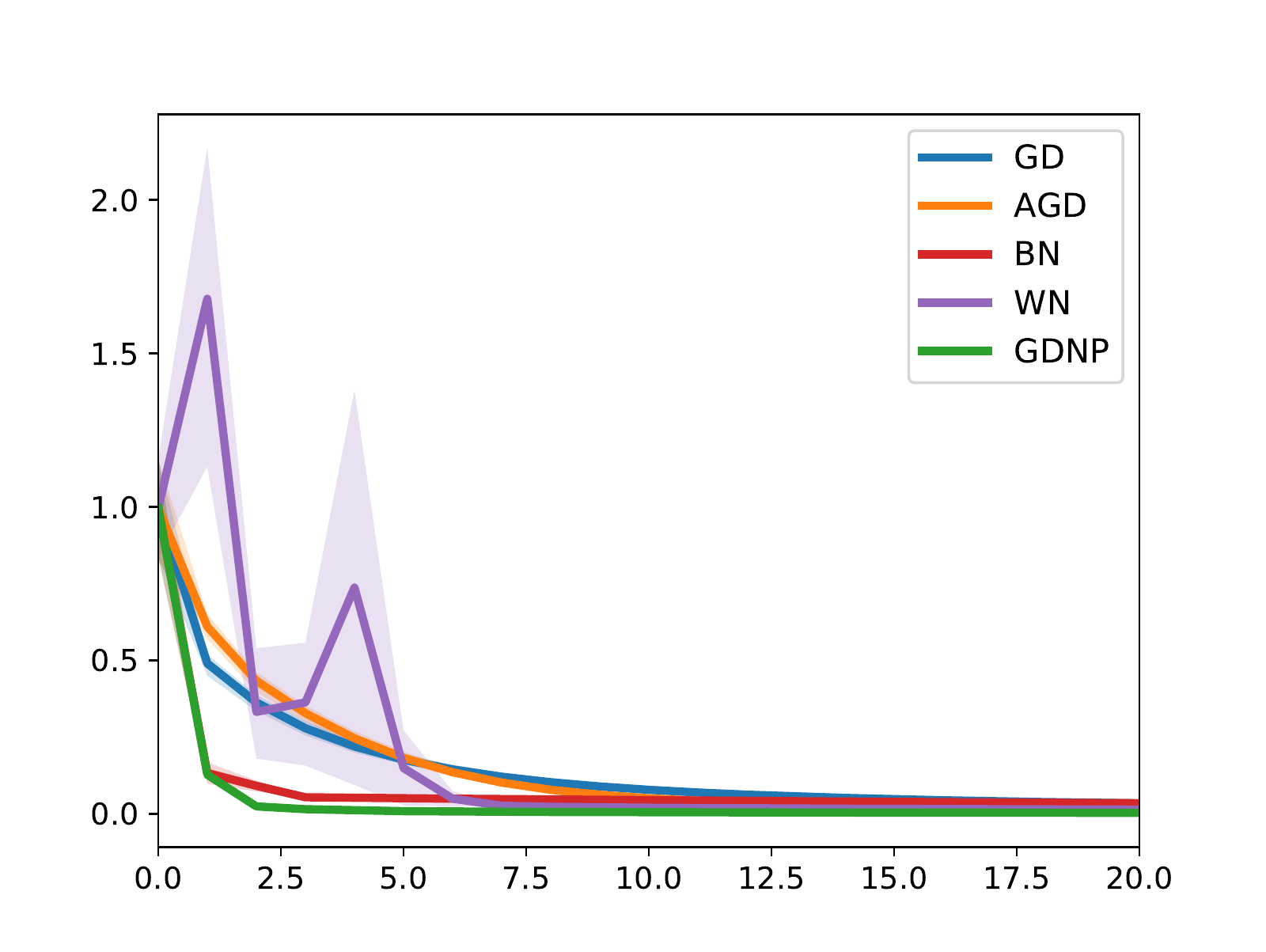}\\
              \footnotesize{{ gaussian softplus}} &
             \footnotesize{{ a9a softplus}}\\ \adjincludegraphics[width=0.45\linewidth, trim={12pt 22pt 35pt 37pt},clip]{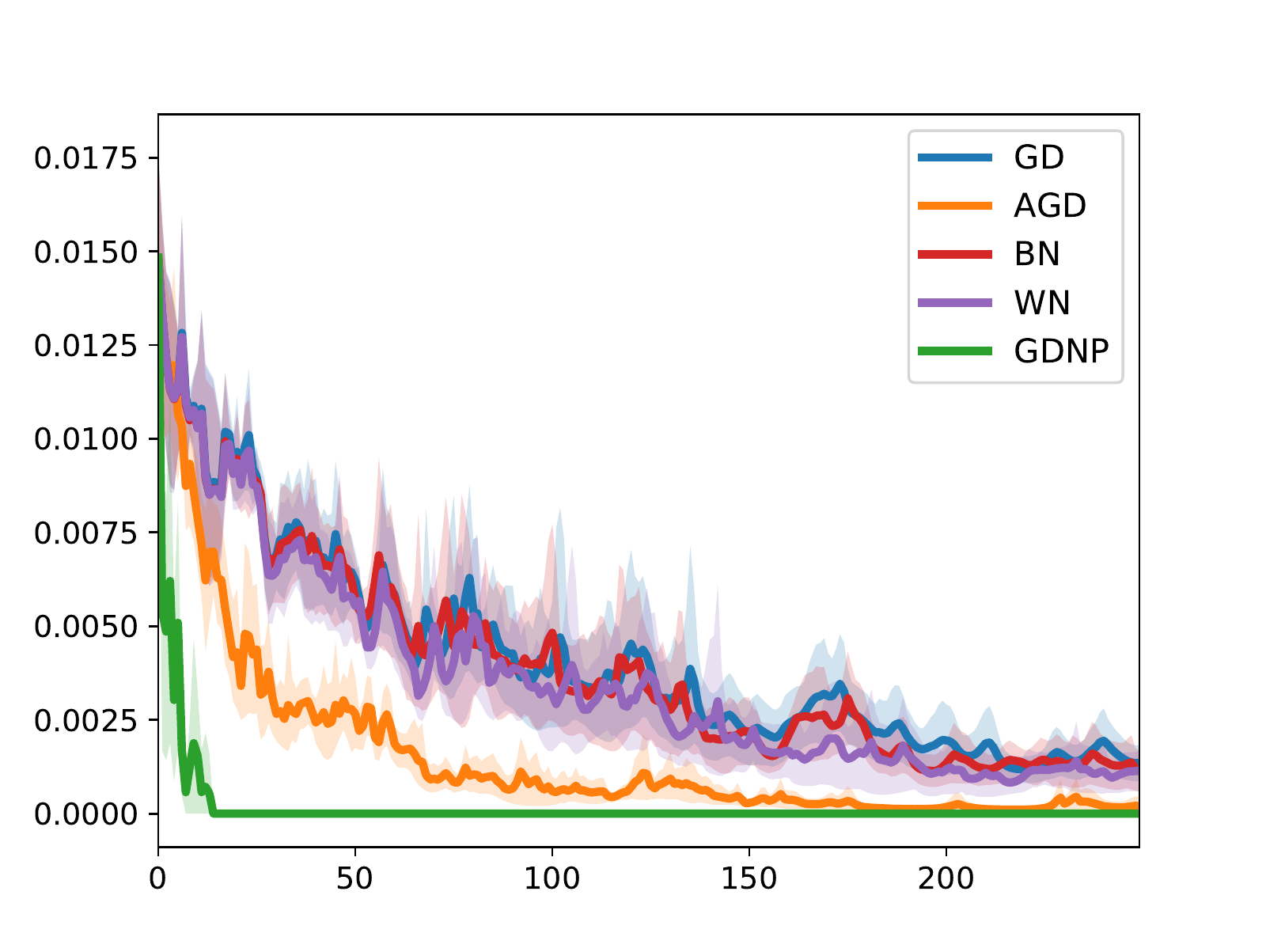}&
              \adjincludegraphics[width=0.45\linewidth, trim={12pt 22pt 35pt 37pt},clip]{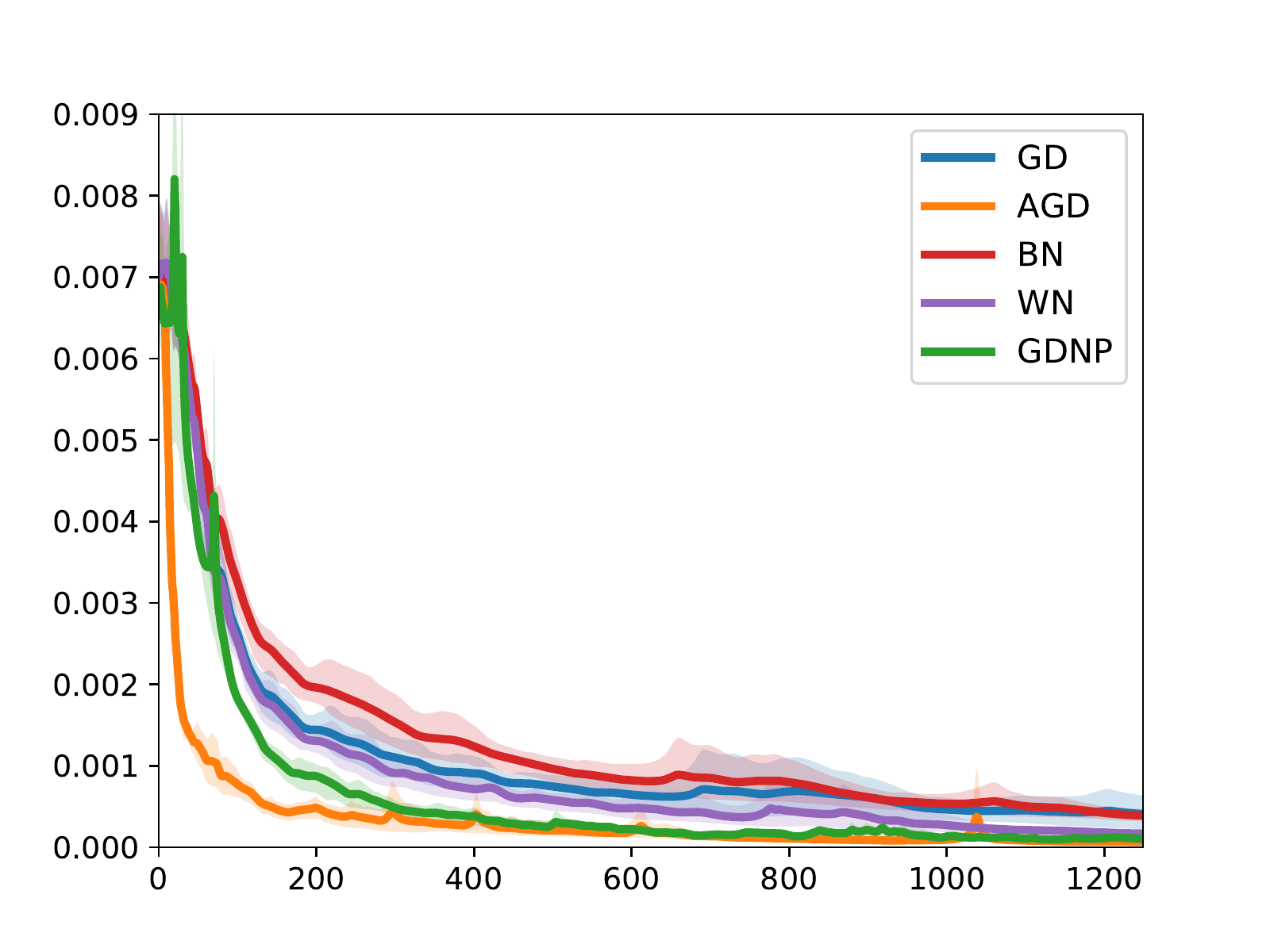} \\  
            \vspace{-1pt}
        
       \hspace{10pt}  \footnotesize{{gaussian sigmoid}}&
        \hspace{5pt}     \footnotesize{{a9a sigmoid}}
	  \end{tabular}
          \caption{The plots are the same as in Figure~\ref{fig:exp_halfspace_log} but show results in linear instead of log terms: Results of an average run (solid line) in terms of log suboptimality (softplus) and log gradient norm (sigmoid) over iterations as well as 90\% confidence intervals of 10 runs with random initialization. }
          \label{fig:exp_halfspace_linear_app}
\end{figure}

All methods use full batch sizes. {\sc Gdnp} uses stepsizes according to the policy proposed in Theorem \ref{thm:lh_convergence}. On the softplus, {\sc Gd}, {\sc Agd}, {\sc Wn} and {\sc Bn} are run with their own, constant, grid-searched stepsizes. Weight- and Batch Norm use a different stepsize for direction and scaling but only take one gradient step in each parameter per iteration. 
Since the sigmoid setting is non-convex and many different local minima and saddle points may be approched by the different algorithms in different runs, there exist no meaningful performance measure to gridsearch the stepsizes. We thus pick the inverse of the gradient Lipschitz constant $\zeta$ for all methods and all parameters, except {\sc Gdnp}. To estimate $\zeta$, we compute $\zeta_{\sup}:=  \|\Z^\top\Z\|_2/10 \geq  \sum_i \varphi(\w^\top\z_i)^{(2)} \|\Z^\top\Z\|_2/n$  where $\Z:=[z_1,\ldots,z_n]^\top \in \mathbb{R}^{n\times d}$ and $\varphi(\cdot)^{(2)}\leq 0.1$. After comparison with the largest eigenvalue of the Hessian at a couple of thousand different parameters $\Tilde{\w}$ we found the bound to be pretty tight. Note that for {\sc Gdnp}, we use $L:= \|\Z^\top\Z\|_2$ which can easily be computed as a pre-processing step and is -- contrary to $\zeta$ -- independent of $\w$. {\sc Agd} computes the momentum parameter $\beta_t=(t-2)/(t+1)$ in the convex case and uses a (grid-searched) constant $\beta_t=\beta \in [0,1]$ in the non-convex setting. We initialize randomly, i.e. we draw $\Tilde{\w}_0 \sim \mathcal{N}(0,1)$, set $\w_0:=\Tilde{\w}_0$ and choose $g_0$ sucht that $\Tilde{\w}_0=g\w_0/\|\w_0\|_\cov$.

\textbf{Results} The Gaussian design experiments clearly confirm Theorem~\ref{thm:lh_convergence} in the sense that the loss in the convex-, as well as the gradient norm in the non-convex case decrease at a linear rate. The results on \textit{a9a} show that {\sc Gdnp} can accelerate optimization even when the normality assumption does not hold and in a setting where no covariate shift is present. This motivates future research of non-linear reparametrizations even in convex optimization.

\begin{figure}[h!]\label{fig:NN_result_log}
	\begin{center}
          \begin{tabular}{c@{}c@{}}
            \adjincludegraphics[width=0.45\linewidth, trim={22pt 22pt 30pt 30pt},clip]{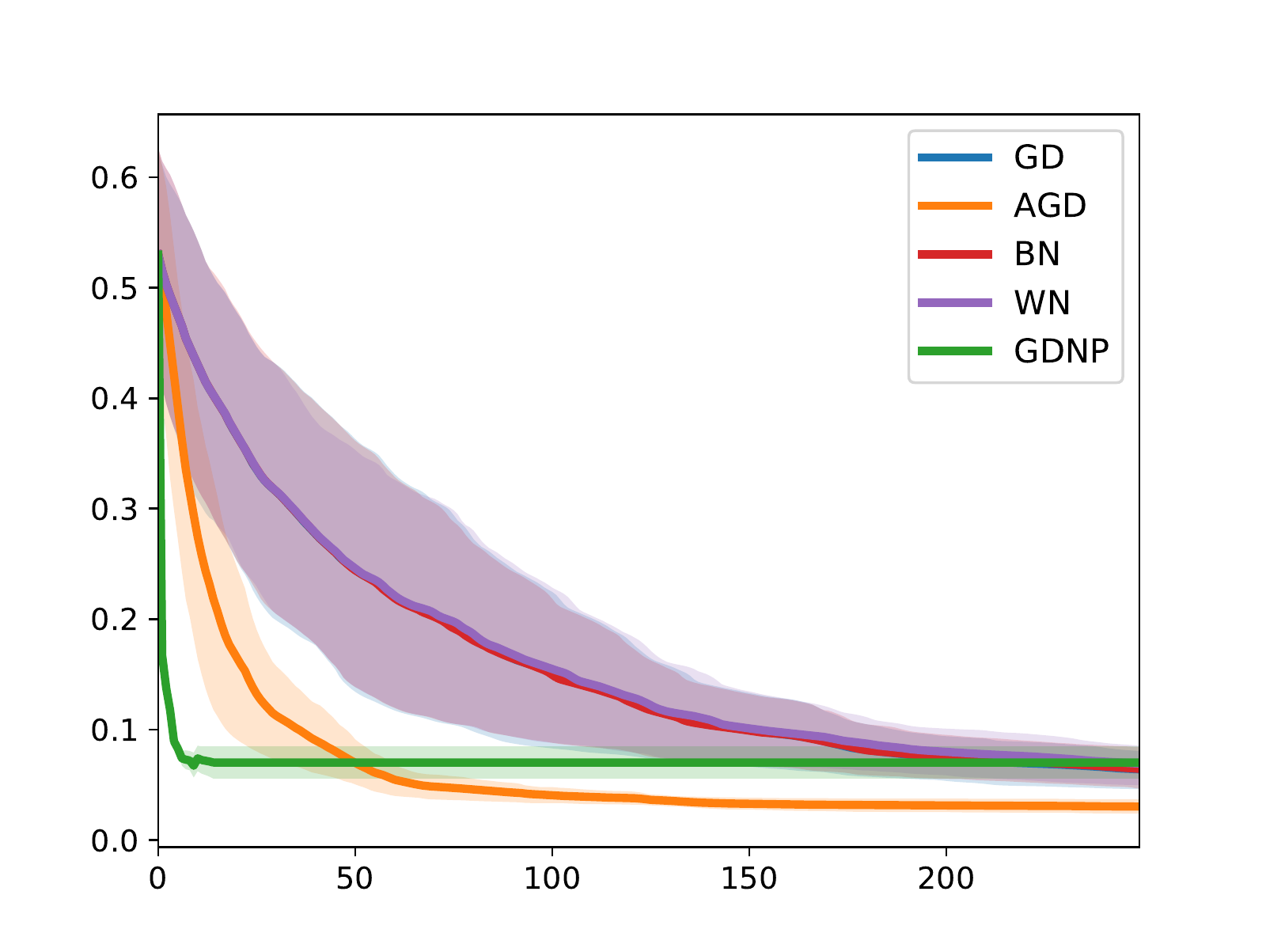} &
            \adjincludegraphics[width=0.45\linewidth, trim={22pt 22pt 30pt 30pt},clip]{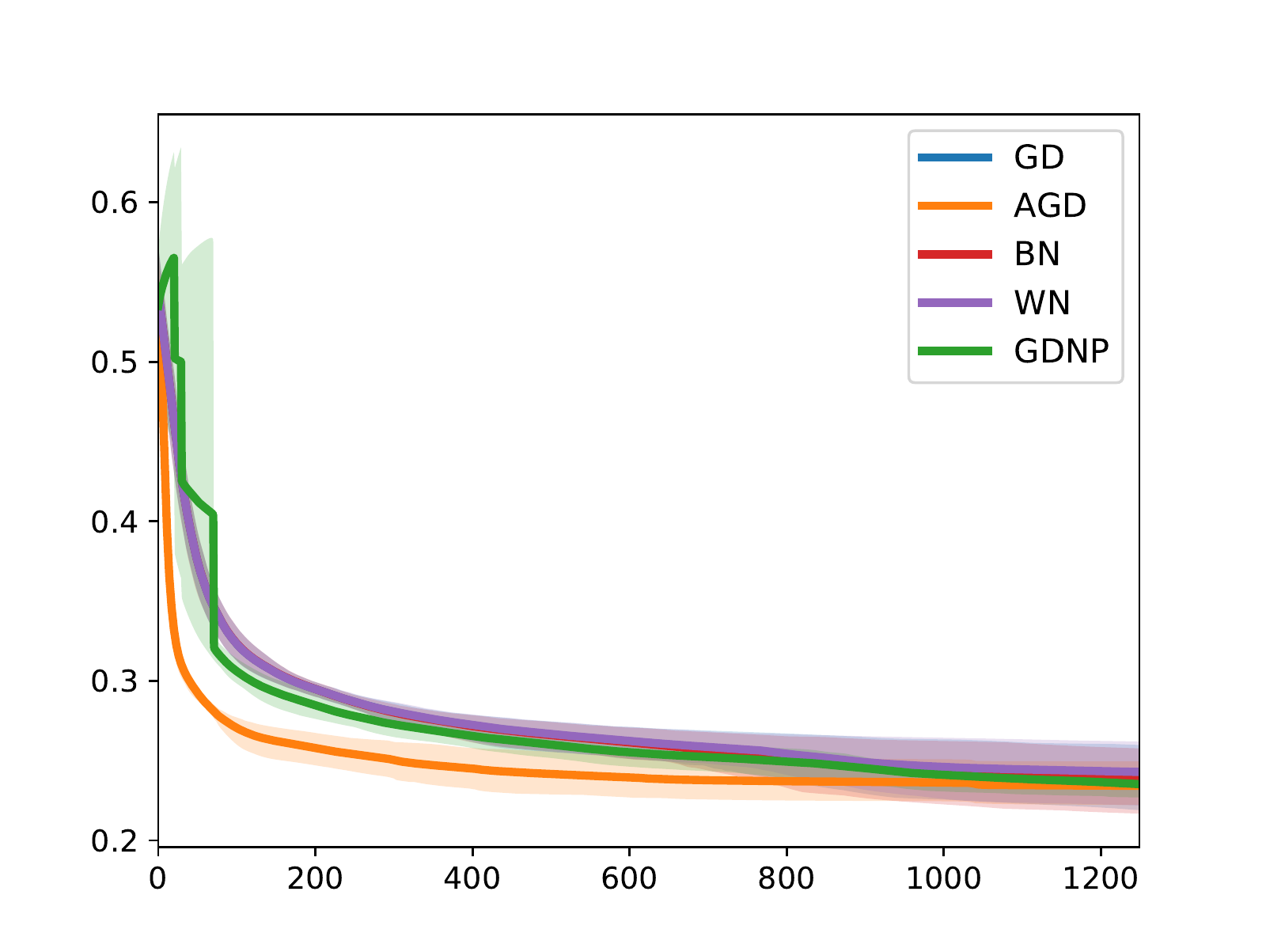} \\  \footnotesize{{ gaussian sigmoid}}&
            
                 \footnotesize{{ a9a sigmoid}}
            

	  \end{tabular}          

          \caption{ Addition to Figure  \ref{fig:exp_halfspace_log} and \ref{fig:exp_halfspace_linear_app}: Suboptimality on the non-convex sigmoid problems in linear terms.}
	\end{center}
\end{figure}

Regarding {\sc Bn} and {\sc Wn} we found a clear trade-off between making fast progress in the beginning and optimizing the last couple of digits. In the above results of Figure \ref{fig:exp_halfspace_log} and \ref{fig:exp_halfspace_linear_app} we report runs with stepsizes that were optimized for the latter case but we here note that early progress can easily be achieved in normalized parametrizations (which the linear a9a softplus plot actually confirms) e.g. by putting a higher learning rate on $\g$. In the long run similar performance to that of {\sc Gd} sets in, which suggests that the length-direction decoupling does not fully do the trick. The superior performance of {\sc Gdnp} points out that either an increased number of steps in the scaling factor $g$ or an adaptive stepsize scheme such as the one given in Eq.~\eqref{eq:stepsize_LS} (or both) may significantly increase the performance of Batch Normalized Gradient Descent {\sc Bn}. 

It is thus an exciting open question whether such simple modifications to {\sc Gd} can also speed up the training of Batch Normalized neural networks. Finally, since {\sc Gdnp} performs similar to {\sc Agd} in the non-gaussian setting, it is a logical next step to study how accelerated gradient methods like {\sc Agd} or Heavy Ball perform in normalized coordinates.

As a side note, Figure \ref{fig:loss_glms} shows how surprisingly different the paths that Gradient Descent takes before and after normalization can be. 
 
\begin{figure}[h!]
\centering          
          \begin{tabular}{c@{}c@{}}
            \adjincludegraphics[width=0.4\linewidth, trim={22pt 22pt 30pt 30pt},clip]{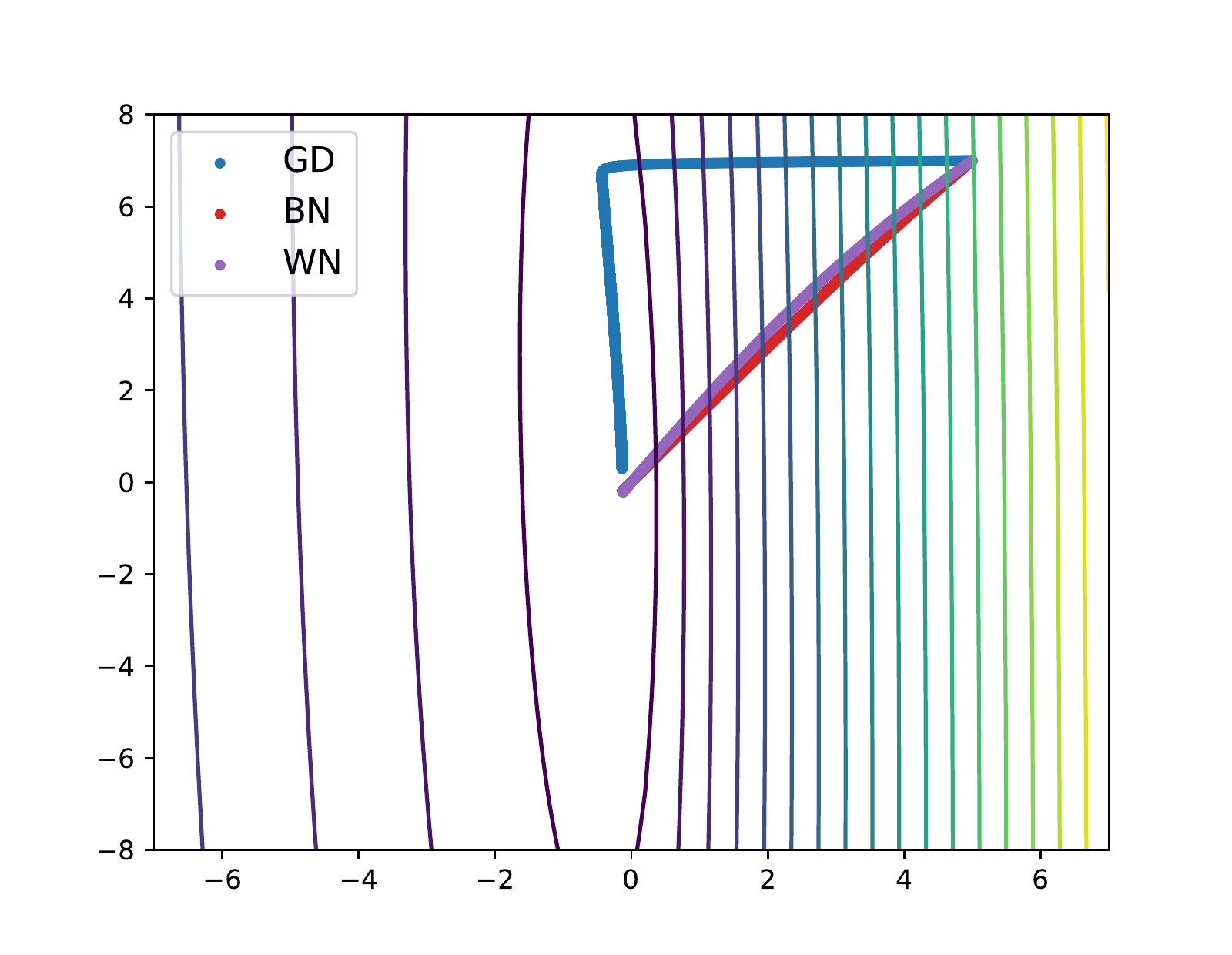} &
             \adjincludegraphics[width=0.4\linewidth, trim={22pt 22pt 30pt 30pt},clip]{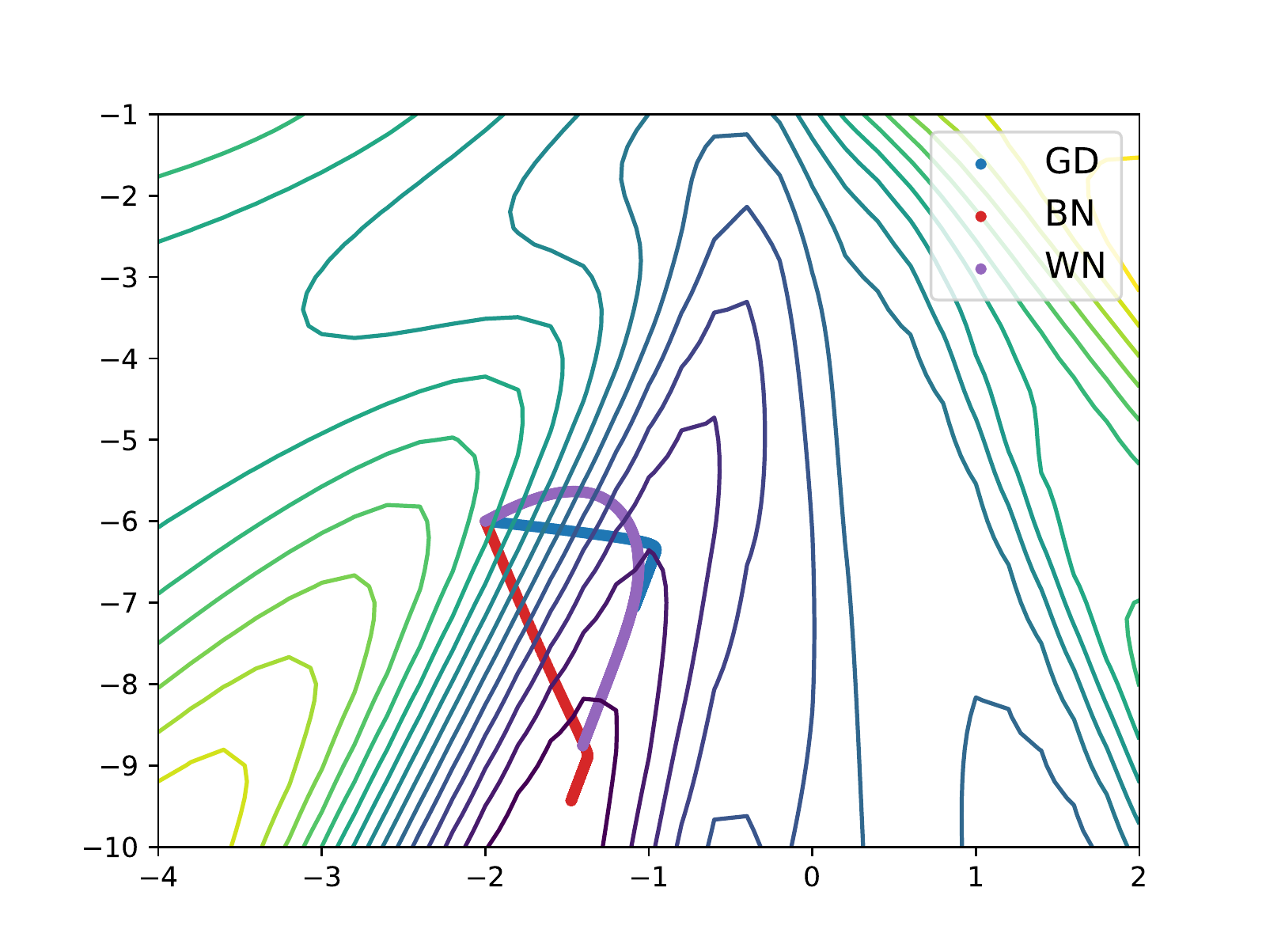}\\
             
            \adjincludegraphics[width=0.4\linewidth, trim={22pt 22pt 30pt 30pt},clip]{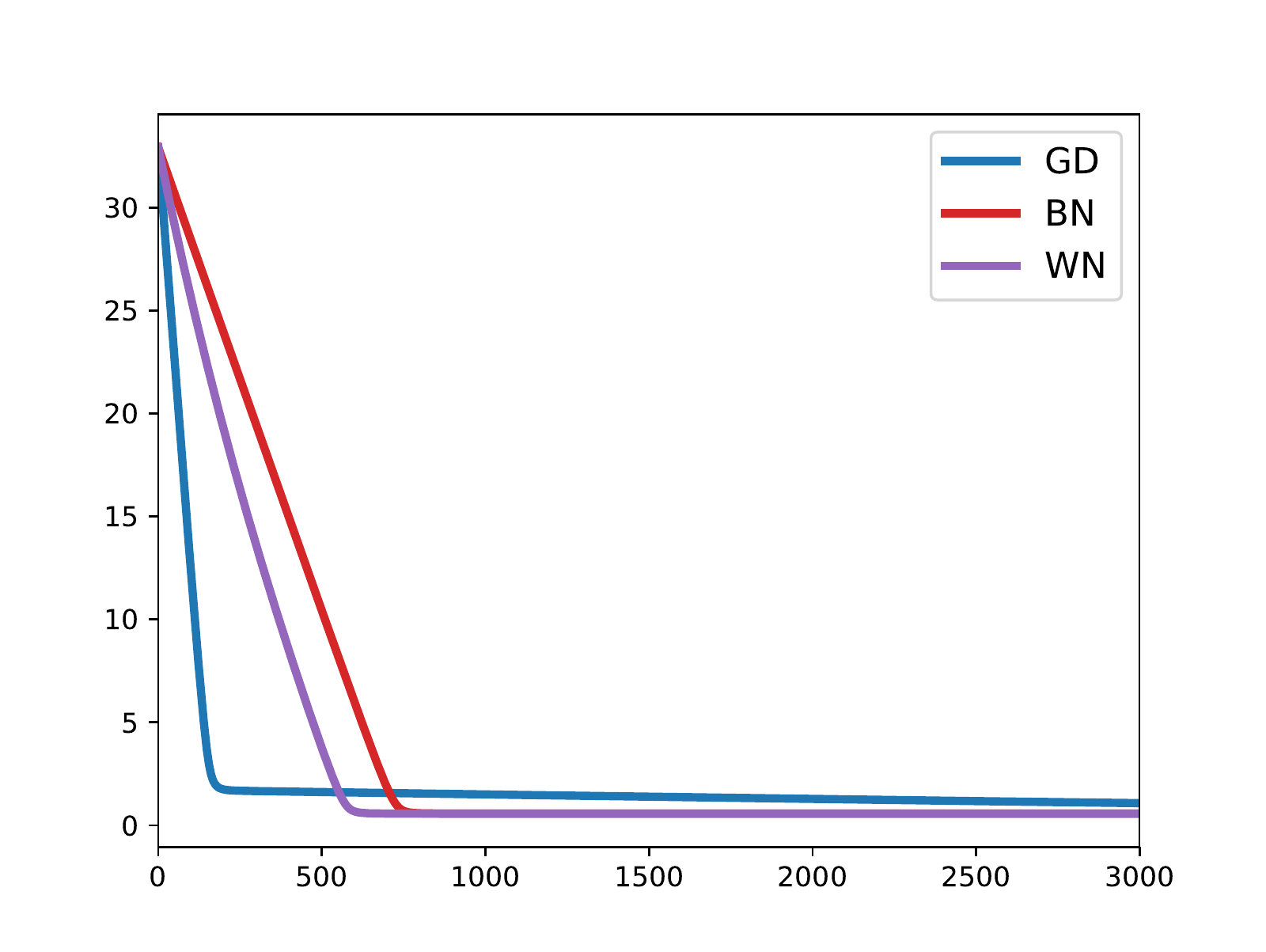} &
             \adjincludegraphics[width=0.4\linewidth, trim={22pt 22pt 30pt 30pt},clip]{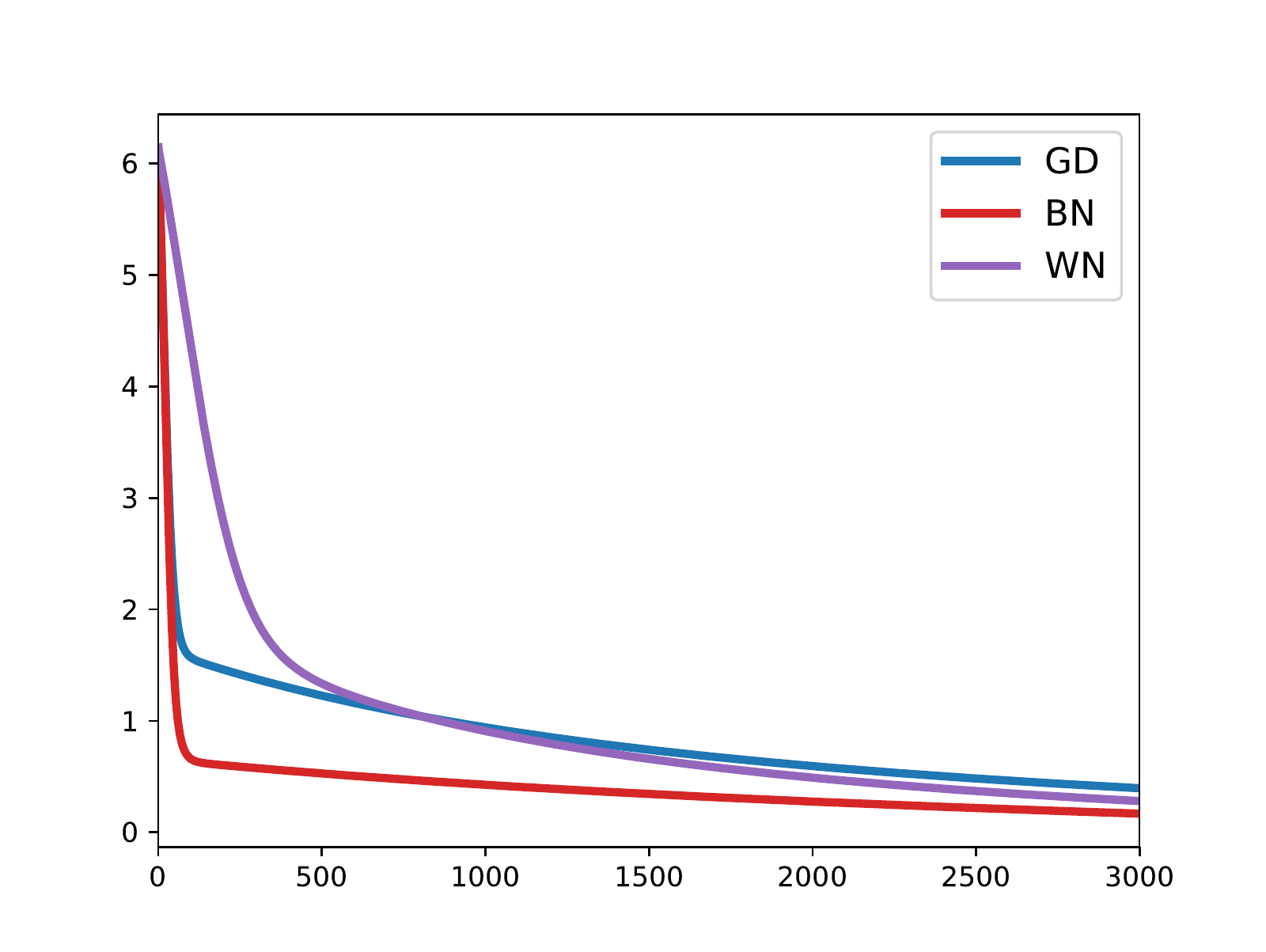}
	  \end{tabular}
          \caption{ \footnotesize{Normalization can lead to suprisingly different paths: Level sets and path (top) as well as sub-optimality (bottom) of {\sc Gd}, {\sc Bn} and {\sc Wn} (with constant step size and fixed number of iterations) on two instances of learning halfspaces with Gaussian data ($n=5000$, $d=2$). Left: convex logistic regression, right: non-convex sigmoidal regression.}}
          \label{fig:loss_glms}
\end{figure}
\subsection{Neural networks}
\textbf{Setting and methods }
We test the validity of Theorem \ref{lem:convergence_nn} and Lemma \ref{lem:critical_point_characterization_nn} outside the Gaussian setting and a normalized and an unnormalized feedforward networks on the CIFAR10 image classification task. This dataset consists of 60000 32x32 images in 10 classes, with 6000 images per class  \citep{krizhevsky2009learning}. The networks have six hidden layers with 50 hidden units in each of them. Each hidden unit has a $\textit{tanh}$ activation function, except for the very last layer which is linear. These scores are fed into a cross entropy loss layer which combines softmax and negative log likelihood loss. The experiments are implemented using the PyTorch framework \citep{paszke2017automatic}.

The first network is trained by standard {\sc Gd} and the second by {\sc Gd} in normalized coordinates (i.e. {\sc Bn}) with the same fixed stepsize on $\title{\w}$ and $\w$, but we increase the learning rate on $\g$ by a factor of 10 which accelerates training significantly. The second network thus resembles performing standard {\sc Gd} in a network where all hidden layers are Batch Normalized. We measure the cross-dependency of the central with all other layers in terms of the Frobenius norm of the second partial derivatives $\frac{\partial^2 \fnn}{\partial \W_4 \partial \W_i}$. This quantity signals how the gradients of layer 4 change when we alter the direction of any other layer. From an optimization perspective, this is a sound measure for the cross-dependencies: If it is close to zero (high), that means that a change in layer $i$ induces no (a large) change in layer 4. Compared to gradient calculations, computing second derivatives is rather expensive $O(nd^2)$ (where $d=66700$), which is why we evaluate this measure every only 250 iterations.

\begin{figure}[h!]\label{fig:NN_result_log}
	\begin{center}
          \begin{tabular}{c@{}c@{}}
            \adjincludegraphics[width=0.45\linewidth, trim={22pt 22pt 30pt 30pt},clip]{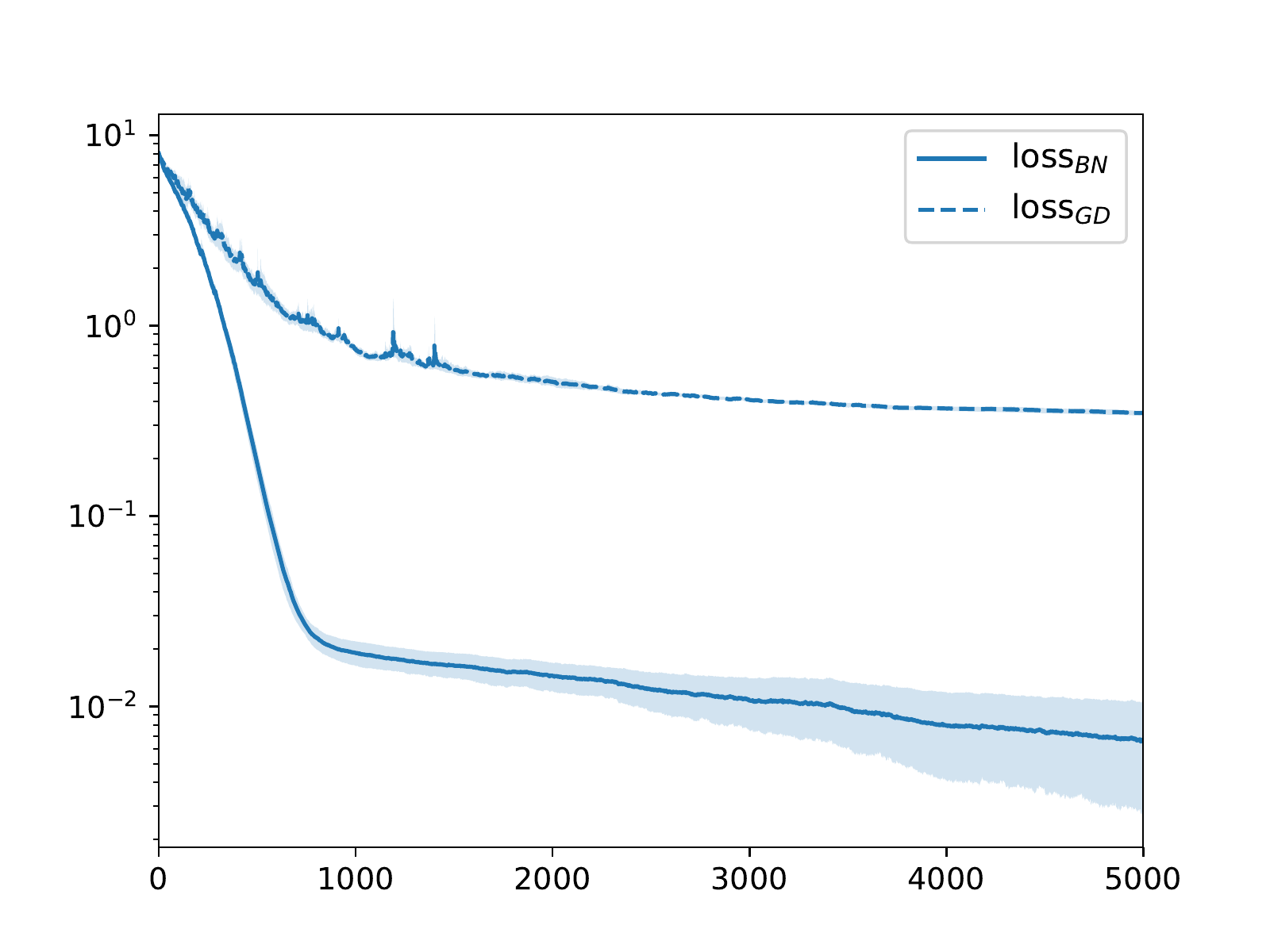} &
            \adjincludegraphics[width=0.45\linewidth, trim={22pt 22pt 30pt 30pt},clip]{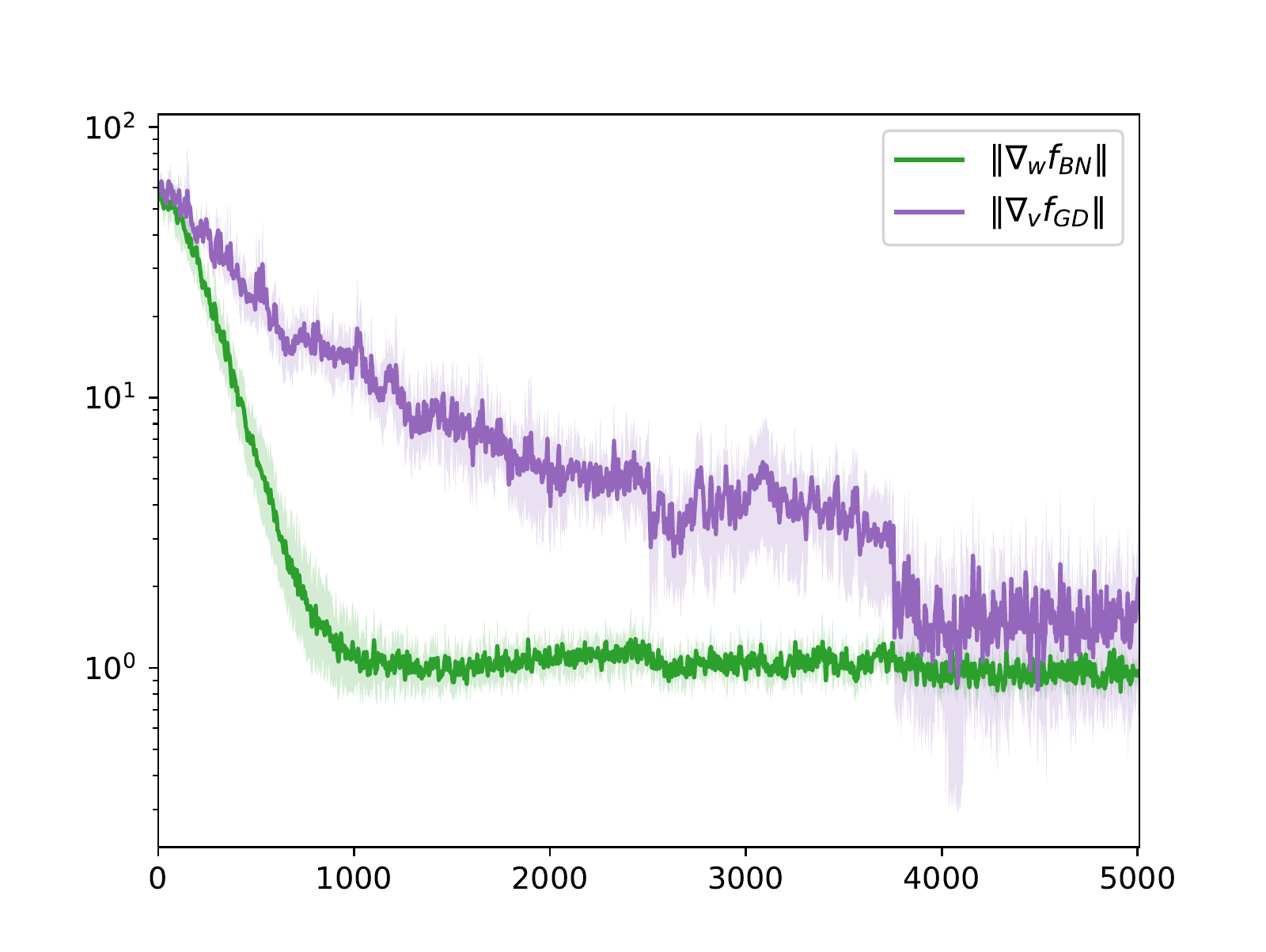} \\ \adjincludegraphics[width=0.45\linewidth, trim={22pt 22pt 30pt 30pt},clip]{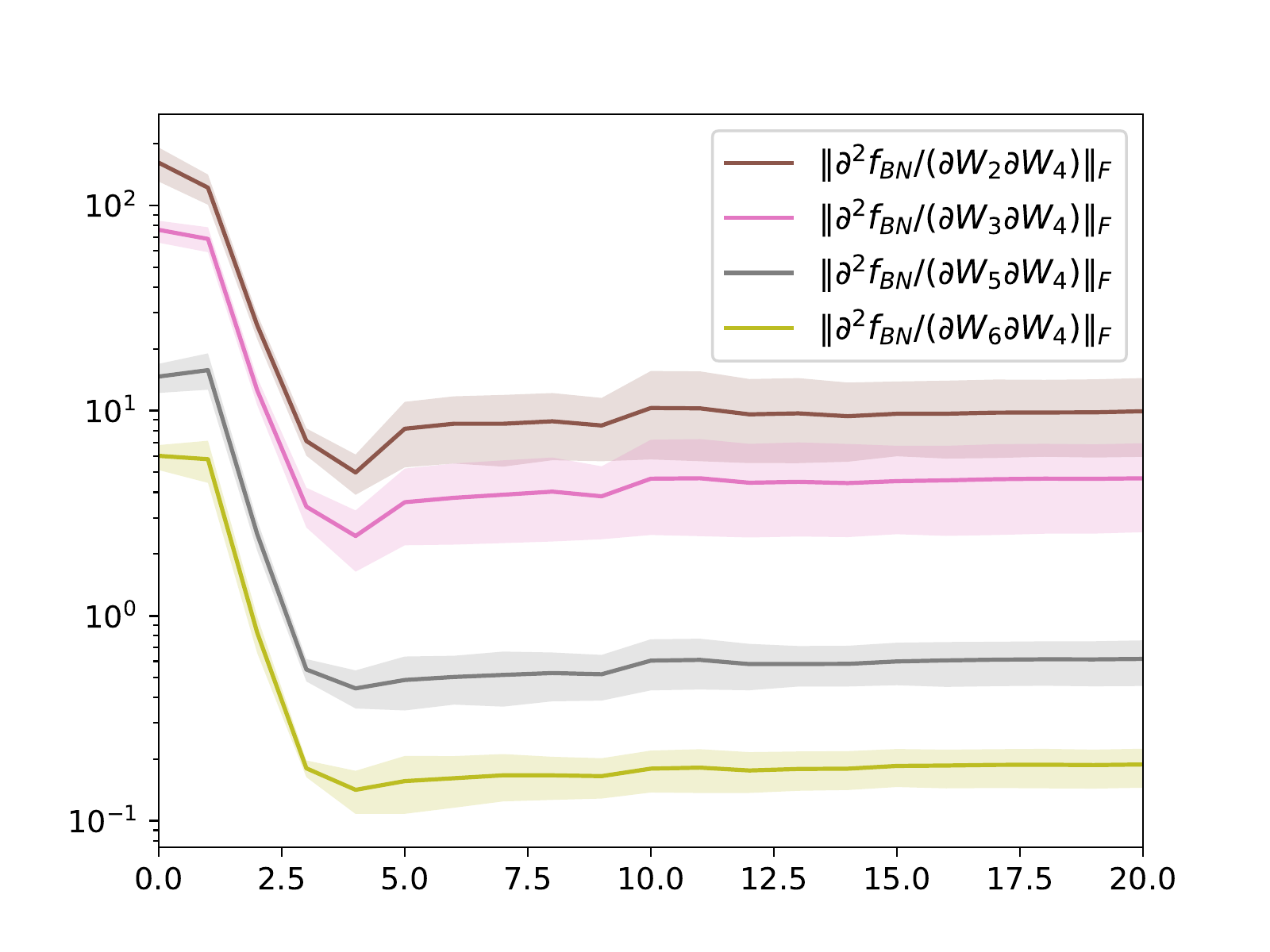}&
            
                \adjincludegraphics[width=0.45\linewidth, trim={22pt 22pt 30pt 30pt},clip]{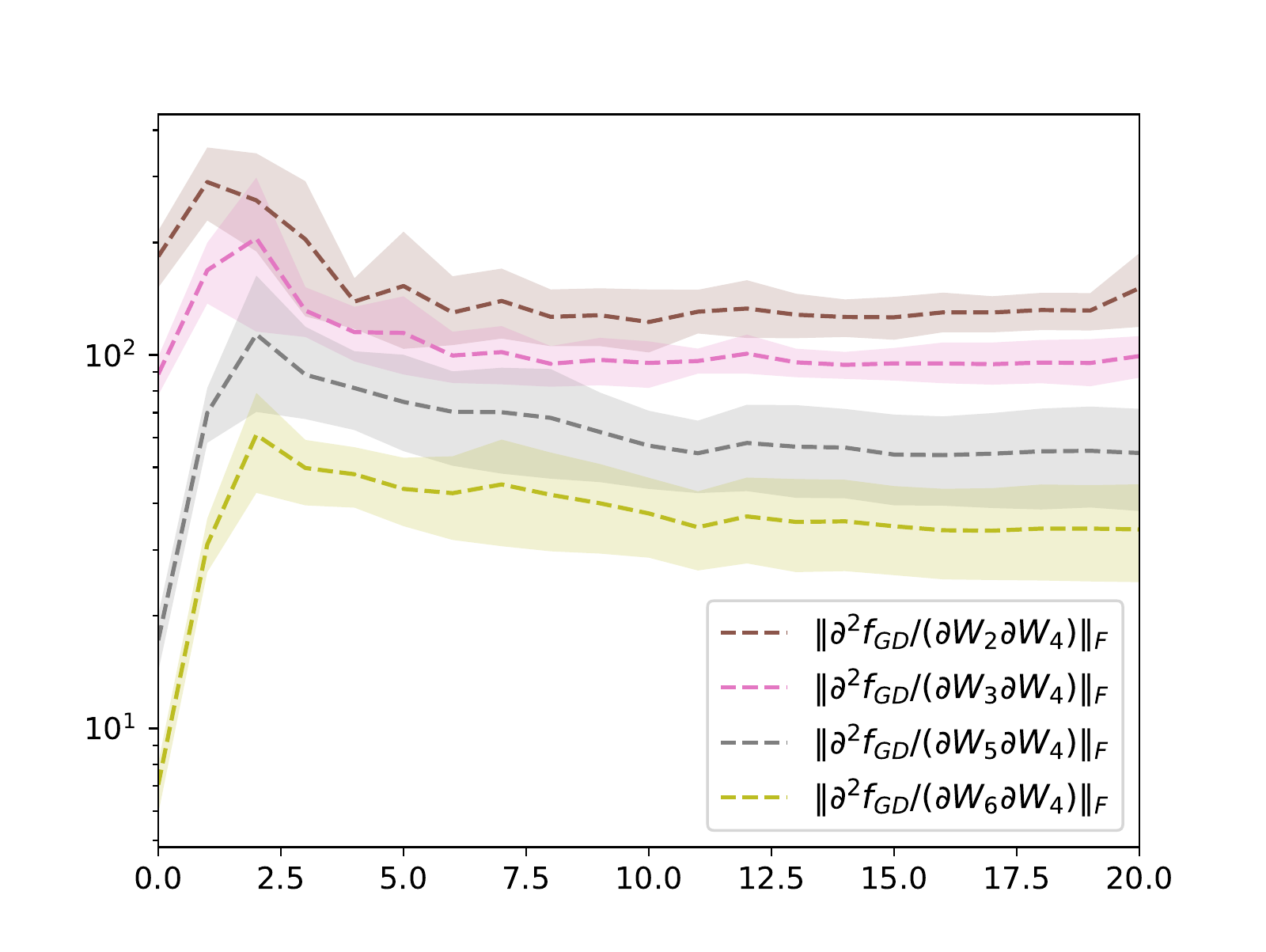}
            

	  \end{tabular}          

          \caption{ The plots are the same as in Figure \ref{fig:exp_NN_linear} but show results in log instead of linear terms: (i) Loss, (ii) gradient norm and dependencies between central- and all other layers for BN (iii) and GD (iv) on a 6 hidden layer network with 50 units (each) on the CIFAR10 dataset.}\label{fig:exp_NN_log}
	\end{center}
\end{figure}

\textbf{Results } Figure \ref{fig:exp_NN_linear} and \ref{fig:exp_NN_log} confirm that the directional gradients of the central layer are affected far more by the upstream than by the downstream layers to a surprisingly large extent. Interestingly, this holds even before reaching a critical point. The cross-dependencies are generally decaying for the Batch Normalized network ({\sc Bn}) while they remain elevated in the un-normalized network ({\sc Gd}), which suggest that using Batch Normalization layers indeed simplifies the networks curvature structure in $\w$ such that the length-direction decoupling allows Gradient Descent to exploit simpler trajectories in these normalized coordinates for faster convergence. Of course, we cannot untangle this effect fully from the covariate shift reduction that was mentioned in the introduction. Yet, the fact that the (de-)coupling increases in the distance to the middle layer (note how earlier (later) layers are more (less) important for the $\W_4$) emphasizes the relevance of this analysis particularly for deep neural network structures, where downstream dependencies might vanish completely with depth. This does not only make gradient based training easier but also suggests the possibility of using partial second order information, such as diagonal Hessian approximations (e.g. proposed in \citep{martens2012estimating}).

\end{document}